\documentclass[a4paper,12pt]{article}

\usepackage[margin=1in]{geometry}
\usepackage{xcolor}
\usepackage{amsmath, amssymb, amsfonts, amsthm,comment}
\usepackage{algorithm}
\usepackage{algpseudocode}
\usepackage{bbm}
\usepackage[font=small]{caption}
\usepackage[font=scriptsize]{subcaption}
\usepackage{enumitem}
\usepackage{graphicx}
\usepackage{pgfplots}
\pgfplotsset{compat=1.17}
\usepackage{soul}
\usepackage{hyperref}                                   

\usepackage[round]{natbib}

\DeclareMathOperator{\bary}{Bary}

\DeclareMathOperator{\cov}{CoV}

\newcommand{\D}{\mathcal{D}}
\newcommand{\Pb}{\mathcal{P}}

\newcommand{\Y}{\mathbf{Y}}
\newcommand{\R}{\mathbb{R}}

\newcommand{\tm}{\Tilde{\mu}}
\newcommand{\tn}{\Tilde{\nu}}

\newcommand{\tpi}{\Tilde{\pi}}

\def\y{\mathbf{y}}
\def\x{\mathbf{x}}

\def\Lam{\mathbf{\Lambda}}
\def\lam{\boldsymbol{\lambda}}

\def\A{\mathbf{A}}
\newtheorem{theorem}{Theorem}
\newtheorem{proposition}{Proposition}

\newtheorem{corollary}{Corollary}
\newtheorem{definition}{Definition}

\DeclareMathOperator*{\argmin}{arg\,min}

\begin{document}

\title{Geometric Sparse Coding in Wasserstein Space}


        \author{\hspace{100pt}Marshall Mueller\footnote{Department of Mathematics, Tufts University, Medford, MA 02155, USA (\url{marshall.mueller@tufts.edu}, \url{jm.murphy@tufts.edu}, \url{abiy.tasissa@tufts.edu})}
	\and \quad   Shuchin Aeron\footnote{Department of Electrical and Computer Engineering, Tufts University, Medford, MA 02155, USA (\url{shuchin@ece.tufts.edu})} \hspace{95pt}\vspace{5pt}
 \and \hspace{-0.5em}James M. Murphy\footnotemark[1]  \footnote{Corresponding authors, listed alphabetically} \and \hspace{-20pt} \hspace{-0.3em}Abiy Tasissa\footnotemark[1]  \footnotemark[3]}
        \maketitle

\begin{abstract}
    Wasserstein dictionary learning is an unsupervised approach to learning a collection of probability distributions that generate observed distributions as Wasserstein barycentric combinations. Existing methods for Wasserstein dictionary learning optimize an objective that seeks a dictionary with sufficient representation capacity via barycentric interpolation to approximate the observed training data, but without imposing additional structural properties on the coefficients associated to the dictionary. This leads to dictionaries that densely represent the observed data, which makes interpretation of the coefficients challenging and may also lead to poor empirical performance when using the learned coefficients in downstream tasks. In contrast and motivated by sparse dictionary learning in Euclidean spaces, we propose a geometrically sparse regularizer for Wasserstein space that promotes representations of a data point using only nearby dictionary elements. We show this approach leads to sparse representations in Wasserstein space and addresses the problem of non-uniqueness of barycentric representation. Moreover, when data is generated as Wasserstein barycenters of fixed distributions, this regularizer facilitates the recovery of the generating distributions in cases that are ill-posed for unregularized Wasserstein dictionary learning. Through experimentation on synthetic and real data, we show that our geometrically regularized approach yields sparser and more interpretable dictionaries in Wasserstein space, which perform better in downstream applications.  
\end{abstract}
\section{Introduction}

A central goal of statistical signal processing is the discovery of latent
structures in complex data.  Indeed, although data often resides in a high-dimensional ambient space, the classical \emph{manifold hypothesis} posits that in fact, the data can be well approximated by low-dimensional manifolds or mixtures thereof, which circumvents the curse of dimensionality that plagues high-dimensional statistics. Linear dimensionality reduction methods such as principal component analysis (PCA) \citep{hotelling1933analysis} and non-linear manifold learning approaches that exploit local connectivity structure in the data 
\citep{scholkopf1997kernel, tenenbaum2000global, roweis2000nonlinear, belkin2003laplacian, coifman2006diffusion} rely on this assumption of intrinsically low-dimensional structure in high-dimensional data to glean insights. These techniques typically output a low-dimensional representation preserving local geometric structures such as pairwise distances or geodesic distances with respect to a Riemannian metric. 

An alternative perspective for efficiently representing complex data is the sparse coding and dictionary learning paradigm \citep{olshausen1996emergence, olshausen1997sparse, barlow1961possible,hromadka2008sparse}.  In the simplest setting when data is considered as elements of $\mathbb{R}^{d}$ (or more generally a normed vector space), the aim of sparse coding is to represent data $\{\y_{i}\}_{i=1}^{n}\subset\mathbb{R}^{d}$, stacked as rows in the matrix $\Y \in \R^{n\times d}$, as a linear combination of vectors $\{\mathbf{d}_j\}_{j=1}^{m}$, stacked as a \emph{dictionary} matrix $\mathbf{D}\in \R^{m\times d}$ such that $\Y\approx \mathbf{\Lambda}\mathbf{D}$ for some coefficients $\Lam\in \mathbb{R}^{n\times m}$, perhaps subject to constraints on $\Lam$.  When the dictionary $\mathbf{D}$ is fixed, this reduces to an optimization over $\mathbf{\Lambda}$ \citep{mallat1999wavelet, engan2000multi}.  More generally, $\mathbf{D}$ and $\Lam$ can be learned simultaneously  with some additional constraints on the dictionary or coefficients \citep{lee1999learning, aharon2006k}. This is typically formulated as an optimization problem 
\begin{equation}\label{eqn:DL}\argmin_{\mathbf{D},\Lam}\,\,\mathcal{L}(\Y,\Lam\mathbf{D})+\rho\,\mathcal{R}(\mathbf{D},\Lam),\end{equation}for some loss function $\mathcal{L}$ (e.g., $\mathcal{L}(\Y,\Lam\mathbf{D})=\|\Y-\Lam\mathbf{D}\|_{F}$) and regularization function $\mathcal{R}$ (e.g., $\mathcal{R}(\mathbf{D},\Lam)=\|\Lam\|_{1}$) balanced by a parameter $\rho>0$.  The regularizers ensure well-posedness of the problem and improve interpretability and robustness. 
 The problem (\ref{eqn:DL}) is the \emph{dictionary learning problem} in $\mathbb{R}^{d}$.

The imposed Euclidean structure is convenient computationally but limiting in practice, as many real data are better modeled as living in spaces with non-Euclidean geometry where instead $\Y \approx \mathcal{F}(\mathbf{D}, \mathbf{\Lambda})$ for some nonlinear reconstruction function $\mathcal{F}$ \citep{tuzel2006region,tuzel2007human, li2008visual, guo2010action, harandi2013dictionary, 
harandi2015sparse,cherian2016riemannian,yin2016kernel,maggioni2016multiscale, liu2018kernel, schmitz2018Wasserstein, tankala2020k}. 
 Important questions in this setting are what notion of reconstruction should take the place of linear combination (i.e. $\mathcal{F}$), how reconstruction quality is assessed without the use of a global norm (i.e. $\mathcal{L}$), and what constraints are natural on the coefficients in the nonlinear space (i.e. $\mathcal{R}$).

This paper focuses on nonlinear sparse coding and dictionary learning  for data that are modeled as \emph{probability distributions in Wasserstein space}. 
 This basic framework was pioneered by \citet{schmitz2018Wasserstein}, where the authors leverage the theory and algorithms of optimal transport to propose the \emph{Wasserstein dictionary learning (WDL)} algorithm, whereby a data point (interpreted as a probability distribution or histogram in $\mathbb{R}^{d}$) is approximated as a Wasserstein barycenter \citep{agueh2011barycenters} of the learned dictionary atoms.  The resulting framework is focused on learning a dictionary that reconstructs well, but neglects other desirable aspects of a dictionary such as sparsity of the induced coefficients.  Moreover, this classical scheme is ill-posed in two senses: for a fixed dictionary, unique coefficients are not assured;  moreover, there may be multiple dictionaries that enable perfect reconstruction of the observed data.

\textbf{Summary of Contributions:}  We generalize the classical WDL algorithm \citep{schmitz2018Wasserstein} by incorporating a novel Wasserstein \emph{geometric sparse regularizer}. Our regularizer encourages an observed data point to be reconstructed as a barycentric representation from \emph{nearby} (in the sense of Wasserstein distances) dictionary atoms. As we vary the balance parameter for this regularizer, the proposed method interpolates between classical WDL (no regularization) and Wasserstein $K$-means (strong regularization). Unlike the original formulation, the proposed regularizer learns dictionary atoms with geometric similarity to the training data.  Theoretically, we demonstrate the ability of the model to learn sparse coefficients and to overcome issues of non-uniqueness both at the level of learning coefficients for a fixed dictionary, and for the general WDL problem on an idealized generative model in Wasserstein space.  Empirically, we provide evidence that our regularized dictionary learning scheme yields more interpretable and useful coefficients for downstream classification tasks; the code to reproduce all experiments in this paper will be released on the authors' GitHub.

\textbf{Notations and Preliminaries:} Lowercase and uppercase boldface letters denote (column) vectors and matrices, respectively. We generally use Greek letters to denote measures, with the exception that $\mathcal{D}=\{\mathcal{D}\}_{j=1}^{m}$ denotes the dictionary when its elements are measures. We denote the Euclidean norm of a vector $\x$ as $||\x||_{2}$.  Let\[\Delta^m = \left\{\mathbf{x} \in \R^m\ \bigg| \ \sum_{i=1}^m x_i = 1, \forall i=1,\ldots, m,  x_i \geq 0\right\}\] denote the discrete probability simplex of dimension $m$. Softmax, as a change of variables, is defined as $\cov(\x) := \exp(\x) / \exp(\x)^T \mathbbm{1}_N$.  Here we take the exponential to be an elementwise operation on the vector and use $\mathbbm{1}_N$ to denote the ones vector of size $N$. When we write $\cov(\mathbf{X})$ for some matrix $\mathbf{X}\in \R^{n \times m}$ we take it to mean applying the change of variables to each row.

\section{Background and Related Work}

\textbf{Classical Dictionary Learning:}  In equation (\ref{eqn:DL}), using $\mathcal{L}(\Y,\Lam\mathbf{D})=\|\Y-\Lam\mathbf{D}\|_{F}$ and $\rho=0$ yields an optimization problem with optimal dictionary and coefficients given by the $m$ singular components with largest singular values \citep{eckart1936approximation}.  To promote interpretable, sparse coefficients that still realize $\Y\approx \Lam\mathbf{D}$, the prototypical regularized dictionary learning problem is $\underset{\mathbf{D},\Lam}{\min}\, ||\Y- \Lam\mathbf{D}||_F^2+\rho||\Lam||_{1}$ where $||\Lam||_{1}=\sum_{i=1}^{n}\sum_{j=1}^{m}|\Lambda_{ij}|$ is a sparsity-promoting regularizer \citep{donoho2006compressed, elad2010sparse}.  Beyond enhancing interpretability, efficiency, and uniqueness of representation, sparse representations improve generalization of efficient supervised learning \citep{mehta2013sparsity,mairal2011task}.  In the non-negative matrix factorization (NMF) paradigm, non-negativity constraints are imposed on the atoms and coefficients, which increases their interpretability and effectiveness in downstream applications \citep{lee1999learning, lee2000algorithms,berry2007algorithms}.

\textbf{Optimal Transport:}  We provide basic background on optimal transport; for more general treatments and theory, see \citep{ambrosio2005gradient, villani2021topics,santambrogio2015optimal,peyre2019computational}. 
Let $\Pb(\R^d)$ be the space of probability measures in $\mathbb{R}^{d}$.  Let $\mu, \nu  \in \Pb(\R^d)$.  Let \begin{align*}\Pi(\mu, \nu)=&\{\gamma:\mathbb{R}^{d}\times\mathbb{R}^{d}\rightarrow\mathbb{R} \ | \ \text{for all } A,B \text{ Borel},
 \ \gamma(A\times\mathbb{R}^{d})=\mu(A), \ \gamma(\mathbb{R}^{d}\times B)=\nu(B) \}\end{align*}be the set of joint distributions with marginals $\mu$ and $\nu$. The squared \emph{Wasserstein$-2$ distance} is defined as: 
\begin{equation}\label{eq:W2}
    W_2^2(\mu, \nu) := \min_{\pi \in \Pi(\mu, \nu)} \int_{\R^d \times \R^d} \Vert \x - \y \Vert_2^2\ \pi(\x,\y).
\end{equation}
Given measures $\{\D_j\}_{j=1}^m \subset  \Pb(\R^d)$ that have finite second moments, along with a vector $\boldsymbol{\lambda}\in\Delta^{m}$, the Wasserstein-(2) barycenter \citep{agueh2011barycenters} is defined as:
\begin{equation}
\label{eq:barycenter}
    \bary(\D, \boldsymbol{\lambda}) := \argmin_{\mu \in \Pb(\R^d)} \sum_{j=1}^m \lambda_j W_2^2(\D_j, \mu).
\end{equation}
The measure $\bary(\D, \boldsymbol{\lambda})$ can be interpreted as a weighted average of the $\{\D_j\}_{j=1}^m$, with the impact of $\D_{j}$ proportional to $\lambda_{j}$.  Wasserstein barycenters have proven useful in a range of applications, and are in a precise sense the ``correct" way to combine measures, in that $\bary(\D,\lam)$ preserves the geometric properties of $\D=\{\D_j\}_{j=1}^m$ in a way that linear mixtures do not \citep{agueh2011barycenters, rabin2011wasserstein, cuturi2014fast, bonneel2016wasserstein}.  Wasserstein barycenters are intimately connected to geodesics in Wasserstein space, in the following sense.  For $\pi^{*}$ optimizing (\ref{eq:W2}), the \emph{McCann interpolation} of $\mu,\nu$ is $(P_t)_\# \pi^{*}$ where $P_{t}(\mathbf{x},\mathbf{y}) = (1-t)\mathbf{x} + t \mathbf{y}$ for $t \in [0,1]$ and where $(P_{t})_{\#}$ denotes the pushforward by $P_{t}$.  The McCann interpolation is the constant-speed geodesic between $\mu,\nu$ in the Wasserstein-2 space \citep{mccann1997convexity, ambrosio2005gradient} and coincides with the Wasserstein barycenter with weight $\lam=(1-t, t)$ on $\mu, \nu$.

\section{Geometric Sparse Regularization for Wasserstein Dictionary Learning}
\label{sec:GWDL}
WDL \citep{schmitz2018Wasserstein} aims to find a dictionary of probability distributions $\D = \{\D_{j}\}_{j=1}^{m}\subset \Pb(\R^d)$ such that observed data $\{\mu_i \}_{i=1}^n\subset\Pb(\R^d)$ can be represented as Wasserstein barycenters of the collection $\D$.  The precise optimization problem is
\begin{equation}
\label{eq:WDL}
    \argmin_{\substack{\D \in \Pb(\R^d)^m\\\Lam \in (\Delta^m)^n}} \,\,\sum_{i=1}^n \mathcal{L}(\bary(\D, \lam_i), \mu_i),
\end{equation}
where the loss function $\mathcal{L}$ is typically taken to be $W_2^2$ and $\lam_{i}\in\Delta^{m}$ is a (column) vector of size $m$, corresponding to a row of $\Lam$. In other words, solving this problem finds the dictionary of probability distributions that finds best approximations to each data point $\mu_i$ using barycentric combinations of $\D$.  WDL was proposed in part as an alternative to geodesic principal component analysis in Wasserstein space \citep{boissard2015distribution, seguy2015principal, bigot2017geodesic}, and has proven highly effective in terms of producing meaningful atoms for representing probability distributions.  However, it may yield non-sparse coefficients. 
 Moreover, as we will establish, it is ill-posed both at the level of having non-unique coefficients for a fixed dictionary $\D$ and at the level of having multiple dictionaries that can reconstruct the data perfectly.

Thus, one might consider adding the $\ell_1$ regularizer to the WDL objective to induce desirable sparsity of the representation.  However, this fails because all coefficients of a Wasserstein barycenter lie on $\Delta^{m}$.  Methods to promote sparsity of coefficients on the simplex can be done with entropy, projections, and suitable use of the the $\ell^{2}$ norm \citep{donoho1992maximum, shashanka2007sparse, larsson2011concave, kyrillidis2013sparse, li2020methods}, but we focus instead on geometric regularization.  An analogous problem has been studied in the linear setting in a situation where the weights are on $\Delta^{m}$ \citep{tankala2020k}.  In that context, the \emph{geometric sparse regularizer} $\displaystyle\sum_{j=1}^m \lambda_j ||\y - \mathbf{d}_j ||_2^2$ (for an individual data point $\y$ with representation coefficient $\boldsymbol{\lambda}$) has been proven to promote sparsity by favoring \emph{local representations}, namely reconstructing using nearby (with respect to Euclidean distances) dictionary atoms.  

We propose to regularize (\ref{eq:WDL}) with a novel \emph{Wasserstein geometric sparse regularizer}:
\begin{equation}\label{eqn:R}
    \mathcal{R}(\D, \Lam) := \sum_{i=1}^n \sum_{j=1}^m (\lam_i)_j W_2^2(\D_j, \mu_i).
\end{equation}
This yields a new, regularized objective:
\begin{align}
\label{eqn:GWDL}    \mathcal{F}(\D, \Lam, \{\mu_i\}_{i=1}^n) :=& \min_{\substack{\D \in \Pb(\R^d)^m\\\Lam \in (\Delta^m)^n}} \,\,\sum_{i=1}^n W_2^2(\bary(\D, \lam_i), \mu_i) + \rho \sum_{i=1}^n \sum_{j=1}^m (\lam_i)_j W_2^2(\D_j, \mu_i), 
\end{align}
where $\rho>0$ is a tuneable balance parameter.  

Note that the geometric sparse regularizer (\ref{eqn:R}) resembles the objective in the definition of the Wasserstein barycenter in \eqref{eq:barycenter}. Crucially, in the unregularized formulation of the WDL problem, the distance from the atoms to the data points is only indirectly minimized, because the WDL formulation focuses solely on the reconstruction accuracy of the generated barycenters and not directly about how similar those atoms are to the data points. As such, the atoms may be arbitrarily far from the data provided that the reconstruction accuracy is low. As discussed below, this may cause the existence of arbitrarily many solutions to the WDL problem (\ref{eq:WDL}), which in general is ill-posed.

\textbf{Interpretations of $\mathcal{R}(\D, \Lam)$:}  The regularization term $\mathcal{R}(\D, \Lam)$ is analogous to Laplacian smoothing in Euclidean space \citep{cai2010graph, dornaika2019sparse} and can be interpreted as non-linear archetypal learning  \citep{cutler1994archetypal} in Wasserstein space.  As previously mentioned, there is no notion of sparse coding in the unregularized WDL formulation. The geometric sparse regularizer promotes sparsity by penalizing the use of atoms that are far from the data to be represented and thus acts as a weighted $\ell_1$ norm on the coefficients \citep{tasissa2021weighed}.

\textbf{Connection with Wasserstein $K$-means:} In this problem \citep{domazakis2019clustering, verdinelli2019hybrid, zhuang2022wasserstein}, given observed measures $\{\mu_{i}\}_{i=1}^{n}\subset\Pb(\mathbb{R}^{d})$ we want to find ``centers" $\D=\{\D_{j}\}_{j=1}^{m}\subset\Pb(\mathbb{R}^{d})$ solving the optimization problem
\[\underset{\mathbf{C}, \D}{\min}\,\,\sum_{i=1}^{n} \sum_{j=1}^{m}    C_{ij} W_2^2(\D_j, \mu_i),\]
where $\mathbf{C}\in\mathbb{R}^{n\times m}$ is such that $C_{ij} \in \{0,1\}$ is a binary assignment variable satisfying $\sum_{j=1}^{m} C_{ij}=1$ for all $i=1,\dots,n$.  Suppose $\D^* \in \Pb(\R^d)^m$ and $\Lam^* \in \R^{n\times m}$ are the optimizers of (\ref{eqn:GWDL}). Note that for \emph{any} feasible $\Lam$ and with dictionary fixed at $\D^{*}$, we have
\begin{align*}
\mathcal{R}(\D^*, \Lam) =&\sum_{i=1}^n \sum_{j=1}^m (\lam_i)_j W_2^2(\D^*_j, \mu_i)\\
\geq &\sum_{i=1}^n \sum_{j=1}^m (\lam_i)_j \underset{1\le p\le m}{\min} W_2^2(\D^*_p, \mu_i)\\ =& \sum_{i=1}^n  \underset{1\le p\le m}{\min} W_2^2(\D^*_p, \mu_i).
\end{align*}
Thus, for fixed $\D^{*}$, coefficients that minimize $\mathcal{R}(\D^*, \Lam)$ have the property that the $i^{th}$ row is all zeros except for a $1$ at the $i^*$ position where $i^* = \underset{1\le p\le m}{\argmin}\,\,W_2^2(\D^*_p, \mu_i)$. In this sense, for a fixed dictionary $\D^{*}$ and with each observation $\mu_{1},\dots,\mu_{n}$ having a unique nearest neighbor in $\D^{*}$, the optimal solution $\Lam^*$ is a matrix whose rows are binary and $1$-sparse, which is exactly of the same form as the binary assignment in Wasserstein $K$-means. When the aforementioned assumption does not hold, uniqueness is not guaranteed but the $1$-sparse solution is in the family of optimal solutions. 

In this sense, incorporating the geometric sparse regularizer (\ref{eqn:R}) into the main objective in \eqref{eq:WDL} with a scaling parameter $\rho$ enables interpolation between learning a dictionary for pure reconstruction $(\rho=0)$ and one with sparsity promoted via $K$-means $(\rho\gg 0)$.  Indeed, the optimization \eqref{eqn:GWDL} is like a \emph{soft Wasserstein $K$-means}, in that it promotes assigning coefficient energy to a single, closest atom.

\textbf{Sparse Coding for Fixed Dictionary:}  For a fixed dictionary $\{\D_j\}_{j=1}^m \subset  \Pb(\R^d)$ and a target measure $\mu$ , we consider the following problem:
\begin{align}\label{eqn:SparseCode}
    &\underset{\lam \in \Delta^m}{\arg\min}\,\, \sum_{j=1}^{m} \lambda_{j} W_2^2(\D_j, \mu) \,\, \text{subject to}  \,\,\mu = \bary(\D, \lam).
\end{align}
The above problem is a sparse coding problem under the constraint that $\mu$ is \emph{exactly reconstructed} in the sense of Wasserstein barycenters. The \emph{barycentric coding model} analyzed in \citet{werenski2022measure} gives a characterization of when $\mu = \bary(\D, \boldsymbol{\lambda})$, which can be leveraged to characterize the sparse coding step of our Wasserstein dictionary learning problem as follows; a precise statement with explicit regularity assumptions and proof appear in the Supplementary Materials.

\begin{proposition}\label{prop:LP}Let $\mu$ be fixed and let $\{\D_j\}_{j=1}^m \subset  \Pb(\R^d)$ be a fixed dictionary.  Under suitable regularity assumptions on $\mu$ and $\{\D_j\}_{j=1}^m$, the solution to (\ref{eqn:SparseCode}) is given by \begin{equation}\label{eqn:SparseCodeLP}
    \underset{\lam \in \Delta^m}{\arg\min}\,\,  \lam^T\mathbf{c} \,\, \,\emph{subject to} \,\,\, \A\lam = \mathbf{0}, 
    \end{equation}
where $\mathbf{c}$ and $\A\in\mathbb{R}^{m\times m}$ are uniquely determined by $\mu,\{\D_j\}_{j=1}^m$.
    
\end{proposition}

Importantly, $\mathbf{c}$ and $\A$ are determined by $\{\D\}_{j=1}^{m}$ for a fixed $\mu$, so that (\ref{eqn:SparseCodeLP}) is a linear program in $\lam$.  In general for fixed $\D$ and $\mu$, the problem of solving for $\lam$ satisfying \[\A\lam=\mathbf{0} \, \text{ such that } \lam\in\Delta^{m}\] may have multiple solutions \citep{werenski2022measure}.  Among all the possible barycentric representations of $\mu$, (\ref{eqn:SparseCodeLP}) chooses the one ``closest" to the dictionary atoms themselves, and thereby promotes uniqueness at the level of sparse coding and under the hard reconstruction constraint.

\textbf{$\mathcal{R}(\D,\Lam)$ Promotes Unique Solutions to WDL:}
The unregularized WDL problem (\ref{eq:WDL}) does not in general have a unique solution.  This can be seen intuitively in the case where the data are generated as barycenters of two measures $\mu, \nu\in\Pb(\mathbb{R}^{d})$. In this case, any barycenter coincides with a point along the McCann interpolation: $\bary(\{\mu,\nu\},(1-t,t))=(P_t)_{\#} \pi^{*}$ where $\pi^{*}$ is the optimal transport plan between $\mu,\nu$.  Then any measures $\tm, \tn$ whose McCann interpolation passes through $\mu,\nu$
 will also generate any barycenters of $\mu,\nu$.  This is visualized in Figure \ref{fig:locality_visualization}.  We will show that in this special case, our geometric sparse regularizer (\ref{prop:LP}) addresses this ill-posedness of WDL.

Note, this is an issue of \emph{non-uniqueness over dictionaries} $\D$; the simpler issue of \emph{non-uniqueness for a fixed $\D$} is analyzed in Proposition \ref{prop:LP}.  Indeed, the non-uniqueness for a fixed $\D$ is characterized \citep{werenski2022measure} by the solution space to $\A\lam=\mathbf{0}$ intersecting $\Delta^{m}$ in multiple places.  On the other hand, our analysis of non-uniqueness over dictionaries requires an analysis of McCann interpolations.

\begin{definition}\label{def:geodesic-containment}
Let $\mu, \nu \in \Pb(\mathbb{R}^{d})$ have optimal transportation plan $\pi^{*}$ and $\tm, \tn\in \Pb(\mathbb{R}^{d})$ have optimal transportation plan $\tpi^{*}$.  The measures $\tm, \tn$ are said to \emph{contain the McCann interpolation} $\{(P_t)_{\#}\pi\}_{t\in [0,1]}$ between $\mu$ and $\nu$ if there exists an interval $[a,b]\subset[0,1]$ such that $\forall t \in [0,1],\ \exists s\in [a,b]$ such that $  (P_{s})_\#\tpi = (P_t)_\# \pi$. 
\end{definition}
We define the set of all pairs of measures $(\tm, \tn)$ that contain the McCann interpolation between $\mu$ and $\nu$ as $M(\mu, \nu)$.  Pairs of measures in $M(\mu, \nu)$ can be thought of as generators of ``extensions" of the McCann interpolation between $\mu, \nu$.  In this sense, barycenters of $(\tm,\tn)$ can perfectly reconstruct any barycenter of $(\mu,\nu)$ if and only if $(\tm,\tn)\in M(\mu,\nu)$.  We show that any ``extension'' of the McCann interpolation from one side results in an increase in the geometric sparse regularizer for any measure in the original interpolation.  The proof, which depends on the geodesic properties of McCann interpolation \citep{ambrosio2005gradient}, appears in the Supplementary Materials.

\begin{figure}[t!]
    \centering
    \includegraphics[width=0.6\linewidth]{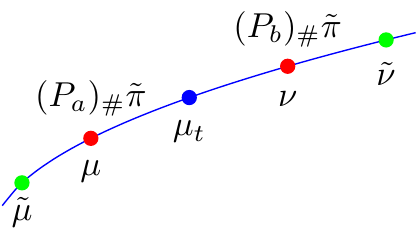}
    \caption{The measures $\tm$ and $\tn$ have the capacity to represent any barycenter $\mu_{t}$ of $\mu=(P_{a})_{\#}\tilde{\pi}$, $\nu=(P_{b})_{\#}\tilde{\pi}$, but they do it in a manner that our geometric sparse regularizer penalizes.}
    \label{fig:locality_visualization}
\end{figure}

\begin{theorem}\label{thm:geodesic-extension}
Consider measures $\mu, \nu, \tn$ with respective optimal transportation plans $\pi$ and $\tpi$ and suppose $(\mu, \tn)\in M(\mu,\nu)$. Let $\mu_{t}=(P_t)_{\#}\pi$ be in the McCann interpolation between $\mu$ and $\nu$, and let $s$ be the associated time coordinate such that $(P_{s})_{\#}\tilde{\pi}=(P_t)_{\#}\pi$. Then
\begin{align*} \label{eqn:geodesic-extension}
    (1-t) W_2^2(\mu, \mu_t) + t W_2^2(\nu, \mu_t)
    \leq (1-s)W_2^2(\mu, \mu_t) + s W_2^2(\tn, \mu_t).
\end{align*}
\end{theorem}

With this, we can establish that subject to the constraint of perfect reconstruction (quantified by $M(\mu,\nu)$), minimizing the geometric sparse regularizer yields a unique solution that coincides with the true generating atoms in the case of all observed data lying on a McCann interpolation.

\begin{corollary}\label{cor}
Let $\mu \neq \nu$ be two measures with optimal transport plan $\pi$.  For any $(\tm, \tn)\in M(\mu,\nu)$, let $\tilde{\pi}$ be the associated optimal transport plan.  Then for any barycenter $\mu_{t}=(P_{t})_{\#}\pi$ generated by $\mu, \nu$,
\begin{align*}
   (\mu,\nu)=\argmin_{\substack{(\tm, \tn) \in M(\mu, \nu),\\ s \ s.t. (P_{s})_{\#}\tilde{\pi}=\mu_{t}}}(1-s) W_2^2(\tm, \mu_t) + s W_2^2(\tn, \mu_t).
\end{align*}

\end{corollary}

\begin{proof}
Apply Theorem \ref{thm:geodesic-extension} twice, the second time after reversing parameterization.
\end{proof}
In other words, among all pairs of measures that contain the geodesic between the original data generators $\mu$ and $\nu$, the geometric sparse regularizer for every data point is minimized by $(\mu,\nu)$.  

\section{Proposed Algorithm}
Optimization in Wasserstein space has been infeasible outside of problems with low-dimensional distributions due to the computational complexities of solving the transport problems \citep{bonneel2016wasserstein}. We make two primary design choices to make a tractable algorithm: 

\paragraph{Shared Fixed Support:} We assume all measures lie on the same fixed finite support $\{\x_i\}_{i=1}^N \subset \R^d$. So, each distribution $\mu$ can be represented as a probability distribution $\mathbf{a} \in \Delta^N$ via $\mu = \sum_{i=1}^N a_i \delta_{\x_i}$. We will abuse notation and write $\mu$ in place of $\mathbf{a}$ when referring to discrete measures. Having a fixed support enables us to compute the pairwise costs $\| \x_{i} - \x_{j} \|_2^2$ upfront, which are used repeatedly in the transport and barycenter computations. 

\paragraph{Entropic Regularization:} We use the entropically regularized Wasserstein distance \citep{cuturi2013sinkhorn,peyre2019computational}. Simply stated, we can get simple and relatively cheap estimates of the Wasserstein distance by a few iterations of the Sinkhorn matrix scaling algorithm. We refer the reader to the aforementioned references for detailed discussion of the entropic regularization and it's effect on computation; in brief it is repeated matrix-vector multiplications done with the kernel matrix $\mathbf{K} := \exp(-\mathbf{C}/\varepsilon)$ where $\mathbf{C} \in \R^{N \times N}$ is the pairwise cost matrix formed from the support. Unsurprisingly, entropic regularization lends to a similar Sinkhorn-like iterative method to compute barycenters \citep{benamou2015iterative}; both methods can be derived as 
Bregman iterations. We make extensive use of automatic differentiation \citep{paszke2017automatic} to handle variable updates. For all transport computations one could also use the unbiased Sinkhorn divergences instead \citep{feydy2019interpolating}; we choose not to in order to compare directly to WDL. 

Our main algorithm, which we call \emph{Geometrically Sparse Wasserstein Dictionary Learning (GeoSWDL)}, is detailed in Algorithm \ref{alg:WDL-l}. 
Following the original WDL formulation of \citet{schmitz2018Wasserstein}, we optimize over arbitrary vectors in Euclidean space that each represent a unique probability distribution (both for atoms and barycentric weights) via softmax as a change of variables. 

\begin{algorithm}
    \caption{Geometrically Sparse Wasserstein Dictionary Learning (GeoSWDL)}
    \label{alg:WDL-l}
    \begin{algorithmic}[1] 
        \State \textbf{Input:} Training data $\{\mu_i\}_{i=1}^n\subset \Delta^N$, $L$, $\verb|optimizer|$
        \State Initialize variables $\boldsymbol{\alpha}^{(0)} \in \R^{n \times m},\ \boldsymbol{\beta}^{(0)} \in \R^{m \times N}$ \Comment{Use any initialization method}
        \For{$k\gets 1,\ldots,  L$}
            \State \verb|loss| $\gets 0$ \Comment{Set loss to zero for each iteration}
            \State $\D^{(k)} \gets \cov(\boldsymbol{\alpha})$, $\Lam^{(k)} \gets \cov(\boldsymbol{\beta})$ \Comment{Get updated dictionary/weights}
            \State \verb|loss| $\gets \mathcal{F}(\D^{(k)}, \Lam^{(k)}, \{\mu_i\}_{i=1}^n )$ \Comment{Compute the objective function}
            \State \verb|loss.backward()| \Comment{Compute the gradients with automatic differentiation}
            \State Update $\boldsymbol{\alpha}^{(k)}, \boldsymbol{\beta}^{(k)}$ with \verb|optimizer.step()| \Comment{Update variables}
        \EndFor
        \State \textbf{Output:} $\D^{(L)}, \Lam^{(L)}$

    \end{algorithmic}
\end{algorithm}

We briefly describe initialization options for the algorithm:

\paragraph{Atom Initialization:} We consider 3 methods of initialization for the atoms: (i) uniform at random samples from $\Delta^N$; (ii) uniform at random data samples: pick $m$ of the data used to learn the dictionary as the initialization; (iii) $k$-means++ initialization: follow the initialization procedure of $k$-means++ algorithm using Wasserstein distances \citep{arthur2006k} and use those choices as the initial atoms\footnote{To do the initialization properly one needs distances to be nonnegative - Sinkhorn ``distances'', as approximated by the value of the entropically regularized problem, may be negative (the more accurate estimate, computing the transport cost without the entropy term, does not have this problem, but is more expensive to compute ($O(N^2)$ compared to $O(N)$; entropy provides the efficient computation via the dual problem). We crudely work around this by adding the smallest number to the distances to make them positive.}.  We expect the data-based initialization schemes to converge faster and to better solutions, particularly with geometric sparse regularizer since the probability distributions that resemble the data will be favored, as we assume generating probability distributions should resemble the data to some degree. 

\paragraph{Weight Initialization:} We consider 3 methods of initialization for the weights: (i) uniform at random samples from $\Delta^m$; (ii) Wasserstein histogram regression \citep{bonneel2016wasserstein} to match each data point to the initialized atoms; (iii) estimating weights using the quadratic program described by \cite{werenski2022measure}; details of this approach are in the Supplementary Materials.  Empirically, atom initialization was more important than the choice of weight initialization; we use method (i) for all experiments in Section \ref{sec:Experiments}.
\paragraph{Runtime Complexity:} To understand the complexity of our algorithm we focus on the complexity of a single iteration in the loop of Algorithm \ref{alg:WDL-l}. We run a fixed number of Sinkhorn iterations, $L_s$, for both entropic distance computations as well as entropic barycenter computations. Thus, for a single data point the cost to compute the barycenter is $O(L_sN^2m)$ as each iteration of the Bregman projection is a matrix multiplication with the kernel. Similarly the computational cost of the loss is $O(L_sN^2 + L_sN^2m) = O(L_sN^2m)$ (entropic distance between barycenter and true data point plus geometric sparse regularizer). With automatic differentiation the cost is the same as computing the values. Thus when considering all the data the cost per iteration is $O(nL_sN^2m)$. One could batch the data for simpler iterations, but we do not evaluate the effects of such a choice. The bottleneck cost is the matrix-vector operations in the Sinkhorn iterations; to circumvent this, particularly for larger problems, one could form a Nystr\"om approximation of the kernel matrix \citep{altschuler2019massively}. Geometric structure in the support, as is the case with data supported on a grid (e.g., images), can be leveraged to reduce computational burdens \citep{solomon2015convolutional}. 

\section{Experiments}
\label{sec:Experiments}
This section summarizes experiments on image and NLP data; further discussion is in the Supplementary Materials.
\subsection{Identifying Generating Distributions With Synthetic MNIST Data} 

We demonstrate the utility of our geometric sparse regularizer in identifying the generating probability distributions in a generative model.  In this experiment, we randomly select 3 samples from each MNIST \citep{lecun1998mnist} data class \{0,1,2,\dots,9\}.  For each of the classes, we use the 3 samples to generate 50 synthetic samples by forming barycenters constructed with weights sampled uniformly from $\Delta^3$.  This yields 500 total training points, each of which is a synthetic MNIST digit generated from a total of $3 \times 10=30$ real MNIST digits.  We then train a dictionary to learn 30 atoms.  An optimal solution is to learn the original generating set of 30 distributions.  

We run Algorithm \ref{alg:WDL-l} for $L=250$ iterations and use the Adam optimizer learning rate of 0.25 (other parameters left as default in PyTorch). For all Sinkhorn computations, we used 50 iterations. We compare the effects of geometric sparse regularization by varying the regularization parameter $\rho\in\{10^{-3},10^{-1},10^{1}\}$.  Atoms are initialized using the K-means++ initialization.  After learning the dictionary, we match the learned atoms to the true atoms by finding the assignment that minimizes transport cost between learned and true atoms.  We solve the non-entropically regularized transport problem to ensure non-negative assignment costs. We visualize the learned vs. true atoms as well as a confusion matrix that demonstrates how well learned coefficients and atoms align with the true class in Figure \ref{fig:mnist}. 

As observed empirically in Figure \ref{fig:mnist} (b), increasing $\rho$ only helps the learned coefficients to be placed more correctly on their class after matching the atoms. This illustrates a trade-off associated with increasing $\rho$: more geometric sparse regularization may help the data be classified correctly, but the learned atoms may not resemble the true generating atoms as well. Indeed, in Figure \ref{fig:mnist} (c) we see that increasing $\rho$ concentrates the coefficients, as we would expect given the connection between GeoSWDL with large $\rho$ and Wasserstein $K$-means discussed in Section \ref{sec:GWDL}.  We also note that as seen in Figure \ref{fig:mnist} (a), some of learned atoms for classes $\{1,2,6\}$ are less visually interpretable; this is due in part to the fact that during the optimization process, none of the training data use these atoms for their barycentric reconstruction. As such, the algorithm loses information about how to optimize those data points---they are effectively never updated because they are never used in reconstruction. A simple workaround we use is to restart the algorithm several times and select an output that places some weight on each atom (i.e. $\forall j, \exists i$ such that  $(\lam_i)_j \gg 0$).

\begin{figure}[htbp!]
\begin{subfigure}{\textwidth}
        \includegraphics[width=\linewidth]{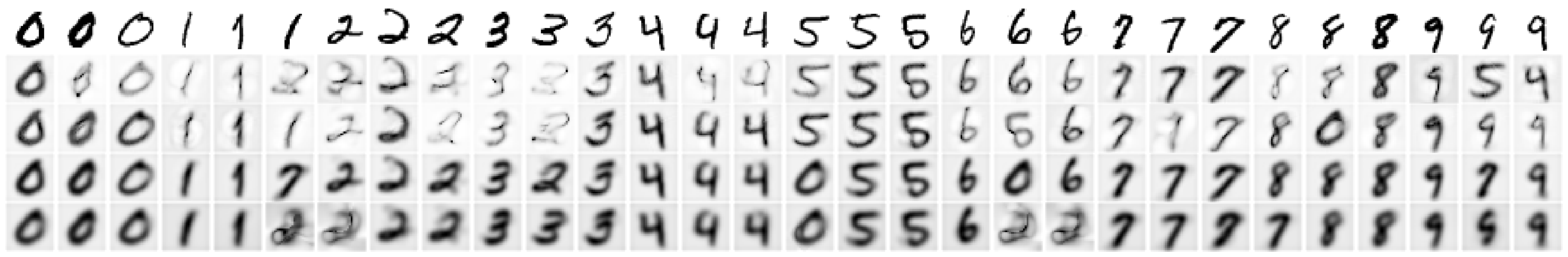}
        \subcaption{Top row shows the true generating probability distributions.  Subsequent rows show learned atoms with increasing $\rho$, after alignment.}
    \end{subfigure}
    \begin{subfigure}{\textwidth}
        \includegraphics[width=\linewidth]{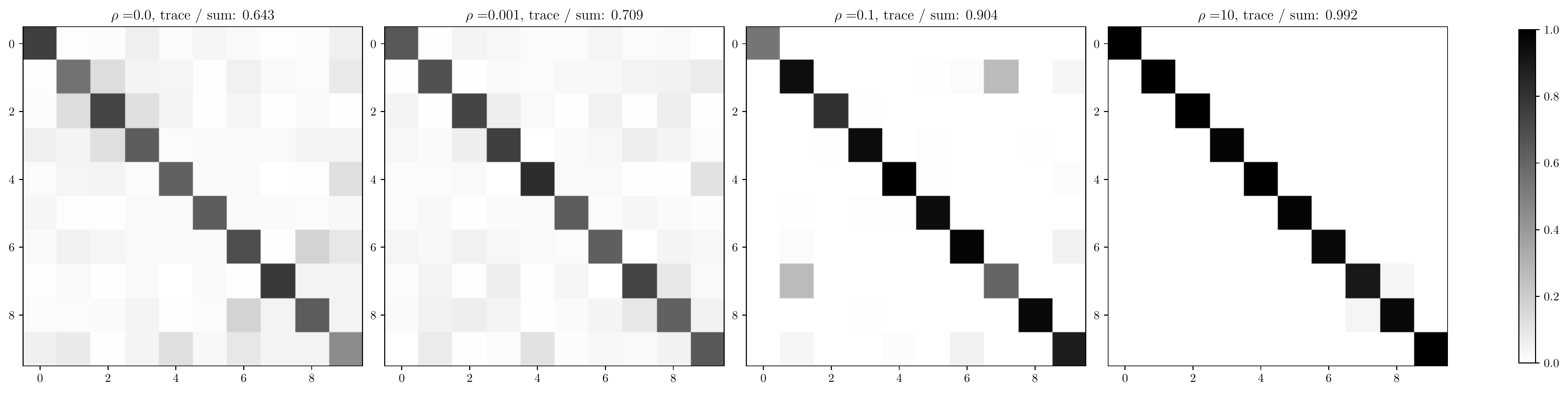}
        \subcaption{The amount of mass assigned to weights of class $i$ to atoms of class $j$}
    \end{subfigure}
    \begin{subfigure}{\textwidth}
        \includegraphics[width=\linewidth]{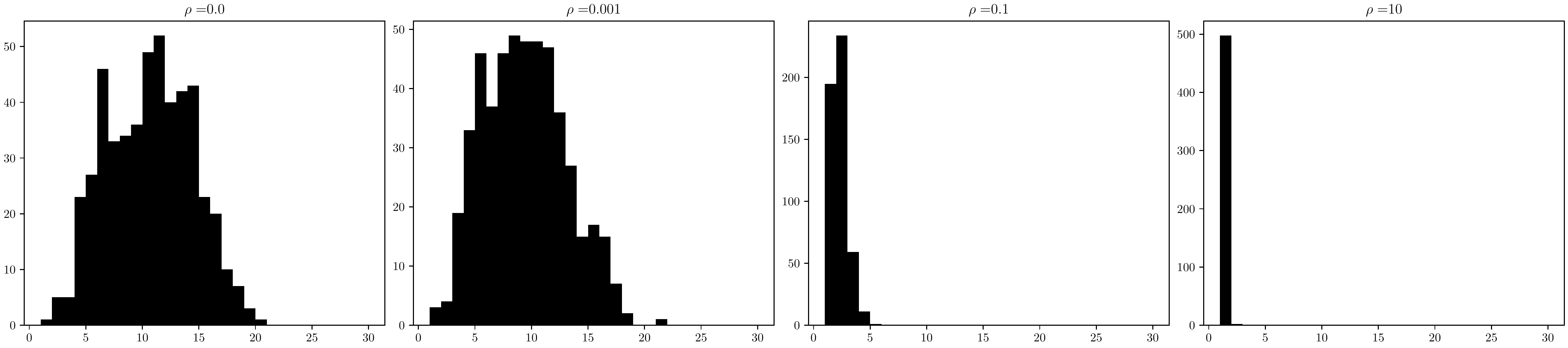}
        \subcaption{The counts of minimal number of atoms needed to assign at least 95\% of the mass in the barycenter as a proxy for sparsity}
    \end{subfigure}
    \caption{The learned digits shown are most interpretable in the moderate regularization regime (e.g., $\rho>0$ but not too large).  As $\rho$ increases, (b) shows the concentration of coefficient energy on atoms that belong to the same class as the test data point, with atoms assigned by finding their closest training data.  Moreover, we see in (c) that increasing $\rho$ increases the sparsity of the coefficients learned.}
    \label{fig:mnist}
\end{figure}

\subsection{Learned Coefficients as Features for Document Classification}

Here we demonstrate the effectiveness of the geometric sparse regularizer when used for a classification task comparing word documents. We represent documents as probability distributions with a bag-of-words (i.e. a vector of counts for each word in the document normalized to lie in the probability simplex) approach where we use learned embeddings of the words from \cite{huang2016supervised} as the support. The documents considered come from the BBCSPORT dataset which consists of 737 documents with 5 classes overall.  We follow the experimental setup of \cite{werenski2022measure}, where a fixed number of \emph{random} reference documents from each class were used.  We instead \emph{learn} a fixed size dictionary for each reference class. 
The documents used for training the dictionary are disjoint from the randomly selected documents used for comparison. We train each dictionary with a number of samples proportional to the size of the dictionary; for a dictionary of $m$ atoms we use $4m$ training samples. Since the dictionary atoms must have fixed support, we fix the support as the union of words present in the training documents. 

For evaluation, we sample 100 random documents disjoint from the dictionary training sets and the randomly selected reference documents for baseline comparison, and classify them using the methods described below. Both the dictionaries and sets of random reference documents have the same number of elements per class and this number is varied from 1 to 12. We repeat this test 30 times and report results averaged across these trials. The results are visualized in Figures and \ref{fig:nlp_fixedmethod} and \ref{fig:nlp_fixedlocality}.  

We note that these plots compare against two baselines in terms of what generating measures to use: (i) random samples from the data; (ii)  WDL, which corresponds to $\rho=0$ in the GeoSWDL framework.  

\textbf{Classification Methods:}  We state the 5 classification methods we consider on the learned coefficients.
\begin{enumerate}[itemsep=-2pt]
\item \textbf{1-Nearest Neighbor (1NN):} classifies based on the class of the nearest reference document in $W_2^2$; 

\item \textbf{Minimum Average Distance (MAD):} selects the class with reference documents on average $W_2^2$-closest to the test document;

\item \textbf{Minimum Barycentric Loss (MBL):} we learn the barycentric weights to represent the test document by solving the quadratic program in \cite{werenski2022measure} for each reference class. We then compute the barycentric representation for each class and classify with the one that minimizes $W_2^2$ to the test document;

\item \textbf{Minimum Barycenter Loss (MBL-QP):} selects the class that minimizes the aforementioned quadratic program's objective value (a proxy for the $W_2^2$ distance between the barycentric representation and test document); 

\item \textbf{Maximum Coordinate (MC):} also solves the aforementioned quadratic program to estimate barycentric weights when using the reference documents \emph{of all} classes to represent the test document. The class is then assigned based on the class whose total portion of the estimated weights is largest.

\end{enumerate}
We notice that at all levels of geometric sparse regularizer tested, learned dictionaries enable the barycenter focused methods to outperform all methods that use random samples of the data. Perhaps more interestingly, we note that increasing levels of geometric sparse regularization increases the performance of the simpler methods of 1NN and MAD. This suggests that the geometric sparse regularizer is promoting dictionary atoms that are more informative generally \emph{as individual atoms} as opposed to the information contained in their collective representational capacity. As mentioned in the discussion of the geometric sparse regularizer in Section \ref{sec:GWDL}, the unregularized WDL objective is minimized by the dictionary probability distributions only with respect of their representative ability. On the other hand, as demonstrated here, GeoSWDL encourages the learned atoms to be informative themselves about the data they model, even generalizing to unseen data.

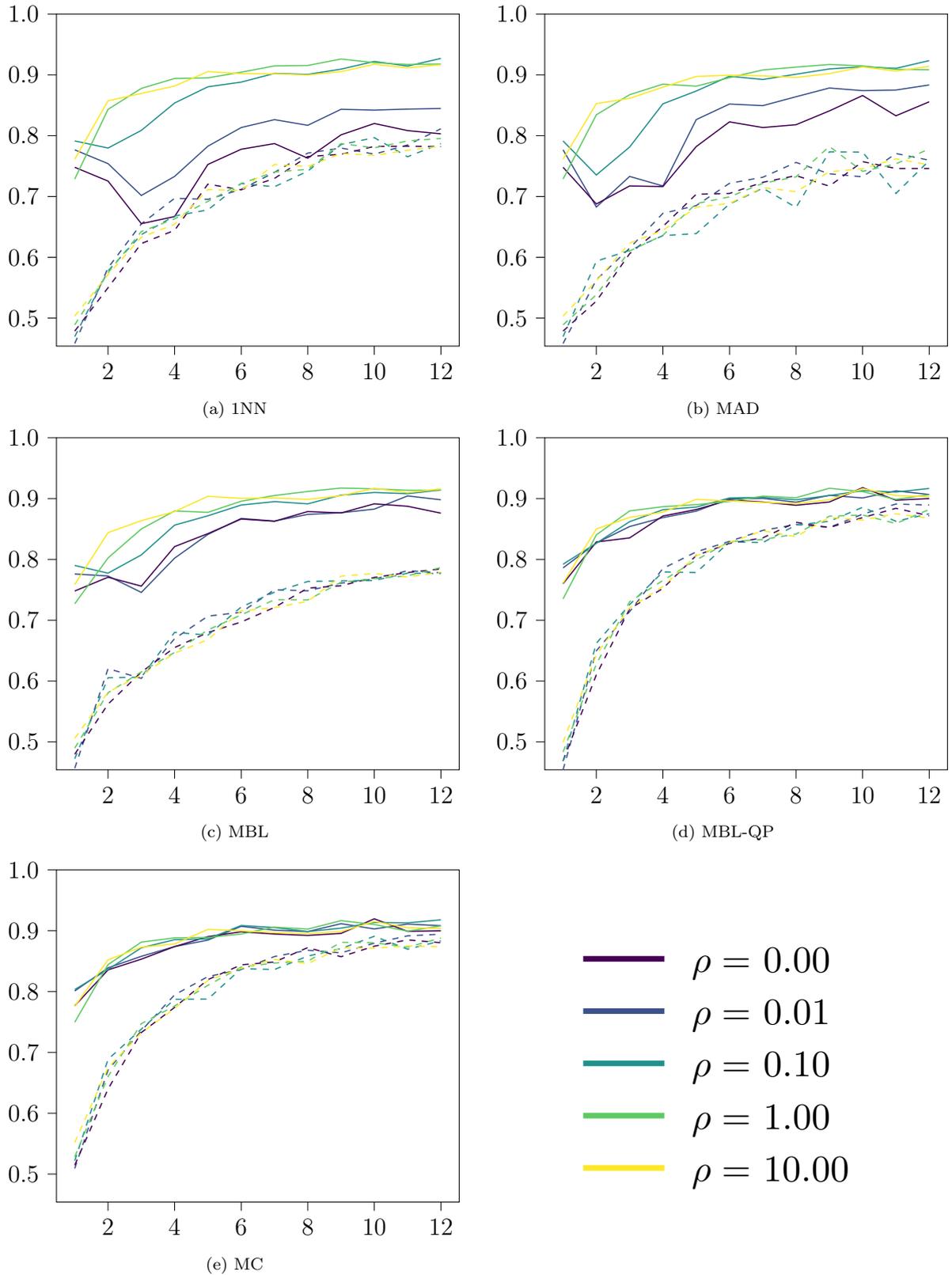
\begin{figure}[htbp!]
    \begin{subfigure}{0.49\linewidth}
    \centering
    \resizebox{\linewidth}{!}{
\begin{tikzpicture}

\definecolor{darkcyan32144140}{RGB}{32,144,140}
\definecolor{darkgray176}{RGB}{176,176,176}
\definecolor{darkslateblue5882139}{RGB}{58,82,139}
\definecolor{gold25323136}{RGB}{253,231,36}
\definecolor{indigo68184}{RGB}{68,1,84}
\definecolor{mediumseagreen9420197}{RGB}{94,201,97}

\begin{axis}[
tick align=outside,
tick pos=left,
x grid style={darkgray176},
xmin=0.45, xmax=12.55,
xtick style={color=black},
xtick={2,4,6,8,10,12},
xticklabels={
  \(\displaystyle {2}\),
  \(\displaystyle {4}\),
  \(\displaystyle {6}\),
  \(\displaystyle {8}\),
  \(\displaystyle {10}\),
  \(\displaystyle {12}\)
},
y grid style={darkgray176},
ymin=0.453999996185303, ymax=1,
ytick style={color=black},
ytick={0.5,0.6,0.7,0.8,0.9,1},
yticklabels={
  \(\displaystyle {0.5}\),
  \(\displaystyle {0.6}\),
  \(\displaystyle {0.7}\),
  \(\displaystyle {0.8}\),
  \(\displaystyle {0.9}\),
  \(\displaystyle {1.0}\)
}
]
\addplot [semithick, darkslateblue5882139, dashed]
table {%
1 0.458333343267441
2 0.582666635513306
3 0.654999971389771
4 0.696999967098236
5 0.695333302021027
6 0.711666643619537
7 0.739666640758514
8 0.771333336830139
9 0.779666662216187
10 0.76966667175293
11 0.784666657447815
12 0.811333358287811
};
\addplot [semithick, darkslateblue5882139]
table {%
1 0.776666700839996
2 0.754000008106232
3 0.701666653156281
4 0.732999980449677
5 0.782666623592377
6 0.813333332538605
7 0.826333224773407
8 0.816999971866608
9 0.84333336353302
10 0.842000007629395
11 0.843666672706604
12 0.844666659832001
};
\addplot [semithick, indigo68184, dashed]
table {%
1 0.478666663169861
2 0.550000011920929
3 0.622666597366333
4 0.644333302974701
5 0.720333218574524
6 0.711000025272369
7 0.729666650295258
8 0.765666663646698
9 0.769333302974701
10 0.782333314418793
11 0.782666623592377
12 0.781999945640564
};
\addplot [semithick, indigo68184]
table {%
1 0.747999966144562
2 0.725333273410797
3 0.655333340167999
4 0.666999936103821
5 0.752666652202606
6 0.777666628360748
7 0.786999940872192
8 0.762999951839447
9 0.801333367824554
10 0.819999933242798
11 0.808333337306976
12 0.803000032901764
};
\addplot [semithick, darkcyan32144140, dashed]
table {%
1 0.469666659832001
2 0.577666640281677
3 0.637333273887634
4 0.668000042438507
5 0.678000032901764
6 0.72133332490921
7 0.715999960899353
8 0.741666674613953
9 0.785333335399628
10 0.796666741371155
11 0.765333354473114
12 0.787000000476837
};
\addplot [semithick, darkcyan32144140]
table {%
1 0.791333377361298
2 0.779666662216187
3 0.80866664648056
4 0.853666663169861
5 0.880333244800568
6 0.887999951839447
7 0.902333319187164
8 0.90066659450531
9 0.909333288669586
10 0.921999931335449
11 0.914333343505859
12 0.926999986171722
};
\addplot [semithick, mediumseagreen9420197, dashed]
table {%
1 0.488999992609024
2 0.574666678905487
3 0.642333269119263
4 0.663333296775818
5 0.691999971866608
6 0.721333265304565
7 0.740333318710327
8 0.744666695594788
9 0.787000000476837
10 0.779999971389771
11 0.791333317756653
12 0.79533326625824
};
\addplot [semithick, mediumseagreen9420197]
table {%
1 0.728666663169861
2 0.843333303928375
3 0.877666652202606
4 0.893999934196472
5 0.894999921321869
6 0.904333353042603
7 0.914666593074799
8 0.915333330631256
9 0.925999939441681
10 0.920000016689301
11 0.916999995708466
12 0.917666614055634
};
\addplot [semithick, gold25323136, dashed]
table {%
1 0.503333330154419
2 0.571666657924652
3 0.633000016212463
4 0.654333293437958
5 0.711333274841309
6 0.709999978542328
7 0.752666592597961
8 0.749333322048187
9 0.770666658878326
10 0.767666637897491
11 0.775999963283539
12 0.782666623592377
};
\addplot [semithick, gold25323136]
table {%
1 0.761333346366882
2 0.857333302497864
3 0.869333267211914
4 0.881666600704193
5 0.905333280563354
6 0.90200001001358
7 0.901666641235352
8 0.899666607379913
9 0.905333340167999
10 0.91733330488205
11 0.91133326292038
12 0.916666626930237
};
\end{axis}

\end{tikzpicture}}
    \caption{1NN}
\end{subfigure}
\hfill
    \begin{subfigure}{0.49\linewidth}
    \centering
    \resizebox{\linewidth}{!}{
\begin{tikzpicture}

\definecolor{darkcyan32144140}{RGB}{32,144,140}
\definecolor{darkgray176}{RGB}{176,176,176}
\definecolor{darkslateblue5882139}{RGB}{58,82,139}
\definecolor{gold25323136}{RGB}{253,231,36}
\definecolor{indigo68184}{RGB}{68,1,84}
\definecolor{mediumseagreen9420197}{RGB}{94,201,97}

\begin{axis}[
tick align=outside,
tick pos=left,
x grid style={darkgray176},
xmin=0.45, xmax=12.55,
xtick style={color=black},
xtick={2,4,6,8,10,12},
xticklabels={
  \(\displaystyle {2}\),
  \(\displaystyle {4}\),
  \(\displaystyle {6}\),
  \(\displaystyle {8}\),
  \(\displaystyle {10}\),
  \(\displaystyle {12}\)
},
y grid style={darkgray176},
ymin=0.453999996185303, ymax=1,
ytick style={color=black},
ytick={0.5,0.6,0.7,0.8,0.9,1},
yticklabels={
  \(\displaystyle {0.5}\),
  \(\displaystyle {0.6}\),
  \(\displaystyle {0.7}\),
  \(\displaystyle {0.8}\),
  \(\displaystyle {0.9}\),
  \(\displaystyle {1.0}\)
}
]
\addplot [semithick, darkslateblue5882139, dashed]
table {%
1 0.458333343267441
2 0.562999963760376
3 0.614999949932098
4 0.672333359718323
5 0.685000002384186
6 0.721666634082794
7 0.731666684150696
8 0.75599992275238
9 0.737333357334137
10 0.732333302497864
11 0.770666658878326
12 0.759666681289673
};
\addplot [semithick, darkslateblue5882139]
table {%
1 0.776666700839996
2 0.682666659355164
3 0.733000040054321
4 0.717333376407623
5 0.826333224773407
6 0.852000057697296
7 0.849333345890045
8 0.864000022411346
9 0.878333330154419
10 0.873999953269958
11 0.875
12 0.883333325386047
};
\addplot [semithick, indigo68184, dashed]
table {%
1 0.478666663169861
2 0.527999997138977
3 0.606000006198883
4 0.65066659450531
5 0.703666627407074
6 0.704999923706055
7 0.723333358764648
8 0.734333276748657
9 0.71699994802475
10 0.757333338260651
11 0.745999932289124
12 0.745999991893768
};
\addplot [semithick, indigo68184]
table {%
1 0.747999966144562
2 0.687666714191437
3 0.717333376407623
4 0.716333329677582
5 0.781999886035919
6 0.822666704654694
7 0.81333327293396
8 0.817999958992004
9 0.840666651725769
10 0.86599999666214
11 0.83266669511795
12 0.855666697025299
};
\addplot [semithick, darkcyan32144140, dashed]
table {%
1 0.469666659832001
2 0.593666672706604
3 0.611333310604095
4 0.63566666841507
5 0.638999938964844
6 0.687666654586792
7 0.713666617870331
8 0.681999981403351
9 0.773333311080933
10 0.773000001907349
11 0.704999983310699
12 0.75900000333786
};
\addplot [semithick, darkcyan32144140]
table {%
1 0.791333377361298
2 0.735333383083344
3 0.781333327293396
4 0.852333307266235
5 0.873333334922791
6 0.897666692733765
7 0.892333328723907
8 0.901000022888184
9 0.909666657447815
10 0.913333296775818
11 0.910666644573212
12 0.923333287239075
};
\addplot [semithick, mediumseagreen9420197, dashed]
table {%
1 0.488999992609024
2 0.538333296775818
3 0.610333323478699
4 0.636333346366882
5 0.688333332538605
6 0.699333310127258
7 0.721666634082794
8 0.732333302497864
9 0.782333314418793
10 0.741333246231079
11 0.753666639328003
12 0.778333306312561
};
\addplot [semithick, mediumseagreen9420197]
table {%
1 0.728666663169861
2 0.83433336019516
3 0.867333292961121
4 0.884666621685028
5 0.881333291530609
6 0.895333290100098
7 0.907999932765961
8 0.912666618824005
9 0.916999995708466
10 0.914666593074799
11 0.908999979496002
12 0.908333361148834
};
\addplot [semithick, gold25323136, dashed]
table {%
1 0.503333330154419
2 0.561333358287811
3 0.623333334922791
4 0.644333243370056
5 0.682333290576935
6 0.688666641712189
7 0.714666604995728
8 0.708000004291534
9 0.740666687488556
10 0.745666563510895
11 0.762000024318695
12 0.751999974250793
};
\addplot [semithick, gold25323136]
table {%
1 0.761333346366882
2 0.852666676044464
3 0.861333250999451
4 0.8796666264534
5 0.897333264350891
6 0.899333298206329
7 0.898333311080933
8 0.895666658878326
9 0.90200001001358
10 0.912999987602234
11 0.905999958515167
12 0.913999915122986
};
\end{axis}

\end{tikzpicture}}
    \caption{MAD}
\end{subfigure}

    \begin{subfigure}{0.49\linewidth}
    \centering
    \resizebox{\linewidth}{!}{
\begin{tikzpicture}

\definecolor{darkcyan32144140}{RGB}{32,144,140}
\definecolor{darkgray176}{RGB}{176,176,176}
\definecolor{darkslateblue5882139}{RGB}{58,82,139}
\definecolor{gold25323136}{RGB}{253,231,36}
\definecolor{indigo68184}{RGB}{68,1,84}
\definecolor{mediumseagreen9420197}{RGB}{94,201,97}

\begin{axis}[
tick align=outside,
tick pos=left,
x grid style={darkgray176},
xmin=0.45, xmax=12.55,
xtick style={color=black},
xtick={2,4,6,8,10,12},
xticklabels={
  \(\displaystyle {2}\),
  \(\displaystyle {4}\),
  \(\displaystyle {6}\),
  \(\displaystyle {8}\),
  \(\displaystyle {10}\),
  \(\displaystyle {12}\)
},
y grid style={darkgray176},
ymin=0.453999996185303, ymax=1,
ytick style={color=black},
ytick={0.5,0.6,0.7,0.8,0.9,1},
yticklabels={
  \(\displaystyle {0.5}\),
  \(\displaystyle {0.6}\),
  \(\displaystyle {0.7}\),
  \(\displaystyle {0.8}\),
  \(\displaystyle {0.9}\),
  \(\displaystyle {1.0}\)
}
]
\addplot [semithick, darkslateblue5882139, dashed]
table {%
1 0.456666648387909
2 0.6203333735466
3 0.604333341121674
4 0.669666647911072
5 0.706666648387909
6 0.712999999523163
7 0.75
8 0.748333334922791
9 0.760999917984009
10 0.767666637897491
11 0.781333267688751
12 0.777999997138977
};
\addplot [semithick, darkslateblue5882139]
table {%
1 0.776000022888184
2 0.772666692733765
3 0.74566662311554
4 0.801999986171722
5 0.841000020503998
6 0.867333352565765
7 0.862999975681305
8 0.873999893665314
9 0.876666665077209
10 0.882666707038879
11 0.904333293437958
12 0.897999942302704
};
\addplot [semithick, indigo68184, dashed]
table {%
1 0.479666650295258
2 0.561999976634979
3 0.614666700363159
4 0.654999971389771
5 0.679666638374329
6 0.696999967098236
7 0.720999956130981
8 0.752666711807251
9 0.756666600704193
10 0.770333290100098
11 0.778333365917206
12 0.784333288669586
};
\addplot [semithick, indigo68184]
table {%
1 0.747999966144562
2 0.770666658878326
3 0.755999982357025
4 0.820999979972839
5 0.842333257198334
6 0.866333305835724
7 0.862666726112366
8 0.878666698932648
9 0.876333296298981
10 0.891333341598511
11 0.887333273887634
12 0.875999987125397
};
\addplot [semithick, darkcyan32144140, dashed]
table {%
1 0.472333312034607
2 0.60533332824707
3 0.606333374977112
4 0.680000007152557
5 0.675333321094513
6 0.721666574478149
7 0.745999991893768
8 0.763666689395905
9 0.764666616916656
10 0.765333294868469
11 0.777333319187164
12 0.776666641235352
};
\addplot [semithick, darkcyan32144140]
table {%
1 0.790000021457672
2 0.777333319187164
3 0.807333290576935
4 0.856333315372467
5 0.871666669845581
6 0.889333367347717
7 0.894999921321869
8 0.891333401203156
9 0.905666589736938
10 0.909999966621399
11 0.907999992370605
12 0.914333343505859
};
\addplot [semithick, mediumseagreen9420197, dashed]
table {%
1 0.489666670560837
2 0.580666661262512
3 0.614999949932098
4 0.647000014781952
5 0.684000015258789
6 0.708666622638702
7 0.733666598796844
8 0.733333349227905
9 0.762333273887634
10 0.768333315849304
11 0.772999882698059
12 0.786666631698608
};
\addplot [semithick, mediumseagreen9420197]
table {%
1 0.726999998092651
2 0.802333295345306
3 0.850000023841858
4 0.879666686058044
5 0.877333343029022
6 0.895666658878326
7 0.904999971389771
8 0.911666691303253
9 0.917333364486694
10 0.916000008583069
11 0.913666665554047
12 0.913666665554047
};
\addplot [semithick, gold25323136, dashed]
table {%
1 0.505666673183441
2 0.581666648387909
3 0.609666585922241
4 0.645666718482971
5 0.667999982833862
6 0.715333342552185
7 0.720666646957397
8 0.731000006198883
9 0.772333264350891
10 0.776666581630707
11 0.77099996805191
12 0.778999984264374
};
\addplot [semithick, gold25323136]
table {%
1 0.758666634559631
2 0.843999981880188
3 0.863666653633118
4 0.878666698932648
5 0.90366667509079
6 0.900333285331726
7 0.901000022888184
8 0.898666620254517
9 0.904666662216187
10 0.916999936103821
11 0.90966659784317
12 0.916000008583069
};
\end{axis}

\end{tikzpicture}}
    \caption{MBL}
\end{subfigure}
    \begin{subfigure}{0.49\linewidth}
    \centering
    \resizebox{\linewidth}{!}{
\begin{tikzpicture}

\definecolor{darkcyan32144140}{RGB}{32,144,140}
\definecolor{darkgray176}{RGB}{176,176,176}
\definecolor{darkslateblue5882139}{RGB}{58,82,139}
\definecolor{gold25323136}{RGB}{253,231,36}
\definecolor{indigo68184}{RGB}{68,1,84}
\definecolor{mediumseagreen9420197}{RGB}{94,201,97}

\begin{axis}[
tick align=outside,
tick pos=left,
x grid style={darkgray176},
xmin=0.45, xmax=12.55,
xtick style={color=black},
xtick={2,4,6,8,10,12},
xticklabels={
  \(\displaystyle {2}\),
  \(\displaystyle {4}\),
  \(\displaystyle {6}\),
  \(\displaystyle {8}\),
  \(\displaystyle {10}\),
  \(\displaystyle {12}\)
},
y grid style={darkgray176},
ymin=0.453999996185303, ymax=1,
ytick style={color=black},
ytick={0.5,0.6,0.7,0.8,0.9,1},
yticklabels={
  \(\displaystyle {0.5}\),
  \(\displaystyle {0.6}\),
  \(\displaystyle {0.7}\),
  \(\displaystyle {0.8}\),
  \(\displaystyle {0.9}\),
  \(\displaystyle {1.0}\)
}
]
\addplot [semithick, darkslateblue5882139, dashed]
table {%
1 0.453999996185303
2 0.649666607379913
3 0.71399998664856
4 0.784999966621399
5 0.812333285808563
6 0.829666674137115
7 0.847666680812836
8 0.857333302497864
9 0.852666676044464
10 0.874333322048187
11 0.890999972820282
12 0.889333367347717
};
\addplot [semithick, darkslateblue5882139]
table {%
1 0.78600001335144
2 0.827000021934509
3 0.853999972343445
4 0.868666589260101
5 0.879333317279816
6 0.899666726589203
7 0.900333344936371
8 0.893999993801117
9 0.905333340167999
10 0.900999963283539
11 0.91266667842865
12 0.90666663646698
};
\addplot [semithick, indigo68184, dashed]
table {%
1 0.468666672706604
2 0.610333323478699
3 0.718666732311249
4 0.752333343029022
5 0.806333303451538
6 0.826333224773407
7 0.835333406925201
8 0.861000001430511
9 0.852333247661591
10 0.868666589260101
11 0.883999943733215
12 0.871666610240936
};
\addplot [semithick, indigo68184]
table {%
1 0.759999990463257
2 0.828666687011719
3 0.835333228111267
4 0.871333301067352
5 0.881999969482422
6 0.89766663312912
7 0.894333302974701
8 0.888999998569489
9 0.894333302974701
10 0.917999982833862
11 0.897333323955536
12 0.899999916553497
};
\addplot [semithick, darkcyan32144140, dashed]
table {%
1 0.468333303928375
2 0.661999940872192
3 0.725666701793671
4 0.779333293437958
5 0.778333365917206
6 0.828666627407074
7 0.827666580677032
8 0.856333374977112
9 0.862666666507721
10 0.885666608810425
11 0.862666666507721
12 0.874999940395355
};
\addplot [semithick, darkcyan32144140]
table {%
1 0.791999995708466
2 0.826666653156281
3 0.862000048160553
4 0.881666600704193
5 0.885999977588654
6 0.901000022888184
7 0.90200001001358
8 0.898000001907349
9 0.904999971389771
10 0.912999987602234
11 0.910666644573212
12 0.916666626930237
};
\addplot [semithick, mediumseagreen9420197, dashed]
table {%
1 0.483333319425583
2 0.628000020980835
3 0.73033332824707
4 0.765333294868469
5 0.799000024795532
6 0.829333364963531
7 0.832666635513306
8 0.841666638851166
9 0.871000051498413
10 0.873000025749207
11 0.859000027179718
12 0.881999969482422
};
\addplot [semithick, mediumseagreen9420197]
table {%
1 0.73499995470047
2 0.840333342552185
3 0.8796666264534
4 0.886666655540466
5 0.88999992609024
6 0.895333290100098
7 0.904000043869019
8 0.901333272457123
9 0.916999936103821
10 0.911333322525024
11 0.898666620254517
12 0.906333267688751
};
\addplot [semithick, gold25323136, dashed]
table {%
1 0.499333322048187
2 0.643333315849304
3 0.720333278179169
4 0.755333304405212
5 0.805999994277954
6 0.827666640281677
7 0.847000002861023
8 0.83733332157135
9 0.869000017642975
10 0.864666640758514
11 0.875
12 0.867666602134705
};
\addplot [semithick, gold25323136]
table {%
1 0.760999977588654
2 0.849999964237213
3 0.869000017642975
4 0.877666652202606
5 0.898666620254517
6 0.89599996805191
7 0.893666625022888
8 0.891999959945679
9 0.897999942302704
10 0.916333258152008
11 0.904999971389771
12 0.902999937534332
};
\end{axis}

\end{tikzpicture}}
    \caption{MBL-QP}
\end{subfigure}
\begin{subfigure}{0.49\linewidth}
    \centering
    \resizebox{\linewidth}{!}{
\begin{tikzpicture}

\definecolor{darkcyan32144140}{RGB}{32,144,140}
\definecolor{darkgray176}{RGB}{176,176,176}
\definecolor{darkslateblue5882139}{RGB}{58,82,139}
\definecolor{gold25323136}{RGB}{253,231,36}
\definecolor{indigo68184}{RGB}{68,1,84}
\definecolor{mediumseagreen9420197}{RGB}{94,201,97}

\begin{axis}[
tick align=outside,
tick pos=left,
x grid style={darkgray176},
xmin=0.45, xmax=12.55,
xtick style={color=black},
xtick={2,4,6,8,10,12},
xticklabels={
  \(\displaystyle {2}\),
  \(\displaystyle {4}\),
  \(\displaystyle {6}\),
  \(\displaystyle {8}\),
  \(\displaystyle {10}\),
  \(\displaystyle {12}\)
},
y grid style={darkgray176},
ymin=0.453999996185303, ymax=1,
ytick style={color=black},
ytick={0.5,0.6,0.7,0.8,0.9,1},
yticklabels={
  \(\displaystyle {0.5}\),
  \(\displaystyle {0.6}\),
  \(\displaystyle {0.7}\),
  \(\displaystyle {0.8}\),
  \(\displaystyle {0.9}\),
  \(\displaystyle {1.0}\)
}
]
\addplot [semithick, darkslateblue5882139, dashed]
table {%
1 0.509333312511444
2 0.670999944210052
3 0.735666632652283
4 0.794666647911072
5 0.824666678905487
6 0.836333334445953
7 0.857666671276093
8 0.868000030517578
9 0.863666594028473
10 0.8796666264534
11 0.891666650772095
12 0.894333302974701
};
\addplot [semithick, darkslateblue5882139]
table {%
1 0.800999939441681
2 0.83899998664856
3 0.857999980449677
4 0.874000012874603
5 0.884666621685028
6 0.907333314418793
7 0.900999963283539
8 0.898666560649872
9 0.911666631698608
10 0.902999937534332
11 0.910999953746796
12 0.908333361148834
};
\addplot [semithick, indigo68184, dashed]
table {%
1 0.514666676521301
2 0.64000004529953
3 0.732666671276093
4 0.773999989032745
5 0.819666683673859
6 0.843999922275543
7 0.848666608333588
8 0.872333288192749
9 0.857333302497864
10 0.874666750431061
11 0.884999990463257
12 0.880333304405212
};
\addplot [semithick, indigo68184]
table {%
1 0.776666641235352
2 0.835666656494141
3 0.853333413600922
4 0.874000012874603
5 0.890333235263824
6 0.898333311080933
7 0.89466667175293
8 0.892333328723907
9 0.895666658878326
10 0.919666647911072
11 0.898999989032745
12 0.899999976158142
};
\addplot [semithick, darkcyan32144140, dashed]
table {%
1 0.521999955177307
2 0.688666641712189
3 0.738333344459534
4 0.787333309650421
5 0.78766655921936
6 0.837666630744934
7 0.836333334445953
8 0.858333349227905
9 0.87066662311554
10 0.890999972820282
11 0.869666635990143
12 0.883333325386047
};
\addplot [semithick, darkcyan32144140]
table {%
1 0.803333342075348
2 0.836333274841309
3 0.871999979019165
4 0.885333299636841
5 0.886666595935822
6 0.908666670322418
7 0.905333340167999
8 0.898666620254517
9 0.904333353042603
10 0.913666665554047
11 0.912999927997589
12 0.917999923229218
};
\addplot [semithick, mediumseagreen9420197, dashed]
table {%
1 0.528333306312561
2 0.660333395004272
3 0.747333347797394
4 0.776333391666412
5 0.810333251953125
6 0.84033328294754
7 0.847000002861023
8 0.849999964237213
9 0.88099992275238
10 0.880333304405212
11 0.872999966144562
12 0.887999951839447
};
\addplot [semithick, mediumseagreen9420197]
table {%
1 0.749666631221771
2 0.844333350658417
3 0.881333291530609
4 0.888333380222321
5 0.888999998569489
6 0.894666612148285
7 0.906000018119812
8 0.902666687965393
9 0.916666626930237
10 0.910333335399628
11 0.899333298206329
12 0.907999932765961
};
\addplot [semithick, gold25323136, dashed]
table {%
1 0.551999986171722
2 0.675000011920929
3 0.732333242893219
4 0.771999955177307
5 0.818333387374878
6 0.838333308696747
7 0.852999985218048
8 0.845999956130981
9 0.871999979019165
10 0.871333301067352
11 0.875
12 0.873666644096375
};
\addplot [semithick, gold25323136]
table {%
1 0.775666654109955
2 0.852333307266235
3 0.872999966144562
4 0.878333270549774
5 0.902333319187164
6 0.899999976158142
7 0.897000014781952
8 0.895333290100098
9 0.898666679859161
10 0.915333330631256
11 0.904999971389771
12 0.903999984264374
};
\end{axis}

\end{tikzpicture}}
    \caption{MC}
\end{subfigure}
\begin{subfigure}{0.49\linewidth}
    \centering
    \resizebox{\linewidth}{!}{
\begingroup%
\makeatletter%
\begin{pgfpicture}%
\pgfpathrectangle{\pgfpointorigin}{\pgfqpoint{1.725315in}{1.625315in}}%
\pgfusepath{use as bounding box, clip}%
\begin{pgfscope}%
\pgfsetbuttcap%
\pgfsetmiterjoin%
\definecolor{currentfill}{rgb}{1.000000,1.000000,1.000000}%
\pgfsetfillcolor{currentfill}%
\pgfsetlinewidth{0.000000pt}%
\definecolor{currentstroke}{rgb}{1.000000,1.000000,1.000000}%
\pgfsetstrokecolor{currentstroke}%
\pgfsetdash{}{0pt}%
\pgfpathmoveto{\pgfqpoint{0.000000in}{0.000000in}}%
\pgfpathlineto{\pgfqpoint{1.725315in}{0.000000in}}%
\pgfpathlineto{\pgfqpoint{1.725315in}{1.625315in}}%
\pgfpathlineto{\pgfqpoint{0.000000in}{1.625315in}}%
\pgfpathclose%
\pgfusepath{fill}%
\end{pgfscope}%
\begin{pgfscope}%
\pgfsetrectcap%
\pgfsetroundjoin%
\pgfsetlinewidth{1.505625pt}%
\definecolor{currentstroke}{rgb}{0.267004,0.004874,0.329415}%
\pgfsetstrokecolor{currentstroke}%
\pgfsetdash{}{0pt}%
\pgfpathmoveto{\pgfqpoint{0.399246in}{1.205380in}}%
\pgfpathlineto{\pgfqpoint{0.677024in}{1.205380in}}%
\pgfusepath{stroke}%
\end{pgfscope}%
\begin{pgfscope}%
\definecolor{textcolor}{rgb}{0.000000,0.000000,0.000000}%
\pgfsetstrokecolor{textcolor}%
\pgfsetfillcolor{textcolor}%
\pgftext[x=0.788135in,y=1.156768in,left,base]{\color{textcolor}\fontsize{10.000000}{12.000000}\selectfont \(\displaystyle \rho=\) 0.00}%
\end{pgfscope}%
\begin{pgfscope}%
\pgfsetrectcap%
\pgfsetroundjoin%
\pgfsetlinewidth{1.505625pt}%
\definecolor{currentstroke}{rgb}{0.229739,0.322361,0.545706}%
\pgfsetstrokecolor{currentstroke}%
\pgfsetdash{}{0pt}%
\pgfpathmoveto{\pgfqpoint{0.399246in}{1.011707in}}%
\pgfpathlineto{\pgfqpoint{0.677024in}{1.011707in}}%
\pgfusepath{stroke}%
\end{pgfscope}%
\begin{pgfscope}%
\definecolor{textcolor}{rgb}{0.000000,0.000000,0.000000}%
\pgfsetstrokecolor{textcolor}%
\pgfsetfillcolor{textcolor}%
\pgftext[x=0.788135in,y=0.963096in,left,base]{\color{textcolor}\fontsize{10.000000}{12.000000}\selectfont \(\displaystyle \rho=\) 0.01}%
\end{pgfscope}%
\begin{pgfscope}%
\pgfsetrectcap%
\pgfsetroundjoin%
\pgfsetlinewidth{1.505625pt}%
\definecolor{currentstroke}{rgb}{0.127568,0.566949,0.550556}%
\pgfsetstrokecolor{currentstroke}%
\pgfsetdash{}{0pt}%
\pgfpathmoveto{\pgfqpoint{0.399246in}{0.818034in}}%
\pgfpathlineto{\pgfqpoint{0.677024in}{0.818034in}}%
\pgfusepath{stroke}%
\end{pgfscope}%
\begin{pgfscope}%
\definecolor{textcolor}{rgb}{0.000000,0.000000,0.000000}%
\pgfsetstrokecolor{textcolor}%
\pgfsetfillcolor{textcolor}%
\pgftext[x=0.788135in,y=0.769423in,left,base]{\color{textcolor}\fontsize{10.000000}{12.000000}\selectfont \(\displaystyle \rho=\) 0.10}%
\end{pgfscope}%
\begin{pgfscope}%
\pgfsetrectcap%
\pgfsetroundjoin%
\pgfsetlinewidth{1.505625pt}%
\definecolor{currentstroke}{rgb}{0.369214,0.788888,0.382914}%
\pgfsetstrokecolor{currentstroke}%
\pgfsetdash{}{0pt}%
\pgfpathmoveto{\pgfqpoint{0.399246in}{0.624361in}}%
\pgfpathlineto{\pgfqpoint{0.677024in}{0.624361in}}%
\pgfusepath{stroke}%
\end{pgfscope}%
\begin{pgfscope}%
\definecolor{textcolor}{rgb}{0.000000,0.000000,0.000000}%
\pgfsetstrokecolor{textcolor}%
\pgfsetfillcolor{textcolor}%
\pgftext[x=0.788135in,y=0.575750in,left,base]{\color{textcolor}\fontsize{10.000000}{12.000000}\selectfont \(\displaystyle \rho=\) 1.00}%
\end{pgfscope}%
\begin{pgfscope}%
\pgfsetrectcap%
\pgfsetroundjoin%
\pgfsetlinewidth{1.505625pt}%
\definecolor{currentstroke}{rgb}{0.993248,0.906157,0.143936}%
\pgfsetstrokecolor{currentstroke}%
\pgfsetdash{}{0pt}%
\pgfpathmoveto{\pgfqpoint{0.399246in}{0.430688in}}%
\pgfpathlineto{\pgfqpoint{0.677024in}{0.430688in}}%
\pgfusepath{stroke}%
\end{pgfscope}%
\begin{pgfscope}%
\definecolor{textcolor}{rgb}{0.000000,0.000000,0.000000}%
\pgfsetstrokecolor{textcolor}%
\pgfsetfillcolor{textcolor}%
\pgftext[x=0.788135in,y=0.382077in,left,base]{\color{textcolor}\fontsize{10.000000}{12.000000}\selectfont \(\displaystyle \rho=\) 10.00}%
\end{pgfscope}%
\end{pgfpicture}%
\makeatother%
\endgroup%
}
\end{subfigure}
\caption{Accuracy vs. number of reference documents. Solid lines and dashed lines correspond to the reference documents being learned dictionary atoms or random documents from each class, respectively.  Here, we group by method and show the impact of different levels of regularization and learning versus random atoms. Except for small $\rho$ in 1NN, using learned reference documents significantly outperforms the use randomly sampled documents across all number of references.}
\label{fig:nlp_fixedmethod}
\end{figure}

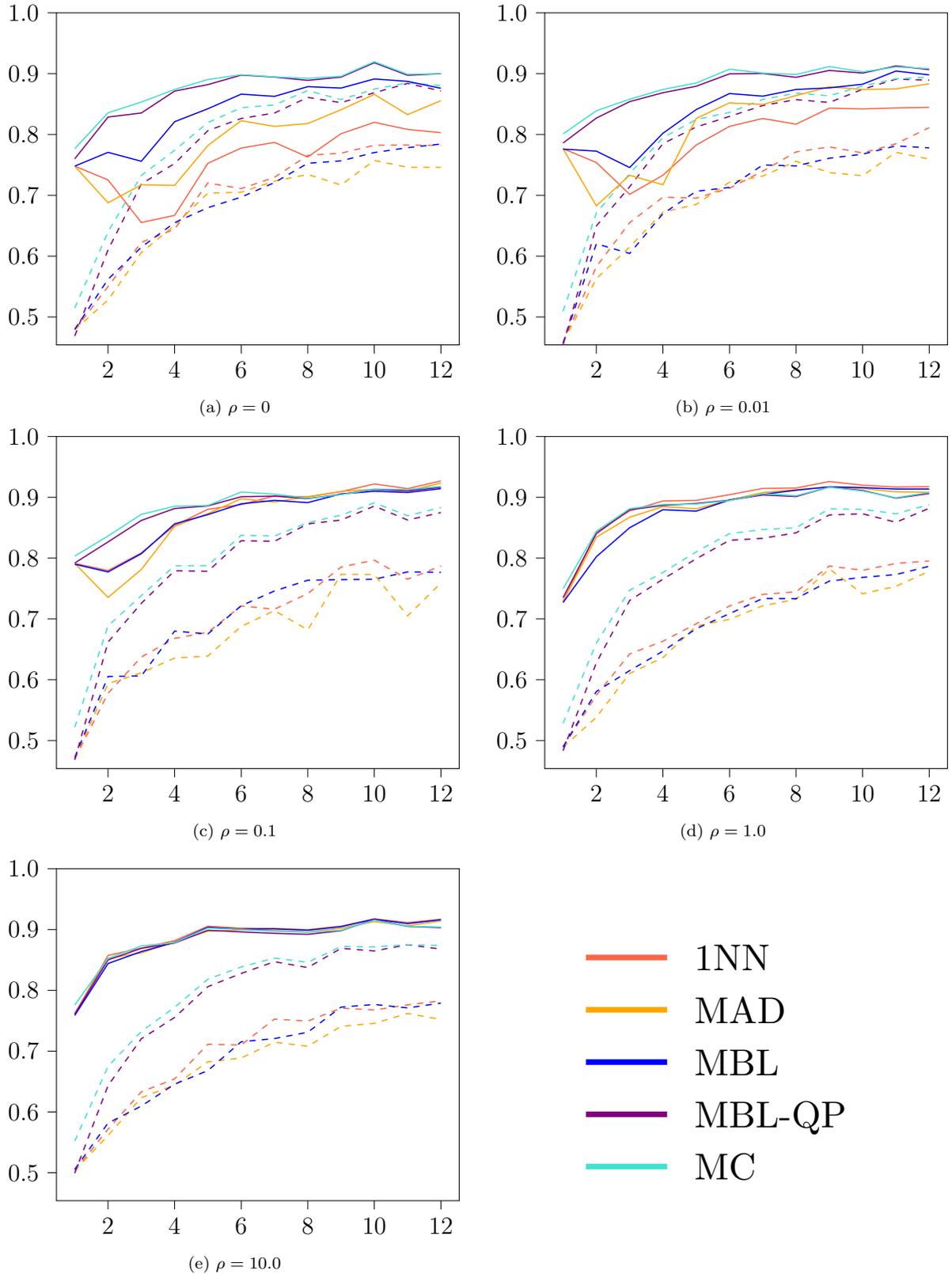
\begin{figure}[htbp!]
    \begin{subfigure}{0.49\linewidth}
    \centering
    \resizebox{\linewidth}{!}{
\begin{tikzpicture}

\definecolor{darkgray176}{RGB}{176,176,176}
\definecolor{orange}{RGB}{255,165,0}
\definecolor{purple}{RGB}{128,0,128}
\definecolor{tomato}{RGB}{255,99,71}
\definecolor{turquoise}{RGB}{64,224,208}

\begin{axis}[
tick align=outside,
tick pos=left,
x grid style={darkgray176},
xmin=0.45, xmax=12.55,
xtick style={color=black},
xtick={2,4,6,8,10,12},
xticklabels={
  \(\displaystyle {2}\),
  \(\displaystyle {4}\),
  \(\displaystyle {6}\),
  \(\displaystyle {8}\),
  \(\displaystyle {10}\),
  \(\displaystyle {12}\)
},
y grid style={darkgray176},
ymin=0.453999996185303, ymax=1,
ytick style={color=black},
ytick={0.5,0.6,0.7,0.8,0.9,1},
yticklabels={
  \(\displaystyle {0.5}\),
  \(\displaystyle {0.6}\),
  \(\displaystyle {0.7}\),
  \(\displaystyle {0.8}\),
  \(\displaystyle {0.9}\),
  \(\displaystyle {1.0}\)
}
]
\addplot [semithick, tomato, dashed]
table {%
1 0.478666663169861
2 0.550000011920929
3 0.622666597366333
4 0.644333302974701
5 0.720333218574524
6 0.711000025272369
7 0.729666650295258
8 0.765666663646698
9 0.769333302974701
10 0.782333314418793
11 0.782666623592377
12 0.781999945640564
};
\addplot [semithick, orange, dashed]
table {%
1 0.478666663169861
2 0.527999997138977
3 0.606000006198883
4 0.65066659450531
5 0.703666627407074
6 0.704999923706055
7 0.723333358764648
8 0.734333276748657
9 0.71699994802475
10 0.757333338260651
11 0.745999932289124
12 0.745999991893768
};
\addplot [semithick, blue, dashed]
table {%
1 0.479666650295258
2 0.561999976634979
3 0.614666700363159
4 0.654999971389771
5 0.679666638374329
6 0.696999967098236
7 0.720999956130981
8 0.752666711807251
9 0.756666600704193
10 0.770333290100098
11 0.778333365917206
12 0.784333288669586
};
\addplot [semithick, purple, dashed]
table {%
1 0.468666672706604
2 0.610333323478699
3 0.718666732311249
4 0.752333343029022
5 0.806333303451538
6 0.826333224773407
7 0.835333406925201
8 0.861000001430511
9 0.852333247661591
10 0.868666589260101
11 0.883999943733215
12 0.871666610240936
};
\addplot [semithick, turquoise, dashed]
table {%
1 0.514666676521301
2 0.64000004529953
3 0.732666671276093
4 0.773999989032745
5 0.819666683673859
6 0.843999922275543
7 0.848666608333588
8 0.872333288192749
9 0.857333302497864
10 0.874666750431061
11 0.884999990463257
12 0.880333304405212
};
\addplot [semithick, tomato]
table {%
1 0.747999966144562
2 0.725333273410797
3 0.655333340167999
4 0.666999936103821
5 0.752666652202606
6 0.777666628360748
7 0.786999940872192
8 0.762999951839447
9 0.801333367824554
10 0.819999933242798
11 0.808333337306976
12 0.803000032901764
};
\addplot [semithick, orange]
table {%
1 0.747999966144562
2 0.687666714191437
3 0.717333376407623
4 0.716333329677582
5 0.781999886035919
6 0.822666704654694
7 0.81333327293396
8 0.817999958992004
9 0.840666651725769
10 0.86599999666214
11 0.83266669511795
12 0.855666697025299
};
\addplot [semithick, blue]
table {%
1 0.747999966144562
2 0.770666658878326
3 0.755999982357025
4 0.820999979972839
5 0.842333257198334
6 0.866333305835724
7 0.862666726112366
8 0.878666698932648
9 0.876333296298981
10 0.891333341598511
11 0.887333273887634
12 0.875999987125397
};
\addplot [semithick, purple]
table {%
1 0.759999990463257
2 0.828666687011719
3 0.835333228111267
4 0.871333301067352
5 0.881999969482422
6 0.89766663312912
7 0.894333302974701
8 0.888999998569489
9 0.894333302974701
10 0.917999982833862
11 0.897333323955536
12 0.899999916553497
};
\addplot [semithick, turquoise]
table {%
1 0.776666641235352
2 0.835666656494141
3 0.853333413600922
4 0.874000012874603
5 0.890333235263824
6 0.898333311080933
7 0.89466667175293
8 0.892333328723907
9 0.895666658878326
10 0.919666647911072
11 0.898999989032745
12 0.899999976158142
};
\end{axis}

\end{tikzpicture}}
    \caption{$\rho=0$}
\end{subfigure}
\hfill
    \begin{subfigure}{0.49\linewidth}
    \centering
    \resizebox{\linewidth}{!}{
\begin{tikzpicture}

\definecolor{darkgray176}{RGB}{176,176,176}
\definecolor{orange}{RGB}{255,165,0}
\definecolor{purple}{RGB}{128,0,128}
\definecolor{tomato}{RGB}{255,99,71}
\definecolor{turquoise}{RGB}{64,224,208}

\begin{axis}[
tick align=outside,
tick pos=left,
x grid style={darkgray176},
xmin=0.45, xmax=12.55,
xtick style={color=black},
xtick={2,4,6,8,10,12},
xticklabels={
  \(\displaystyle {2}\),
  \(\displaystyle {4}\),
  \(\displaystyle {6}\),
  \(\displaystyle {8}\),
  \(\displaystyle {10}\),
  \(\displaystyle {12}\)
},
y grid style={darkgray176},
ymin=0.453999996185303, ymax=1,
ytick style={color=black},
ytick={0.5,0.6,0.7,0.8,0.9,1},
yticklabels={
  \(\displaystyle {0.5}\),
  \(\displaystyle {0.6}\),
  \(\displaystyle {0.7}\),
  \(\displaystyle {0.8}\),
  \(\displaystyle {0.9}\),
  \(\displaystyle {1.0}\)
}
]
\addplot [semithick, tomato, dashed]
table {%
1 0.458333343267441
2 0.582666635513306
3 0.654999971389771
4 0.696999967098236
5 0.695333302021027
6 0.711666643619537
7 0.739666640758514
8 0.771333336830139
9 0.779666662216187
10 0.76966667175293
11 0.784666657447815
12 0.811333358287811
};
\addplot [semithick, orange, dashed]
table {%
1 0.458333343267441
2 0.562999963760376
3 0.614999949932098
4 0.672333359718323
5 0.685000002384186
6 0.721666634082794
7 0.731666684150696
8 0.75599992275238
9 0.737333357334137
10 0.732333302497864
11 0.770666658878326
12 0.759666681289673
};
\addplot [semithick, blue, dashed]
table {%
1 0.456666648387909
2 0.6203333735466
3 0.604333341121674
4 0.669666647911072
5 0.706666648387909
6 0.712999999523163
7 0.75
8 0.748333334922791
9 0.760999917984009
10 0.767666637897491
11 0.781333267688751
12 0.777999997138977
};
\addplot [semithick, purple, dashed]
table {%
1 0.453999996185303
2 0.649666607379913
3 0.71399998664856
4 0.784999966621399
5 0.812333285808563
6 0.829666674137115
7 0.847666680812836
8 0.857333302497864
9 0.852666676044464
10 0.874333322048187
11 0.890999972820282
12 0.889333367347717
};
\addplot [semithick, turquoise, dashed]
table {%
1 0.509333312511444
2 0.670999944210052
3 0.735666632652283
4 0.794666647911072
5 0.824666678905487
6 0.836333334445953
7 0.857666671276093
8 0.868000030517578
9 0.863666594028473
10 0.8796666264534
11 0.891666650772095
12 0.894333302974701
};
\addplot [semithick, tomato]
table {%
1 0.776666700839996
2 0.754000008106232
3 0.701666653156281
4 0.732999980449677
5 0.782666623592377
6 0.813333332538605
7 0.826333224773407
8 0.816999971866608
9 0.84333336353302
10 0.842000007629395
11 0.843666672706604
12 0.844666659832001
};
\addplot [semithick, orange]
table {%
1 0.776666700839996
2 0.682666659355164
3 0.733000040054321
4 0.717333376407623
5 0.826333224773407
6 0.852000057697296
7 0.849333345890045
8 0.864000022411346
9 0.878333330154419
10 0.873999953269958
11 0.875
12 0.883333325386047
};
\addplot [semithick, blue]
table {%
1 0.776000022888184
2 0.772666692733765
3 0.74566662311554
4 0.801999986171722
5 0.841000020503998
6 0.867333352565765
7 0.862999975681305
8 0.873999893665314
9 0.876666665077209
10 0.882666707038879
11 0.904333293437958
12 0.897999942302704
};
\addplot [semithick, purple]
table {%
1 0.78600001335144
2 0.827000021934509
3 0.853999972343445
4 0.868666589260101
5 0.879333317279816
6 0.899666726589203
7 0.900333344936371
8 0.893999993801117
9 0.905333340167999
10 0.900999963283539
11 0.91266667842865
12 0.90666663646698
};
\addplot [semithick, turquoise]
table {%
1 0.800999939441681
2 0.83899998664856
3 0.857999980449677
4 0.874000012874603
5 0.884666621685028
6 0.907333314418793
7 0.900999963283539
8 0.898666560649872
9 0.911666631698608
10 0.902999937534332
11 0.910999953746796
12 0.908333361148834
};
\end{axis}

\end{tikzpicture}}
    \caption{$\rho=0.01$}
\end{subfigure}

    \begin{subfigure}{0.49\linewidth}
    \centering
    \resizebox{\linewidth}{!}{
\begin{tikzpicture}

\definecolor{darkgray176}{RGB}{176,176,176}
\definecolor{orange}{RGB}{255,165,0}
\definecolor{purple}{RGB}{128,0,128}
\definecolor{tomato}{RGB}{255,99,71}
\definecolor{turquoise}{RGB}{64,224,208}

\begin{axis}[
tick align=outside,
tick pos=left,
x grid style={darkgray176},
xmin=0.45, xmax=12.55,
xtick style={color=black},
xtick={2,4,6,8,10,12},
xticklabels={
  \(\displaystyle {2}\),
  \(\displaystyle {4}\),
  \(\displaystyle {6}\),
  \(\displaystyle {8}\),
  \(\displaystyle {10}\),
  \(\displaystyle {12}\)
},
y grid style={darkgray176},
ymin=0.453999996185303, ymax=1,
ytick style={color=black},
ytick={0.5,0.6,0.7,0.8,0.9,1},
yticklabels={
  \(\displaystyle {0.5}\),
  \(\displaystyle {0.6}\),
  \(\displaystyle {0.7}\),
  \(\displaystyle {0.8}\),
  \(\displaystyle {0.9}\),
  \(\displaystyle {1.0}\)
}
]
\addplot [semithick, tomato, dashed]
table {%
1 0.469666659832001
2 0.577666640281677
3 0.637333273887634
4 0.668000042438507
5 0.678000032901764
6 0.72133332490921
7 0.715999960899353
8 0.741666674613953
9 0.785333335399628
10 0.796666741371155
11 0.765333354473114
12 0.787000000476837
};
\addplot [semithick, orange, dashed]
table {%
1 0.469666659832001
2 0.593666672706604
3 0.611333310604095
4 0.63566666841507
5 0.638999938964844
6 0.687666654586792
7 0.713666617870331
8 0.681999981403351
9 0.773333311080933
10 0.773000001907349
11 0.704999983310699
12 0.75900000333786
};
\addplot [semithick, blue, dashed]
table {%
1 0.472333312034607
2 0.60533332824707
3 0.606333374977112
4 0.680000007152557
5 0.675333321094513
6 0.721666574478149
7 0.745999991893768
8 0.763666689395905
9 0.764666616916656
10 0.765333294868469
11 0.777333319187164
12 0.776666641235352
};
\addplot [semithick, purple, dashed]
table {%
1 0.468333303928375
2 0.661999940872192
3 0.725666701793671
4 0.779333293437958
5 0.778333365917206
6 0.828666627407074
7 0.827666580677032
8 0.856333374977112
9 0.862666666507721
10 0.885666608810425
11 0.862666666507721
12 0.874999940395355
};
\addplot [semithick, turquoise, dashed]
table {%
1 0.521999955177307
2 0.688666641712189
3 0.738333344459534
4 0.787333309650421
5 0.78766655921936
6 0.837666630744934
7 0.836333334445953
8 0.858333349227905
9 0.87066662311554
10 0.890999972820282
11 0.869666635990143
12 0.883333325386047
};
\addplot [semithick, tomato]
table {%
1 0.791333377361298
2 0.779666662216187
3 0.80866664648056
4 0.853666663169861
5 0.880333244800568
6 0.887999951839447
7 0.902333319187164
8 0.90066659450531
9 0.909333288669586
10 0.921999931335449
11 0.914333343505859
12 0.926999986171722
};
\addplot [semithick, orange]
table {%
1 0.791333377361298
2 0.735333383083344
3 0.781333327293396
4 0.852333307266235
5 0.873333334922791
6 0.897666692733765
7 0.892333328723907
8 0.901000022888184
9 0.909666657447815
10 0.913333296775818
11 0.910666644573212
12 0.923333287239075
};
\addplot [semithick, blue]
table {%
1 0.790000021457672
2 0.777333319187164
3 0.807333290576935
4 0.856333315372467
5 0.871666669845581
6 0.889333367347717
7 0.894999921321869
8 0.891333401203156
9 0.905666589736938
10 0.909999966621399
11 0.907999992370605
12 0.914333343505859
};
\addplot [semithick, purple]
table {%
1 0.791999995708466
2 0.826666653156281
3 0.862000048160553
4 0.881666600704193
5 0.885999977588654
6 0.901000022888184
7 0.90200001001358
8 0.898000001907349
9 0.904999971389771
10 0.912999987602234
11 0.910666644573212
12 0.916666626930237
};
\addplot [semithick, turquoise]
table {%
1 0.803333342075348
2 0.836333274841309
3 0.871999979019165
4 0.885333299636841
5 0.886666595935822
6 0.908666670322418
7 0.905333340167999
8 0.898666620254517
9 0.904333353042603
10 0.913666665554047
11 0.912999927997589
12 0.917999923229218
};
\end{axis}

\end{tikzpicture}}
    \caption{$\rho=0.1$}
\end{subfigure}
    \begin{subfigure}{0.49\linewidth}
    \centering
    \resizebox{\linewidth}{!}{
\begin{tikzpicture}

\definecolor{darkgray176}{RGB}{176,176,176}
\definecolor{orange}{RGB}{255,165,0}
\definecolor{purple}{RGB}{128,0,128}
\definecolor{tomato}{RGB}{255,99,71}
\definecolor{turquoise}{RGB}{64,224,208}

\begin{axis}[
tick align=outside,
tick pos=left,
x grid style={darkgray176},
xmin=0.45, xmax=12.55,
xtick style={color=black},
xtick={2,4,6,8,10,12},
xticklabels={
  \(\displaystyle {2}\),
  \(\displaystyle {4}\),
  \(\displaystyle {6}\),
  \(\displaystyle {8}\),
  \(\displaystyle {10}\),
  \(\displaystyle {12}\)
},
y grid style={darkgray176},
ymin=0.453999996185303, ymax=1,
ytick style={color=black},
ytick={0.5,0.6,0.7,0.8,0.9,1},
yticklabels={
  \(\displaystyle {0.5}\),
  \(\displaystyle {0.6}\),
  \(\displaystyle {0.7}\),
  \(\displaystyle {0.8}\),
  \(\displaystyle {0.9}\),
  \(\displaystyle {1.0}\)
}
]
\addplot [semithick, tomato, dashed]
table {%
1 0.488999992609024
2 0.574666678905487
3 0.642333269119263
4 0.663333296775818
5 0.691999971866608
6 0.721333265304565
7 0.740333318710327
8 0.744666695594788
9 0.787000000476837
10 0.779999971389771
11 0.791333317756653
12 0.79533326625824
};
\addplot [semithick, orange, dashed]
table {%
1 0.488999992609024
2 0.538333296775818
3 0.610333323478699
4 0.636333346366882
5 0.688333332538605
6 0.699333310127258
7 0.721666634082794
8 0.732333302497864
9 0.782333314418793
10 0.741333246231079
11 0.753666639328003
12 0.778333306312561
};
\addplot [semithick, blue, dashed]
table {%
1 0.489666670560837
2 0.580666661262512
3 0.614999949932098
4 0.647000014781952
5 0.684000015258789
6 0.708666622638702
7 0.733666598796844
8 0.733333349227905
9 0.762333273887634
10 0.768333315849304
11 0.772999882698059
12 0.786666631698608
};
\addplot [semithick, purple, dashed]
table {%
1 0.483333319425583
2 0.628000020980835
3 0.73033332824707
4 0.765333294868469
5 0.799000024795532
6 0.829333364963531
7 0.832666635513306
8 0.841666638851166
9 0.871000051498413
10 0.873000025749207
11 0.859000027179718
12 0.881999969482422
};
\addplot [semithick, turquoise, dashed]
table {%
1 0.528333306312561
2 0.660333395004272
3 0.747333347797394
4 0.776333391666412
5 0.810333251953125
6 0.84033328294754
7 0.847000002861023
8 0.849999964237213
9 0.88099992275238
10 0.880333304405212
11 0.872999966144562
12 0.887999951839447
};
\addplot [semithick, tomato]
table {%
1 0.728666663169861
2 0.843333303928375
3 0.877666652202606
4 0.893999934196472
5 0.894999921321869
6 0.904333353042603
7 0.914666593074799
8 0.915333330631256
9 0.925999939441681
10 0.920000016689301
11 0.916999995708466
12 0.917666614055634
};
\addplot [semithick, orange]
table {%
1 0.728666663169861
2 0.83433336019516
3 0.867333292961121
4 0.884666621685028
5 0.881333291530609
6 0.895333290100098
7 0.907999932765961
8 0.912666618824005
9 0.916999995708466
10 0.914666593074799
11 0.908999979496002
12 0.908333361148834
};
\addplot [semithick, blue]
table {%
1 0.726999998092651
2 0.802333295345306
3 0.850000023841858
4 0.879666686058044
5 0.877333343029022
6 0.895666658878326
7 0.904999971389771
8 0.911666691303253
9 0.917333364486694
10 0.916000008583069
11 0.913666665554047
12 0.913666665554047
};
\addplot [semithick, purple]
table {%
1 0.73499995470047
2 0.840333342552185
3 0.8796666264534
4 0.886666655540466
5 0.88999992609024
6 0.895333290100098
7 0.904000043869019
8 0.901333272457123
9 0.916999936103821
10 0.911333322525024
11 0.898666620254517
12 0.906333267688751
};
\addplot [semithick, turquoise]
table {%
1 0.749666631221771
2 0.844333350658417
3 0.881333291530609
4 0.888333380222321
5 0.888999998569489
6 0.894666612148285
7 0.906000018119812
8 0.902666687965393
9 0.916666626930237
10 0.910333335399628
11 0.899333298206329
12 0.907999932765961
};
\end{axis}

\end{tikzpicture}}
    \caption{$\rho=1.0$}
\end{subfigure}
\begin{subfigure}{0.49\linewidth}
    \centering
    \resizebox{\linewidth}{!}{
\begin{tikzpicture}

\definecolor{darkgray176}{RGB}{176,176,176}
\definecolor{orange}{RGB}{255,165,0}
\definecolor{purple}{RGB}{128,0,128}
\definecolor{tomato}{RGB}{255,99,71}
\definecolor{turquoise}{RGB}{64,224,208}

\begin{axis}[
tick align=outside,
tick pos=left,
x grid style={darkgray176},
xmin=0.45, xmax=12.55,
xtick style={color=black},
xtick={2,4,6,8,10,12},
xticklabels={
  \(\displaystyle {2}\),
  \(\displaystyle {4}\),
  \(\displaystyle {6}\),
  \(\displaystyle {8}\),
  \(\displaystyle {10}\),
  \(\displaystyle {12}\)
},
y grid style={darkgray176},
ymin=0.453999996185303, ymax=1,
ytick style={color=black},
ytick={0.5,0.6,0.7,0.8,0.9,1},
yticklabels={
  \(\displaystyle {0.5}\),
  \(\displaystyle {0.6}\),
  \(\displaystyle {0.7}\),
  \(\displaystyle {0.8}\),
  \(\displaystyle {0.9}\),
  \(\displaystyle {1.0}\)
}
]
\addplot [semithick, tomato, dashed]
table {%
1 0.503333330154419
2 0.571666657924652
3 0.633000016212463
4 0.654333293437958
5 0.711333274841309
6 0.709999978542328
7 0.752666592597961
8 0.749333322048187
9 0.770666658878326
10 0.767666637897491
11 0.775999963283539
12 0.782666623592377
};
\addplot [semithick, orange, dashed]
table {%
1 0.503333330154419
2 0.561333358287811
3 0.623333334922791
4 0.644333243370056
5 0.682333290576935
6 0.688666641712189
7 0.714666604995728
8 0.708000004291534
9 0.740666687488556
10 0.745666563510895
11 0.762000024318695
12 0.751999974250793
};
\addplot [semithick, blue, dashed]
table {%
1 0.505666673183441
2 0.581666648387909
3 0.609666585922241
4 0.645666718482971
5 0.667999982833862
6 0.715333342552185
7 0.720666646957397
8 0.731000006198883
9 0.772333264350891
10 0.776666581630707
11 0.77099996805191
12 0.778999984264374
};
\addplot [semithick, purple, dashed]
table {%
1 0.499333322048187
2 0.643333315849304
3 0.720333278179169
4 0.755333304405212
5 0.805999994277954
6 0.827666640281677
7 0.847000002861023
8 0.83733332157135
9 0.869000017642975
10 0.864666640758514
11 0.875
12 0.867666602134705
};
\addplot [semithick, turquoise, dashed]
table {%
1 0.551999986171722
2 0.675000011920929
3 0.732333242893219
4 0.771999955177307
5 0.818333387374878
6 0.838333308696747
7 0.852999985218048
8 0.845999956130981
9 0.871999979019165
10 0.871333301067352
11 0.875
12 0.873666644096375
};
\addplot [semithick, tomato]
table {%
1 0.761333346366882
2 0.857333302497864
3 0.869333267211914
4 0.881666600704193
5 0.905333280563354
6 0.90200001001358
7 0.901666641235352
8 0.899666607379913
9 0.905333340167999
10 0.91733330488205
11 0.91133326292038
12 0.916666626930237
};
\addplot [semithick, orange]
table {%
1 0.761333346366882
2 0.852666676044464
3 0.861333250999451
4 0.8796666264534
5 0.897333264350891
6 0.899333298206329
7 0.898333311080933
8 0.895666658878326
9 0.90200001001358
10 0.912999987602234
11 0.905999958515167
12 0.913999915122986
};
\addplot [semithick, blue]
table {%
1 0.758666634559631
2 0.843999981880188
3 0.863666653633118
4 0.878666698932648
5 0.90366667509079
6 0.900333285331726
7 0.901000022888184
8 0.898666620254517
9 0.904666662216187
10 0.916999936103821
11 0.90966659784317
12 0.916000008583069
};
\addplot [semithick, purple]
table {%
1 0.760999977588654
2 0.849999964237213
3 0.869000017642975
4 0.877666652202606
5 0.898666620254517
6 0.89599996805191
7 0.893666625022888
8 0.891999959945679
9 0.897999942302704
10 0.916333258152008
11 0.904999971389771
12 0.902999937534332
};
\addplot [semithick, turquoise]
table {%
1 0.775666654109955
2 0.852333307266235
3 0.872999966144562
4 0.878333270549774
5 0.902333319187164
6 0.899999976158142
7 0.897000014781952
8 0.895333290100098
9 0.898666679859161
10 0.915333330631256
11 0.904999971389771
12 0.903999984264374
};
\end{axis}

\end{tikzpicture}}
    \caption{$\rho=10.0$}
\end{subfigure}
\begin{subfigure}{0.49\linewidth}
    \centering
    \resizebox{\linewidth}{!}{
\begingroup%
\makeatletter%
\begin{pgfpicture}%
\pgfpathrectangle{\pgfpointorigin}{\pgfqpoint{1.725315in}{1.625315in}}%
\pgfusepath{use as bounding box, clip}%
\begin{pgfscope}%
\pgfsetbuttcap%
\pgfsetmiterjoin%
\definecolor{currentfill}{rgb}{1.000000,1.000000,1.000000}%
\pgfsetfillcolor{currentfill}%
\pgfsetlinewidth{0.000000pt}%
\definecolor{currentstroke}{rgb}{1.000000,1.000000,1.000000}%
\pgfsetstrokecolor{currentstroke}%
\pgfsetdash{}{0pt}%
\pgfpathmoveto{\pgfqpoint{0.000000in}{0.000000in}}%
\pgfpathlineto{\pgfqpoint{1.725315in}{0.000000in}}%
\pgfpathlineto{\pgfqpoint{1.725315in}{1.625315in}}%
\pgfpathlineto{\pgfqpoint{0.000000in}{1.625315in}}%
\pgfpathclose%
\pgfusepath{fill}%
\end{pgfscope}%
\begin{pgfscope}%
\pgfsetrectcap%
\pgfsetroundjoin%
\pgfsetlinewidth{1.505625pt}%
\definecolor{currentstroke}{rgb}{1.000000,0.388235,0.278431}%
\pgfsetstrokecolor{currentstroke}%
\pgfsetdash{}{0pt}%
\pgfpathmoveto{\pgfqpoint{0.409108in}{1.205380in}}%
\pgfpathlineto{\pgfqpoint{0.686885in}{1.205380in}}%
\pgfusepath{stroke}%
\end{pgfscope}%
\begin{pgfscope}%
\definecolor{textcolor}{rgb}{0.000000,0.000000,0.000000}%
\pgfsetstrokecolor{textcolor}%
\pgfsetfillcolor{textcolor}%
\pgftext[x=0.797996in,y=1.156768in,left,base]{\color{textcolor}\fontsize{10.000000}{12.000000}\selectfont 1NN}%
\end{pgfscope}%
\begin{pgfscope}%
\pgfsetrectcap%
\pgfsetroundjoin%
\pgfsetlinewidth{1.505625pt}%
\definecolor{currentstroke}{rgb}{1.000000,0.647059,0.000000}%
\pgfsetstrokecolor{currentstroke}%
\pgfsetdash{}{0pt}%
\pgfpathmoveto{\pgfqpoint{0.409108in}{1.011707in}}%
\pgfpathlineto{\pgfqpoint{0.686885in}{1.011707in}}%
\pgfusepath{stroke}%
\end{pgfscope}%
\begin{pgfscope}%
\definecolor{textcolor}{rgb}{0.000000,0.000000,0.000000}%
\pgfsetstrokecolor{textcolor}%
\pgfsetfillcolor{textcolor}%
\pgftext[x=0.797996in,y=0.963096in,left,base]{\color{textcolor}\fontsize{10.000000}{12.000000}\selectfont MAD}%
\end{pgfscope}%
\begin{pgfscope}%
\pgfsetrectcap%
\pgfsetroundjoin%
\pgfsetlinewidth{1.505625pt}%
\definecolor{currentstroke}{rgb}{0.000000,0.000000,1.000000}%
\pgfsetstrokecolor{currentstroke}%
\pgfsetdash{}{0pt}%
\pgfpathmoveto{\pgfqpoint{0.409108in}{0.818034in}}%
\pgfpathlineto{\pgfqpoint{0.686885in}{0.818034in}}%
\pgfusepath{stroke}%
\end{pgfscope}%
\begin{pgfscope}%
\definecolor{textcolor}{rgb}{0.000000,0.000000,0.000000}%
\pgfsetstrokecolor{textcolor}%
\pgfsetfillcolor{textcolor}%
\pgftext[x=0.797996in,y=0.769423in,left,base]{\color{textcolor}\fontsize{10.000000}{12.000000}\selectfont MBL}%
\end{pgfscope}%
\begin{pgfscope}%
\pgfsetrectcap%
\pgfsetroundjoin%
\pgfsetlinewidth{1.505625pt}%
\definecolor{currentstroke}{rgb}{0.501961,0.000000,0.501961}%
\pgfsetstrokecolor{currentstroke}%
\pgfsetdash{}{0pt}%
\pgfpathmoveto{\pgfqpoint{0.409108in}{0.624361in}}%
\pgfpathlineto{\pgfqpoint{0.686885in}{0.624361in}}%
\pgfusepath{stroke}%
\end{pgfscope}%
\begin{pgfscope}%
\definecolor{textcolor}{rgb}{0.000000,0.000000,0.000000}%
\pgfsetstrokecolor{textcolor}%
\pgfsetfillcolor{textcolor}%
\pgftext[x=0.797996in,y=0.575750in,left,base]{\color{textcolor}\fontsize{10.000000}{12.000000}\selectfont MBL-QP}%
\end{pgfscope}%
\begin{pgfscope}%
\pgfsetrectcap%
\pgfsetroundjoin%
\pgfsetlinewidth{1.505625pt}%
\definecolor{currentstroke}{rgb}{0.250980,0.878431,0.815686}%
\pgfsetstrokecolor{currentstroke}%
\pgfsetdash{}{0pt}%
\pgfpathmoveto{\pgfqpoint{0.409108in}{0.430688in}}%
\pgfpathlineto{\pgfqpoint{0.686885in}{0.430688in}}%
\pgfusepath{stroke}%
\end{pgfscope}%
\begin{pgfscope}%
\definecolor{textcolor}{rgb}{0.000000,0.000000,0.000000}%
\pgfsetstrokecolor{textcolor}%
\pgfsetfillcolor{textcolor}%
\pgftext[x=0.797996in,y=0.382077in,left,base]{\color{textcolor}\fontsize{10.000000}{12.000000}\selectfont MC}%
\end{pgfscope}%
\end{pgfpicture}%
\makeatother%
\endgroup%
}
\end{subfigure}
\caption{Accuracy vs. number of reference documents. Solid lines and dashed lines correspond to the reference documents being learned dictionary atoms or random documents from each class respectively.  Here, we group by regularization parameter and show the impact of different methods and learning versus random atoms. Increasing $\rho$ brings the non-barycentric based methods to performance parity with the barycentric based classification approaches.}
\label{fig:nlp_fixedlocality}
\end{figure}

\section{Conclusion}

We have extended the WDL framework by introducing a geometric sparse regularizer that interpolates between WDL and Wasserstein $K$-means according to a tuneable parameter $\rho$. We have shown the geometric sparse regularizer itself is useful in solving uniqueness and identifiability problems relating to the dictionary and weights, by leveraging characterizations of the nonlinear problem of exact barycentric reconstruction as well as geometric properties of Wasserstein geodesics. Additionally, we have shown the usefulness of our extension in providing improvements to classification methods on real data.  In particular our regularized framework improves over classical WDL in terms of atom interpretability and performance in classification.

\textbf{Future Work:}  Computational runtime remains a burden for both GeoSWDL and classical WDL.  While automatic differentiation provides for simple implementations, there may exist more specific algorithms to the dictionary learning framework; in particular the use of the geometric sparse regularizer may enable faster algorithms as is done in the linear case \citep{mallat1993matching}.  Relatedly, notions of linear optimal transport have important computational potential \citep{wang2013linear, moosmuller2020linear, hamm2022wassmap} in speeding up runtime of $W_{2}^{2}$ pairwise calculations for certain classes of measures. 

The analysis of the sparse coding step in Section \ref{sec:GWDL} does not immediately extend to case when the dictionary $\D$ is changing (which causes $\A$ to change).  Understanding how the matrix $\A$ changes with $\D$ is a topic of ongoing research, and may allow for a closed-form solution to (\ref{eqn:GWDL}).  Relatedly, Theorem \ref{thm:geodesic-extension} applies only to the case when the original data lie exactly on a Wasserstein barycenter between $m=2$ distributions; extending to $m\ge 3$ is a topic of ongoing research.

\textbf{\large Acknowledgements}: We thank Demba Ba, JinCheng Wang, and Matt Werenski for insightful discussions.  MM was partially supported was partially supported by NSF DMS 1924513 and CCF-1934553.  SA was partially supported by NSF CCF 1553075, NSF DRL 1931978, NSF EEC 1937057, and AFOSR FA9550-18-1-0465.  JMM was partially supported by NSF DMS 1912737, NSF DMS 1924513, and The Camille \& Henry Dreyfus Foundation.  AT was partially supported by DMS 2208392. The authors acknowledge the Tufts University High Performance Compute Cluster (\url{https://it.tufts.edu/high-performance-computing}) which was utilized for the research reported in this paper.

\bibliography{bib}

\newpage

\appendix

\section{Precise Statement and Proof of Proposition 1}

In order to establish this result, it is essential to note that if measures $\mu$ and $\nu$ satisfy the constraints that they have finite second moments and do not give mass to small sets (e.g., are absolutely continuous) \citep{villani2021topics}, their optimal plan $\pi^{*}$ concentrates on the graph of $T^{*}=\nabla\phi$ for a strictly convex $\phi$, so that \[W_{2}^{2}(\mu,\nu)=\displaystyle\int_{\mathbb{R}^{d}}\|T^{*}(\mathbf{x})-\mathbf{x}\|_{2}^{2}d\mu(\mathbf{x})\] where $T^{*}$ satisfies the pushforward constraint $T^{*}_{\#}\mu=\nu$ and is called the \emph{optimal transport map} \citep{smith1987note, brenier1991polar}.  

In order to precisely state Proposition 1, we require a few regularity assumptions \textbf{A1}-\textbf{A3} on the dictionary $\D=\{\D_{j}\}_{j=1}^{m}$ and measure $\mu$. These are required to invoke Theorem 1 in \cite{werenski2022measure}, which characterizes exactly when \[\mu=\bary(\D,\lam)\] for some $\lam\in\Delta^{m}$.

\textbf{A1:}\,\,The measures $\{\D_j\}_{j=1}^m$ and $\mu$ are absolutely continuous and supported on either all of $\mathbb{R}^d$ or a bounded open convex subset. Call this shared support set $\Omega$.

\textbf{A2:}\,\,The measures $\{\D_j\}_{j=1}^m$ 
 and $\mu$ have respective densities $\{g_j\}_{j=1}^m$ and $g$ which are bounded above and $g_1,...,g_m$ are strictly positive on $\Omega$.

\textbf{A3:}\,\,If $\Omega = \mathbb{R}^d$ then $\{g_j\}_{j=1}^m$ and $g$ are locally H\"older continuous. Otherwise $\{g_j\}_{j=1}^m$ and $g$ are bounded away from zero on $\Omega$.

\begin{proposition}Let $\mu$ be fixed and let $\{\D_j\}_{j=1}^m \subset  \Pb(\R^d)$ be a fixed dictionary.  Consider\begin{align}\label{eqn:SparseCodeAppdx}
    &\underset{\lam \in \Delta^m}{\arg\min}\,\, \sum_{j=1}^{m} \lambda_{j} W_2^2(\D_j, \mu) \,\,\emph{subject to}  \,\,\mu = \bary(\D, \lam).
\end{align}
If $\D$ and $\mu$ satisfy the assumptions \textbf{A1}-\textbf{A3}, the solution to (\ref{eqn:SparseCodeAppdx}) is given by \begin{equation}
    \underset{\lam \in \Delta^m}{\arg\min}\,\,  \lam^T\mathbf{c} \,\, \emph{subject to} \,\, \A\lam = \mathbf{0}, 
    \end{equation}
where $\mathbf{c}$ and $\A\in\mathbb{R}^{m\times m}$ are uniquely determined by $\mu,\{\D_j\}_{j=1}^m$.
    
\end{proposition}

\begin{proof}  Let $\{T_{j}\}_{j=1}^{m}$ be the optimal transport maps between $\mu$ and $\D_{j}$.  Define $\A\in\mathbb{R}^{m\times m}$ by \[A_{j\ell}=\int_{\mathbb{R}^{d}}\langle T_{j}(\mathbf{x})-\mathbf{x},T_{\ell}(\mathbf{x})-\mathbf{x}\rangle d\mu(\mathbf{x}).\]  Then by Theorem 1 in \cite{werenski2022measure}, which holds because \textbf{A1}-\textbf{A3} hold, \[\mu=\bary(\D,\lam) \Longleftrightarrow \lam^T\A\lam=0.\]  Since $\A$ is symmetric and positive semidefinite (it is in fact a Gram matrix), $\lam^T\A\lam = 0$ is equivalent to $\A\lam = \mathbf{0}$.  Letting $\mathbf{c}$ be defined as $c_{j}=W_{2}^{2}(\D_{j},\mu)$ gives the result.
\end{proof}

\section{Proof of Theorem 1}

\begin{proof}  Rearranging, we aim to show 
\[0\le (t-s)W_{2}^{2}(\mu,\mu_{t})+sW_{2}^{2}(\mu_{t},\tilde{\nu})-tW_{2}^{2}(\mu_{t},\nu).
\]Noting that McCann interpolants are in fact constant-speed geodesics in Wasserstein space \citep{ambrosio2005gradient}, we have that 
\begin{align*}t=\frac{W_{2}(\mu,\mu_{t})}{W_{2}(\mu,\nu)}, \quad s= \frac{W_{2}(\mu,\mu_{t})}{W_{2}(\mu,\tilde{\nu})}
\end{align*}
and 
\begin{align*}W_{2}(\mu,\tilde{\nu})&=W_{2}(\mu,\nu)+W_2(\nu,\tilde{\nu}),\\
W_{2}(\mu_{t},\tilde{\nu})&=W_{2}(\mu_{t},\nu)+W_2(\nu,\tilde{\nu}).\nonumber
\end{align*}In particular, $s<t$ and so it suffices to show 
\begin{align*}
tW_{2}^{2}(\mu_{t},\nu)&\le sW_{2}^{2}(\mu_{t},\tilde{\nu})\\
\Longleftrightarrow \frac{W_{2}(\mu,\mu_{t})}{W_{2}(\mu,\nu)}W_{2}^{2}(\mu_{t},\nu)&\le \frac{W_{2}(\mu,\mu_{t})}{W_{2}(\mu,\tilde{\nu})}W_{2}^{2}(\mu_{t},\tilde{\nu})\\
\Longleftrightarrow \frac{W_{2}^{2}(\mu_{t},\nu)}{W_{2}(\mu,\nu)}&\le 
\frac{W_{2}^{2}(\mu_{t},\tilde{\nu})}{W_{2}(\mu,\tilde{\nu})}\\
\Longleftrightarrow \frac{W_{2}^{2}(\mu_{t},\nu)}{W_{2}(\mu,\nu)}&\le \frac{(W_{2}(\mu_{t},\nu)+W_{2}(\nu,\tilde{\nu}))^{2}}{W_{2}(\mu,\nu)+W_{2}(\nu,\tilde{\nu})}\\
\Longleftrightarrow W_{2}^{2}(\mu_{t},\nu)(W_{2}(\mu,\nu)+W_{2}(\nu,\tilde{\nu}))&\le W_{2}(\mu,\nu)(W_{2}(\mu_{t},\nu)+W_{2}(\nu,\tilde{\nu}))^{2}\\
\Longleftrightarrow W_{2}^{2}(\mu_{t},\nu)(W_{2}(\mu,\nu)+W_{2}(\nu,\tilde{\nu}))&\le W_{2}(\mu,\nu)(W_{2}^{2}(\mu_{t},\nu)+2W_{2}(\mu_{t},\nu)W_{2}(\nu,\tilde{\nu})+W_{2}^{2}(\nu,\tilde{\nu}))\\
\Longleftrightarrow W_{2}^{2}(\mu_{t},\nu)W_{2}(\nu,\tilde{\nu})&\le W_{2}(\mu,\nu)(2W_{2}(\mu_{t},\nu)W_{2}(\nu,\tilde{\nu})+W_{2}^{2}(\nu,\tilde{\nu})).
\end{align*}
If $\tilde{\nu}=\nu$, the result follows trivially.  So, assume $W_{2}(\nu,\tilde{\nu})>0$.  Then the above reduces to 
\[W_{2}^{2}(\mu_{t},\nu)\le W_{2}(\mu,\nu)(2W_{2}(\mu_{t},\nu)+W_{2}(\nu,\tilde{\nu})).\]

The result follows by noting that $W_{2}(\mu_{t},\nu)\le W_{2}(\mu,\nu)$ and that $W_{2}(\nu,\tilde{\nu})\ge 0$.
\end{proof}

\section{Experimental Details}
For each of the experiments we report specific parameters used to generate the results. We also report the timings based on our (not necessarily optimal) code. 
\subsection{MNIST}
Specific parameter choices:
\begin{itemize}
    \item Atom initialization: Wasserstein K-Means++.
    \item Weight initialization: Uniform samples from the simplex.
    \item Optimizer: Adam with default parameters except for learning rate as 0.25.
    \item We use $L=250$ iterations for reasonable convergence.
    \item We use $L_s=50$ Sinkhorn iterations for both Wasserstein distance and barycenter computations. 
    \item Entropic transport computations were accelerated with Convolutional Wasserstein \citep{solomon2015convolutional}.
\end{itemize}
This experiment took 8 hours total using all cpu cores of an Apple M1 chip (no gpu). 

\subsection{NLP}
We plot the 1 standard deviation bars for the NLP experiments in Figure \ref{fig:nlp-std}. Specific parameter choices:
\begin{itemize}
    \item Atom initialization: Wasserstein K-Means++.
    \item Weight initialization: Each weight is initialized as a vector with uniform random samples and then normalized to lie on the simplex. This differs from uniform samples from the simplex, but in practice there were no performance differences. 
    \item Optimizer: Adam with default parameters except for learning rate as 0.25.
    \item We use $L=300$ iterations for reasonable convergence.
    \item We use $L_s=25$ Sinkhorn iterations for both Wasserstein distance and barycenter computations. 
\end{itemize}
Running one of the thirty trials of the experiment for all number of the references took at most 2 days on an HPC node using 2 cpu cores and one Nvidia gpu of type T4, RTX 6000, V100, or P100 (depending on node availability). 

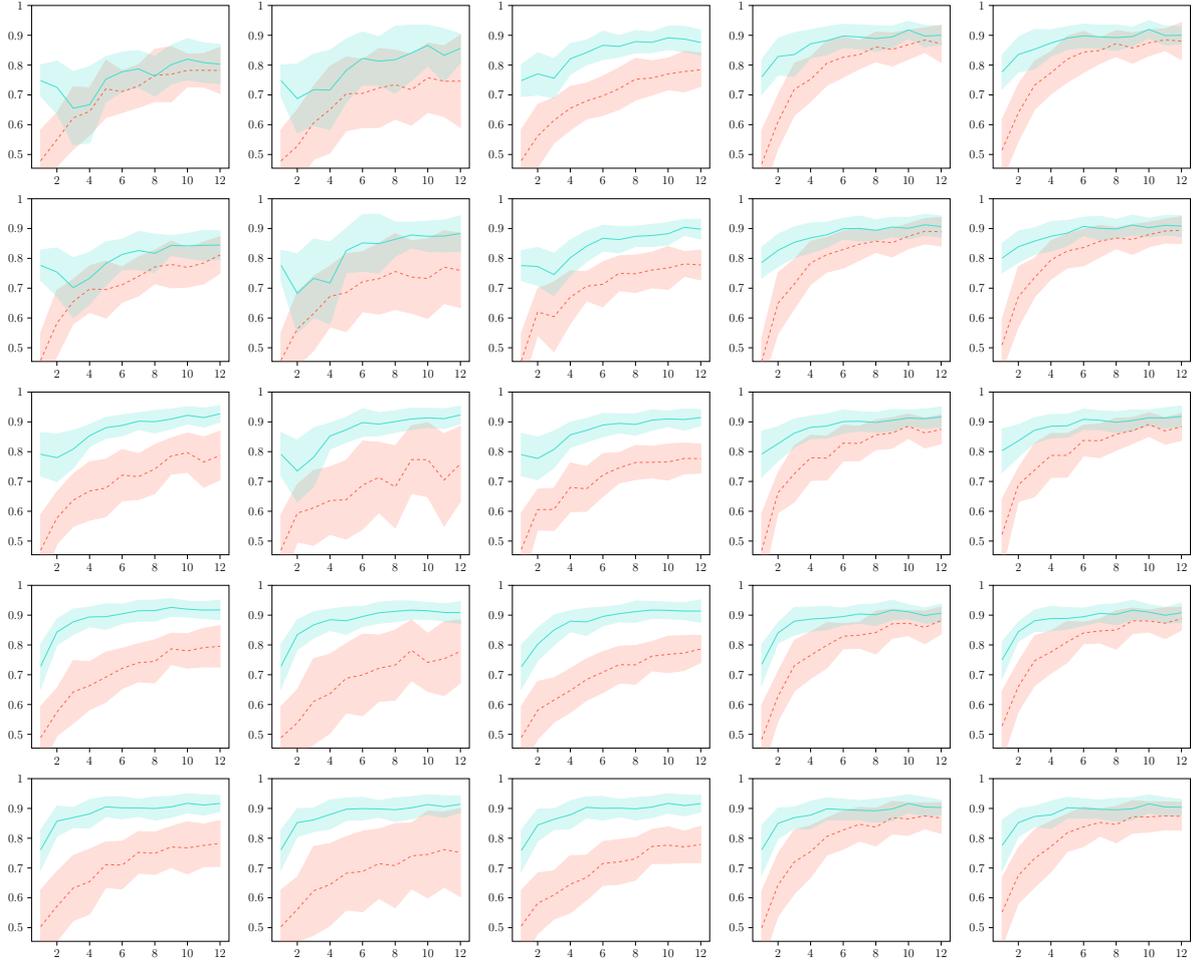
\begin{figure}
    \begin{subfigure}{0.19\linewidth}
    \centering
    \resizebox{\linewidth}{!}{
\begin{tikzpicture}

\definecolor{darkgray176}{RGB}{176,176,176}
\definecolor{tomato}{RGB}{255,99,71}
\definecolor{turquoise}{RGB}{64,224,208}

\begin{axis}[
tick align=outside,
tick pos=left,
x grid style={darkgray176},
xmin=0.45, xmax=12.55,
xtick style={color=black},
y grid style={darkgray176},
ymin=0.453999996185303, ymax=1,
ytick style={color=black}
]
\path [draw=tomato, fill=tomato, opacity=0.2]
(axis cs:1,0.58291220664978)
--(axis cs:1,0.374421149492264)
--(axis cs:2,0.459942579269409)
--(axis cs:3,0.517853677272797)
--(axis cs:4,0.564822614192963)
--(axis cs:5,0.62429141998291)
--(axis cs:6,0.640712440013885)
--(axis cs:7,0.657915949821472)
--(axis cs:8,0.677870512008667)
--(axis cs:9,0.676897585391998)
--(axis cs:10,0.727581024169922)
--(axis cs:11,0.726224184036255)
--(axis cs:12,0.704879343509674)
--(axis cs:12,0.859120547771454)
--(axis cs:12,0.859120547771454)
--(axis cs:11,0.839109063148499)
--(axis cs:10,0.837085604667664)
--(axis cs:9,0.861769020557404)
--(axis cs:8,0.853462815284729)
--(axis cs:7,0.801417350769043)
--(axis cs:6,0.781287610530853)
--(axis cs:5,0.816375017166138)
--(axis cs:4,0.723843991756439)
--(axis cs:3,0.727479517459869)
--(axis cs:2,0.640057444572449)
--(axis cs:1,0.58291220664978)
--cycle;

\path [draw=turquoise, fill=turquoise, opacity=0.2]
(axis cs:1,0.800351560115814)
--(axis cs:1,0.695648372173309)
--(axis cs:2,0.63706511259079)
--(axis cs:3,0.532669425010681)
--(axis cs:4,0.538745045661926)
--(axis cs:5,0.677102863788605)
--(axis cs:6,0.711091101169586)
--(axis cs:7,0.725179851055145)
--(axis cs:8,0.696245133876801)
--(axis cs:9,0.735746741294861)
--(axis cs:10,0.750917255878448)
--(axis cs:11,0.742152631282806)
--(axis cs:12,0.736867964267731)
--(axis cs:12,0.869132101535797)
--(axis cs:12,0.869132101535797)
--(axis cs:11,0.874514043331146)
--(axis cs:10,0.889082610607147)
--(axis cs:9,0.866919994354248)
--(axis cs:8,0.829754769802094)
--(axis cs:7,0.84882003068924)
--(axis cs:6,0.84424215555191)
--(axis cs:5,0.828230440616608)
--(axis cs:4,0.795254826545715)
--(axis cs:3,0.777997255325317)
--(axis cs:2,0.813601434230804)
--(axis cs:1,0.800351560115814)
--cycle;

\addplot [semithick, tomato, dashed]
table {%
1 0.478666663169861
2 0.550000011920929
3 0.622666597366333
4 0.644333302974701
5 0.720333218574524
6 0.711000025272369
7 0.729666650295258
8 0.765666663646698
9 0.769333302974701
10 0.782333314418793
11 0.782666623592377
12 0.781999945640564
};
\addplot [semithick, turquoise]
table {%
1 0.747999966144562
2 0.725333273410797
3 0.655333340167999
4 0.666999936103821
5 0.752666652202606
6 0.777666628360748
7 0.786999940872192
8 0.762999951839447
9 0.801333367824554
10 0.819999933242798
11 0.808333337306976
12 0.803000032901764
};
\end{axis}

\end{tikzpicture}}
\end{subfigure}
\begin{subfigure}{0.19\linewidth}
    \centering
    \resizebox{\linewidth}{!}{
\begin{tikzpicture}

\definecolor{darkgray176}{RGB}{176,176,176}
\definecolor{tomato}{RGB}{255,99,71}
\definecolor{turquoise}{RGB}{64,224,208}

\begin{axis}[
tick align=outside,
tick pos=left,
x grid style={darkgray176},
xmin=0.45, xmax=12.55,
xtick style={color=black},
y grid style={darkgray176},
ymin=0.453999996185303, ymax=1,
ytick style={color=black}
]
\path [draw=tomato, fill=tomato, opacity=0.2]
(axis cs:1,0.58291220664978)
--(axis cs:1,0.374421149492264)
--(axis cs:2,0.405683279037476)
--(axis cs:3,0.483143210411072)
--(axis cs:4,0.505350470542908)
--(axis cs:5,0.579890310764313)
--(axis cs:6,0.590386629104614)
--(axis cs:7,0.590711712837219)
--(axis cs:8,0.618554532527924)
--(axis cs:9,0.59826672077179)
--(axis cs:10,0.641279458999634)
--(axis cs:11,0.627573490142822)
--(axis cs:12,0.591223180294037)
--(axis cs:12,0.9007768034935)
--(axis cs:12,0.9007768034935)
--(axis cs:11,0.864426374435425)
--(axis cs:10,0.873387217521667)
--(axis cs:9,0.83573317527771)
--(axis cs:8,0.850112020969391)
--(axis cs:7,0.855955004692078)
--(axis cs:6,0.819613218307495)
--(axis cs:5,0.827442944049835)
--(axis cs:4,0.795982718467712)
--(axis cs:3,0.728856801986694)
--(axis cs:2,0.650316715240479)
--(axis cs:1,0.58291220664978)
--cycle;

\path [draw=turquoise, fill=turquoise, opacity=0.2]
(axis cs:1,0.800351560115814)
--(axis cs:1,0.695648372173309)
--(axis cs:2,0.572126507759094)
--(axis cs:3,0.597738265991211)
--(axis cs:4,0.582981765270233)
--(axis cs:5,0.676254570484161)
--(axis cs:6,0.733651638031006)
--(axis cs:7,0.731613278388977)
--(axis cs:8,0.705707252025604)
--(axis cs:9,0.747487783432007)
--(axis cs:10,0.797791361808777)
--(axis cs:11,0.743303656578064)
--(axis cs:12,0.807408511638641)
--(axis cs:12,0.903924882411957)
--(axis cs:12,0.903924882411957)
--(axis cs:11,0.922029733657837)
--(axis cs:10,0.934208631515503)
--(axis cs:9,0.933845520019531)
--(axis cs:8,0.930292665958405)
--(axis cs:7,0.895053267478943)
--(axis cs:6,0.911681771278381)
--(axis cs:5,0.887745201587677)
--(axis cs:4,0.84968489408493)
--(axis cs:3,0.836928486824036)
--(axis cs:2,0.803206920623779)
--(axis cs:1,0.800351560115814)
--cycle;

\addplot [semithick, tomato, dashed]
table {%
1 0.478666663169861
2 0.527999997138977
3 0.606000006198883
4 0.65066659450531
5 0.703666627407074
6 0.704999923706055
7 0.723333358764648
8 0.734333276748657
9 0.71699994802475
10 0.757333338260651
11 0.745999932289124
12 0.745999991893768
};
\addplot [semithick, turquoise]
table {%
1 0.747999966144562
2 0.687666714191437
3 0.717333376407623
4 0.716333329677582
5 0.781999886035919
6 0.822666704654694
7 0.81333327293396
8 0.817999958992004
9 0.840666651725769
10 0.86599999666214
11 0.83266669511795
12 0.855666697025299
};
\end{axis}

\end{tikzpicture}}
\end{subfigure}
\begin{subfigure}{0.19\linewidth}
    \centering
    \resizebox{\linewidth}{!}{
\begin{tikzpicture}

\definecolor{darkgray176}{RGB}{176,176,176}
\definecolor{tomato}{RGB}{255,99,71}
\definecolor{turquoise}{RGB}{64,224,208}

\begin{axis}[
tick align=outside,
tick pos=left,
x grid style={darkgray176},
xmin=0.45, xmax=12.55,
xtick style={color=black},
y grid style={darkgray176},
ymin=0.453999996185303, ymax=1,
ytick style={color=black}
]
\path [draw=tomato, fill=tomato, opacity=0.2]
(axis cs:1,0.583605766296387)
--(axis cs:1,0.375727564096451)
--(axis cs:2,0.45609176158905)
--(axis cs:3,0.53879326581955)
--(axis cs:4,0.581667959690094)
--(axis cs:5,0.619753837585449)
--(axis cs:6,0.631024539470673)
--(axis cs:7,0.650810539722443)
--(axis cs:8,0.682497799396515)
--(axis cs:9,0.69917094707489)
--(axis cs:10,0.717126905918121)
--(axis cs:11,0.710252106189728)
--(axis cs:12,0.728151142597198)
--(axis cs:12,0.840515434741974)
--(axis cs:12,0.840515434741974)
--(axis cs:11,0.846414625644684)
--(axis cs:10,0.823539674282074)
--(axis cs:9,0.814162254333496)
--(axis cs:8,0.822835624217987)
--(axis cs:7,0.79118937253952)
--(axis cs:6,0.7629753947258)
--(axis cs:5,0.739579439163208)
--(axis cs:4,0.728331983089447)
--(axis cs:3,0.690540134906769)
--(axis cs:2,0.667908191680908)
--(axis cs:1,0.583605766296387)
--cycle;

\path [draw=turquoise, fill=turquoise, opacity=0.2]
(axis cs:1,0.801459610462189)
--(axis cs:1,0.694540321826935)
--(axis cs:2,0.699473202228546)
--(axis cs:3,0.690157949924469)
--(axis cs:4,0.769448220729828)
--(axis cs:5,0.800836801528931)
--(axis cs:6,0.818737089633942)
--(axis cs:7,0.826474368572235)
--(axis cs:8,0.841459572315216)
--(axis cs:9,0.841544926166534)
--(axis cs:10,0.852228403091431)
--(axis cs:11,0.84450775384903)
--(axis cs:12,0.834261894226074)
--(axis cs:12,0.917738080024719)
--(axis cs:12,0.917738080024719)
--(axis cs:11,0.930158793926239)
--(axis cs:10,0.930438280105591)
--(axis cs:9,0.911121666431427)
--(axis cs:8,0.915873825550079)
--(axis cs:7,0.898859083652496)
--(axis cs:6,0.913929522037506)
--(axis cs:5,0.883829712867737)
--(axis cs:4,0.872551739215851)
--(axis cs:3,0.821842014789581)
--(axis cs:2,0.841860115528107)
--(axis cs:1,0.801459610462189)
--cycle;

\addplot [semithick, tomato, dashed]
table {%
1 0.479666650295258
2 0.561999976634979
3 0.614666700363159
4 0.654999971389771
5 0.679666638374329
6 0.696999967098236
7 0.720999956130981
8 0.752666711807251
9 0.756666600704193
10 0.770333290100098
11 0.778333365917206
12 0.784333288669586
};
\addplot [semithick, turquoise]
table {%
1 0.747999966144562
2 0.770666658878326
3 0.755999982357025
4 0.820999979972839
5 0.842333257198334
6 0.866333305835724
7 0.862666726112366
8 0.878666698932648
9 0.876333296298981
10 0.891333341598511
11 0.887333273887634
12 0.875999987125397
};
\end{axis}

\end{tikzpicture}}
\end{subfigure}
\begin{subfigure}{0.19\linewidth}
    \centering
    \resizebox{\linewidth}{!}{
\begin{tikzpicture}

\definecolor{darkgray176}{RGB}{176,176,176}
\definecolor{tomato}{RGB}{255,99,71}
\definecolor{turquoise}{RGB}{64,224,208}

\begin{axis}[
tick align=outside,
tick pos=left,
x grid style={darkgray176},
xmin=0.45, xmax=12.55,
xtick style={color=black},
y grid style={darkgray176},
ymin=0.453999996185303, ymax=1,
ytick style={color=black}
]
\path [draw=tomato, fill=tomato, opacity=0.2]
(axis cs:1,0.580415546894073)
--(axis cs:1,0.356917798519135)
--(axis cs:2,0.520641028881073)
--(axis cs:3,0.631459772586823)
--(axis cs:4,0.677330672740936)
--(axis cs:5,0.73184734582901)
--(axis cs:6,0.759046733379364)
--(axis cs:7,0.787936806678772)
--(axis cs:8,0.811010360717773)
--(axis cs:9,0.796892464160919)
--(axis cs:10,0.819484770298004)
--(axis cs:11,0.842760503292084)
--(axis cs:12,0.809294879436493)
--(axis cs:12,0.93403834104538)
--(axis cs:12,0.93403834104538)
--(axis cs:11,0.925239384174347)
--(axis cs:10,0.917848408222198)
--(axis cs:9,0.907774031162262)
--(axis cs:8,0.91098964214325)
--(axis cs:7,0.882730007171631)
--(axis cs:6,0.89361971616745)
--(axis cs:5,0.880819261074066)
--(axis cs:4,0.827336013317108)
--(axis cs:3,0.805873692035675)
--(axis cs:2,0.700025618076324)
--(axis cs:1,0.580415546894073)
--cycle;

\path [draw=turquoise, fill=turquoise, opacity=0.2]
(axis cs:1,0.818545579910278)
--(axis cs:1,0.701454401016235)
--(axis cs:2,0.766925692558289)
--(axis cs:3,0.761907994747162)
--(axis cs:4,0.824082434177399)
--(axis cs:5,0.834993004798889)
--(axis cs:6,0.858654379844666)
--(axis cs:7,0.853676319122314)
--(axis cs:8,0.84598034620285)
--(axis cs:9,0.85816478729248)
--(axis cs:10,0.889424622058868)
--(axis cs:11,0.861332058906555)
--(axis cs:12,0.869149923324585)
--(axis cs:12,0.93084990978241)
--(axis cs:12,0.93084990978241)
--(axis cs:11,0.933334589004517)
--(axis cs:10,0.946575343608856)
--(axis cs:9,0.930501818656921)
--(axis cs:8,0.932019650936127)
--(axis cs:7,0.934990286827087)
--(axis cs:6,0.936678886413574)
--(axis cs:5,0.929006934165955)
--(axis cs:4,0.918584167957306)
--(axis cs:3,0.908758461475372)
--(axis cs:2,0.890407681465149)
--(axis cs:1,0.818545579910278)
--cycle;

\addplot [semithick, tomato, dashed]
table {%
1 0.468666672706604
2 0.610333323478699
3 0.718666732311249
4 0.752333343029022
5 0.806333303451538
6 0.826333224773407
7 0.835333406925201
8 0.861000001430511
9 0.852333247661591
10 0.868666589260101
11 0.883999943733215
12 0.871666610240936
};
\addplot [semithick, turquoise]
table {%
1 0.759999990463257
2 0.828666687011719
3 0.835333228111267
4 0.871333301067352
5 0.881999969482422
6 0.89766663312912
7 0.894333302974701
8 0.888999998569489
9 0.894333302974701
10 0.917999982833862
11 0.897333323955536
12 0.899999916553497
};
\end{axis}

\end{tikzpicture}}
\end{subfigure}
\begin{subfigure}{0.19\linewidth}
    \centering
    \resizebox{\linewidth}{!}{
\begin{tikzpicture}

\definecolor{darkgray176}{RGB}{176,176,176}
\definecolor{tomato}{RGB}{255,99,71}
\definecolor{turquoise}{RGB}{64,224,208}

\begin{axis}[
tick align=outside,
tick pos=left,
x grid style={darkgray176},
xmin=0.45, xmax=12.55,
xtick style={color=black},
y grid style={darkgray176},
ymin=0.453999996185303, ymax=1,
ytick style={color=black}
]
\path [draw=tomato, fill=tomato, opacity=0.2]
(axis cs:1,0.616502702236176)
--(axis cs:1,0.412830650806427)
--(axis cs:2,0.544624209403992)
--(axis cs:3,0.651429116725922)
--(axis cs:4,0.702774286270142)
--(axis cs:5,0.742274582386017)
--(axis cs:6,0.775187313556671)
--(axis cs:7,0.805536210536957)
--(axis cs:8,0.828650772571564)
--(axis cs:9,0.806280016899109)
--(axis cs:10,0.831059277057648)
--(axis cs:11,0.848634958267212)
--(axis cs:12,0.81838321685791)
--(axis cs:12,0.942283391952515)
--(axis cs:12,0.942283391952515)
--(axis cs:11,0.921365022659302)
--(axis cs:10,0.918274223804474)
--(axis cs:9,0.908386588096619)
--(axis cs:8,0.916015803813934)
--(axis cs:7,0.891797006130219)
--(axis cs:6,0.912812530994415)
--(axis cs:5,0.8970587849617)
--(axis cs:4,0.845225691795349)
--(axis cs:3,0.813904225826263)
--(axis cs:2,0.735375881195068)
--(axis cs:1,0.616502702236176)
--cycle;

\path [draw=turquoise, fill=turquoise, opacity=0.2]
(axis cs:1,0.833194553852081)
--(axis cs:1,0.720138728618622)
--(axis cs:2,0.776787281036377)
--(axis cs:3,0.787087619304657)
--(axis cs:4,0.829464077949524)
--(axis cs:5,0.851962685585022)
--(axis cs:6,0.859021723270416)
--(axis cs:7,0.855210602283478)
--(axis cs:8,0.849851429462433)
--(axis cs:9,0.863644778728485)
--(axis cs:10,0.89036637544632)
--(axis cs:11,0.866906344890594)
--(axis cs:12,0.871472895145416)
--(axis cs:12,0.928527057170868)
--(axis cs:12,0.928527057170868)
--(axis cs:11,0.931093633174896)
--(axis cs:10,0.948966920375824)
--(axis cs:9,0.927688539028168)
--(axis cs:8,0.934815227985382)
--(axis cs:7,0.934122741222382)
--(axis cs:6,0.937644898891449)
--(axis cs:5,0.928703784942627)
--(axis cs:4,0.918535947799683)
--(axis cs:3,0.919579207897186)
--(axis cs:2,0.894546031951904)
--(axis cs:1,0.833194553852081)
--cycle;

\addplot [semithick, tomato, dashed]
table {%
1 0.514666676521301
2 0.64000004529953
3 0.732666671276093
4 0.773999989032745
5 0.819666683673859
6 0.843999922275543
7 0.848666608333588
8 0.872333288192749
9 0.857333302497864
10 0.874666750431061
11 0.884999990463257
12 0.880333304405212
};
\addplot [semithick, turquoise]
table {%
1 0.776666641235352
2 0.835666656494141
3 0.853333413600922
4 0.874000012874603
5 0.890333235263824
6 0.898333311080933
7 0.89466667175293
8 0.892333328723907
9 0.895666658878326
10 0.919666647911072
11 0.898999989032745
12 0.899999976158142
};
\end{axis}

\end{tikzpicture}}
\end{subfigure}

\begin{subfigure}{0.19\linewidth}
    \centering
    \resizebox{\linewidth}{!}{
\begin{tikzpicture}

\definecolor{darkgray176}{RGB}{176,176,176}
\definecolor{tomato}{RGB}{255,99,71}
\definecolor{turquoise}{RGB}{64,224,208}

\begin{axis}[
tick align=outside,
tick pos=left,
x grid style={darkgray176},
xmin=0.45, xmax=12.55,
xtick style={color=black},
y grid style={darkgray176},
ymin=0.453999996185303, ymax=1,
ytick style={color=black}
]
\path [draw=tomato, fill=tomato, opacity=0.2]
(axis cs:1,0.549573183059692)
--(axis cs:1,0.367093503475189)
--(axis cs:2,0.472121477127075)
--(axis cs:3,0.580733418464661)
--(axis cs:4,0.618957102298737)
--(axis cs:5,0.599894821643829)
--(axis cs:6,0.652413606643677)
--(axis cs:7,0.673988044261932)
--(axis cs:8,0.715086758136749)
--(axis cs:9,0.701212465763092)
--(axis cs:10,0.70372611284256)
--(axis cs:11,0.714458405971527)
--(axis cs:12,0.749872207641602)
--(axis cs:12,0.872794508934021)
--(axis cs:12,0.872794508934021)
--(axis cs:11,0.854874908924103)
--(axis cs:10,0.8356072306633)
--(axis cs:9,0.858120858669281)
--(axis cs:8,0.827579915523529)
--(axis cs:7,0.805345237255096)
--(axis cs:6,0.770919680595398)
--(axis cs:5,0.790771782398224)
--(axis cs:4,0.775042831897736)
--(axis cs:3,0.72926652431488)
--(axis cs:2,0.693211793899536)
--(axis cs:1,0.549573183059692)
--cycle;

\path [draw=turquoise, fill=turquoise, opacity=0.2]
(axis cs:1,0.827340304851532)
--(axis cs:1,0.725993096828461)
--(axis cs:2,0.672649323940277)
--(axis cs:3,0.601353943347931)
--(axis cs:4,0.645697474479675)
--(axis cs:5,0.709420084953308)
--(axis cs:6,0.762795984745026)
--(axis cs:7,0.762249052524567)
--(axis cs:8,0.752879917621613)
--(axis cs:9,0.810080587863922)
--(axis cs:10,0.800444066524506)
--(axis cs:11,0.793050944805145)
--(axis cs:12,0.796548008918762)
--(axis cs:12,0.892785310745239)
--(axis cs:12,0.892785310745239)
--(axis cs:11,0.894282400608063)
--(axis cs:10,0.883555948734283)
--(axis cs:9,0.876586139202118)
--(axis cs:8,0.881120026111603)
--(axis cs:7,0.890417397022247)
--(axis cs:6,0.863870680332184)
--(axis cs:5,0.855913162231445)
--(axis cs:4,0.820302486419678)
--(axis cs:3,0.80197936296463)
--(axis cs:2,0.835350692272186)
--(axis cs:1,0.827340304851532)
--cycle;

\addplot [semithick, tomato, dashed]
table {%
1 0.458333343267441
2 0.582666635513306
3 0.654999971389771
4 0.696999967098236
5 0.695333302021027
6 0.711666643619537
7 0.739666640758514
8 0.771333336830139
9 0.779666662216187
10 0.76966667175293
11 0.784666657447815
12 0.811333358287811
};
\addplot [semithick, turquoise]
table {%
1 0.776666700839996
2 0.754000008106232
3 0.701666653156281
4 0.732999980449677
5 0.782666623592377
6 0.813333332538605
7 0.826333224773407
8 0.816999971866608
9 0.84333336353302
10 0.842000007629395
11 0.843666672706604
12 0.844666659832001
};
\end{axis}

\end{tikzpicture}}
\end{subfigure}
\begin{subfigure}{0.19\linewidth}
    \centering
    \resizebox{\linewidth}{!}{
\begin{tikzpicture}

\definecolor{darkgray176}{RGB}{176,176,176}
\definecolor{tomato}{RGB}{255,99,71}
\definecolor{turquoise}{RGB}{64,224,208}

\begin{axis}[
tick align=outside,
tick pos=left,
x grid style={darkgray176},
xmin=0.45, xmax=12.55,
xtick style={color=black},
y grid style={darkgray176},
ymin=0.453999996185303, ymax=1,
ytick style={color=black}
]
\path [draw=tomato, fill=tomato, opacity=0.2]
(axis cs:1,0.549573183059692)
--(axis cs:1,0.367093503475189)
--(axis cs:2,0.439093708992004)
--(axis cs:3,0.490245014429092)
--(axis cs:4,0.57039749622345)
--(axis cs:5,0.553884088993073)
--(axis cs:6,0.62107926607132)
--(axis cs:7,0.612703263759613)
--(axis cs:8,0.628271758556366)
--(axis cs:9,0.61604905128479)
--(axis cs:10,0.59846043586731)
--(axis cs:11,0.648560643196106)
--(axis cs:12,0.634753346443176)
--(axis cs:12,0.884580016136169)
--(axis cs:12,0.884580016136169)
--(axis cs:11,0.892772674560547)
--(axis cs:10,0.866206169128418)
--(axis cs:9,0.858617663383484)
--(axis cs:8,0.883728086948395)
--(axis cs:7,0.850630104541779)
--(axis cs:6,0.822254002094269)
--(axis cs:5,0.816115915775299)
--(axis cs:4,0.774269223213196)
--(axis cs:3,0.739754855632782)
--(axis cs:2,0.686906218528748)
--(axis cs:1,0.549573183059692)
--cycle;

\path [draw=turquoise, fill=turquoise, opacity=0.2]
(axis cs:1,0.827340304851532)
--(axis cs:1,0.725993096828461)
--(axis cs:2,0.550237536430359)
--(axis cs:3,0.600669503211975)
--(axis cs:4,0.579795122146606)
--(axis cs:5,0.751569986343384)
--(axis cs:6,0.756682097911835)
--(axis cs:7,0.75023627281189)
--(axis cs:8,0.806488394737244)
--(axis cs:9,0.835254967212677)
--(axis cs:10,0.822834253311157)
--(axis cs:11,0.821646928787231)
--(axis cs:12,0.824007630348206)
--(axis cs:12,0.942659020423889)
--(axis cs:12,0.942659020423889)
--(axis cs:11,0.928353071212769)
--(axis cs:10,0.92516565322876)
--(axis cs:9,0.921411693096161)
--(axis cs:8,0.921511650085449)
--(axis cs:7,0.948430418968201)
--(axis cs:6,0.947318017482758)
--(axis cs:5,0.90109646320343)
--(axis cs:4,0.85487163066864)
--(axis cs:3,0.865330576896667)
--(axis cs:2,0.815095782279968)
--(axis cs:1,0.827340304851532)
--cycle;

\addplot [semithick, tomato, dashed]
table {%
1 0.458333343267441
2 0.562999963760376
3 0.614999949932098
4 0.672333359718323
5 0.685000002384186
6 0.721666634082794
7 0.731666684150696
8 0.75599992275238
9 0.737333357334137
10 0.732333302497864
11 0.770666658878326
12 0.759666681289673
};
\addplot [semithick, turquoise]
table {%
1 0.776666700839996
2 0.682666659355164
3 0.733000040054321
4 0.717333376407623
5 0.826333224773407
6 0.852000057697296
7 0.849333345890045
8 0.864000022411346
9 0.878333330154419
10 0.873999953269958
11 0.875
12 0.883333325386047
};
\end{axis}

\end{tikzpicture}}
\end{subfigure}
\begin{subfigure}{0.19\linewidth}
    \centering
    \resizebox{\linewidth}{!}{
\begin{tikzpicture}

\definecolor{darkgray176}{RGB}{176,176,176}
\definecolor{tomato}{RGB}{255,99,71}
\definecolor{turquoise}{RGB}{64,224,208}

\begin{axis}[
tick align=outside,
tick pos=left,
x grid style={darkgray176},
xmin=0.45, xmax=12.55,
xtick style={color=black},
y grid style={darkgray176},
ymin=0.453999996185303, ymax=1,
ytick style={color=black}
]
\path [draw=tomato, fill=tomato, opacity=0.2]
(axis cs:1,0.548968017101288)
--(axis cs:1,0.36436527967453)
--(axis cs:2,0.542585611343384)
--(axis cs:3,0.487398833036423)
--(axis cs:4,0.581330180168152)
--(axis cs:5,0.657583117485046)
--(axis cs:6,0.63760906457901)
--(axis cs:7,0.691572308540344)
--(axis cs:8,0.685466051101685)
--(axis cs:9,0.701325535774231)
--(axis cs:10,0.695379674434662)
--(axis cs:11,0.739666819572449)
--(axis cs:12,0.728006899356842)
--(axis cs:12,0.827993094921112)
--(axis cs:12,0.827993094921112)
--(axis cs:11,0.822999715805054)
--(axis cs:10,0.839953601360321)
--(axis cs:9,0.820674300193787)
--(axis cs:8,0.811200618743896)
--(axis cs:7,0.808427691459656)
--(axis cs:6,0.788390934467316)
--(axis cs:5,0.755750179290771)
--(axis cs:4,0.758003115653992)
--(axis cs:3,0.721267819404602)
--(axis cs:2,0.698081135749817)
--(axis cs:1,0.548968017101288)
--cycle;

\path [draw=turquoise, fill=turquoise, opacity=0.2]
(axis cs:1,0.826213359832764)
--(axis cs:1,0.725786685943604)
--(axis cs:2,0.708504557609558)
--(axis cs:3,0.674291789531708)
--(axis cs:4,0.742612302303314)
--(axis cs:5,0.789649307727814)
--(axis cs:6,0.823868453502655)
--(axis cs:7,0.818576455116272)
--(axis cs:8,0.840205371379852)
--(axis cs:9,0.839496672153473)
--(axis cs:10,0.845072388648987)
--(axis cs:11,0.877592742443085)
--(axis cs:12,0.864948272705078)
--(axis cs:12,0.93105161190033)
--(axis cs:12,0.93105161190033)
--(axis cs:11,0.931073844432831)
--(axis cs:10,0.920261025428772)
--(axis cs:9,0.913836658000946)
--(axis cs:8,0.907794415950775)
--(axis cs:7,0.907423496246338)
--(axis cs:6,0.910798251628876)
--(axis cs:5,0.892350733280182)
--(axis cs:4,0.861387670040131)
--(axis cs:3,0.817041456699371)
--(axis cs:2,0.836828827857971)
--(axis cs:1,0.826213359832764)
--cycle;

\addplot [semithick, tomato, dashed]
table {%
1 0.456666648387909
2 0.6203333735466
3 0.604333341121674
4 0.669666647911072
5 0.706666648387909
6 0.712999999523163
7 0.75
8 0.748333334922791
9 0.760999917984009
10 0.767666637897491
11 0.781333267688751
12 0.777999997138977
};
\addplot [semithick, turquoise]
table {%
1 0.776000022888184
2 0.772666692733765
3 0.74566662311554
4 0.801999986171722
5 0.841000020503998
6 0.867333352565765
7 0.862999975681305
8 0.873999893665314
9 0.876666665077209
10 0.882666707038879
11 0.904333293437958
12 0.897999942302704
};
\end{axis}

\end{tikzpicture}}
\end{subfigure}
\begin{subfigure}{0.19\linewidth}
    \centering
    \resizebox{\linewidth}{!}{
\begin{tikzpicture}

\definecolor{darkgray176}{RGB}{176,176,176}
\definecolor{tomato}{RGB}{255,99,71}
\definecolor{turquoise}{RGB}{64,224,208}

\begin{axis}[
tick align=outside,
tick pos=left,
x grid style={darkgray176},
xmin=0.45, xmax=12.55,
xtick style={color=black},
y grid style={darkgray176},
ymin=0.453999996185303, ymax=1,
ytick style={color=black}
]
\path [draw=tomato, fill=tomato, opacity=0.2]
(axis cs:1,0.532459318637848)
--(axis cs:1,0.37554070353508)
--(axis cs:2,0.547737538814545)
--(axis cs:3,0.631137430667877)
--(axis cs:4,0.712519288063049)
--(axis cs:5,0.753865301609039)
--(axis cs:6,0.767328143119812)
--(axis cs:7,0.793547809123993)
--(axis cs:8,0.820107698440552)
--(axis cs:9,0.804605424404144)
--(axis cs:10,0.832753896713257)
--(axis cs:11,0.848545134067535)
--(axis cs:12,0.841415822505951)
--(axis cs:12,0.937250912189484)
--(axis cs:12,0.937250912189484)
--(axis cs:11,0.933454811573029)
--(axis cs:10,0.915912747383118)
--(axis cs:9,0.900727927684784)
--(axis cs:8,0.894558906555176)
--(axis cs:7,0.901785552501678)
--(axis cs:6,0.892005205154419)
--(axis cs:5,0.870801270008087)
--(axis cs:4,0.857480645179749)
--(axis cs:3,0.796862542629242)
--(axis cs:2,0.751595675945282)
--(axis cs:1,0.532459318637848)
--cycle;

\path [draw=turquoise, fill=turquoise, opacity=0.2]
(axis cs:1,0.83703076839447)
--(axis cs:1,0.734969258308411)
--(axis cs:2,0.772915184497833)
--(axis cs:3,0.802098214626312)
--(axis cs:4,0.818375468254089)
--(axis cs:5,0.837238729000092)
--(axis cs:6,0.864977657794952)
--(axis cs:7,0.858439922332764)
--(axis cs:8,0.858513236045837)
--(axis cs:9,0.872456014156342)
--(axis cs:10,0.864767968654633)
--(axis cs:11,0.877734899520874)
--(axis cs:12,0.871011853218079)
--(axis cs:12,0.942321419715881)
--(axis cs:12,0.942321419715881)
--(axis cs:11,0.947598457336426)
--(axis cs:10,0.937231957912445)
--(axis cs:9,0.938210666179657)
--(axis cs:8,0.929486751556396)
--(axis cs:7,0.942226767539978)
--(axis cs:6,0.934355795383453)
--(axis cs:5,0.92142790555954)
--(axis cs:4,0.918957710266113)
--(axis cs:3,0.905901730060577)
--(axis cs:2,0.881084859371185)
--(axis cs:1,0.83703076839447)
--cycle;

\addplot [semithick, tomato, dashed]
table {%
1 0.453999996185303
2 0.649666607379913
3 0.71399998664856
4 0.784999966621399
5 0.812333285808563
6 0.829666674137115
7 0.847666680812836
8 0.857333302497864
9 0.852666676044464
10 0.874333322048187
11 0.890999972820282
12 0.889333367347717
};
\addplot [semithick, turquoise]
table {%
1 0.78600001335144
2 0.827000021934509
3 0.853999972343445
4 0.868666589260101
5 0.879333317279816
6 0.899666726589203
7 0.900333344936371
8 0.893999993801117
9 0.905333340167999
10 0.900999963283539
11 0.91266667842865
12 0.90666663646698
};
\end{axis}

\end{tikzpicture}}
\end{subfigure}
\begin{subfigure}{0.19\linewidth}
    \centering
    \resizebox{\linewidth}{!}{
\begin{tikzpicture}

\definecolor{darkgray176}{RGB}{176,176,176}
\definecolor{tomato}{RGB}{255,99,71}
\definecolor{turquoise}{RGB}{64,224,208}

\begin{axis}[
tick align=outside,
tick pos=left,
x grid style={darkgray176},
xmin=0.45, xmax=12.55,
xtick style={color=black},
y grid style={darkgray176},
ymin=0.453999996185303, ymax=1,
ytick style={color=black}
]
\path [draw=tomato, fill=tomato, opacity=0.2]
(axis cs:1,0.594507932662964)
--(axis cs:1,0.424158722162247)
--(axis cs:2,0.569856464862823)
--(axis cs:3,0.670822858810425)
--(axis cs:4,0.729979991912842)
--(axis cs:5,0.767749726772308)
--(axis cs:6,0.775054812431335)
--(axis cs:7,0.803995668888092)
--(axis cs:8,0.832435607910156)
--(axis cs:9,0.813461244106293)
--(axis cs:10,0.838353395462036)
--(axis cs:11,0.851573348045349)
--(axis cs:12,0.84894734621048)
--(axis cs:12,0.939719259738922)
--(axis cs:12,0.939719259738922)
--(axis cs:11,0.93175995349884)
--(axis cs:10,0.920979857444763)
--(axis cs:9,0.913871943950653)
--(axis cs:8,0.903564453125)
--(axis cs:7,0.911337673664093)
--(axis cs:6,0.897611856460571)
--(axis cs:5,0.881583631038666)
--(axis cs:4,0.859353303909302)
--(axis cs:3,0.800510406494141)
--(axis cs:2,0.772143423557281)
--(axis cs:1,0.594507932662964)
--cycle;

\path [draw=turquoise, fill=turquoise, opacity=0.2]
(axis cs:1,0.849661350250244)
--(axis cs:1,0.752338528633118)
--(axis cs:2,0.785867094993591)
--(axis cs:3,0.813634753227234)
--(axis cs:4,0.824339091777802)
--(axis cs:5,0.842099547386169)
--(axis cs:6,0.876379132270813)
--(axis cs:7,0.861315846443176)
--(axis cs:8,0.866424679756165)
--(axis cs:9,0.878027200698853)
--(axis cs:10,0.871579527854919)
--(axis cs:11,0.877849280834198)
--(axis cs:12,0.873485565185547)
--(axis cs:12,0.943181157112122)
--(axis cs:12,0.943181157112122)
--(axis cs:11,0.944150626659393)
--(axis cs:10,0.934420347213745)
--(axis cs:9,0.945306062698364)
--(axis cs:8,0.930908441543579)
--(axis cs:7,0.940684080123901)
--(axis cs:6,0.938287496566772)
--(axis cs:5,0.927233695983887)
--(axis cs:4,0.923660933971405)
--(axis cs:3,0.902365207672119)
--(axis cs:2,0.892132878303528)
--(axis cs:1,0.849661350250244)
--cycle;

\addplot [semithick, tomato, dashed]
table {%
1 0.509333312511444
2 0.670999944210052
3 0.735666632652283
4 0.794666647911072
5 0.824666678905487
6 0.836333334445953
7 0.857666671276093
8 0.868000030517578
9 0.863666594028473
10 0.8796666264534
11 0.891666650772095
12 0.894333302974701
};
\addplot [semithick, turquoise]
table {%
1 0.800999939441681
2 0.83899998664856
3 0.857999980449677
4 0.874000012874603
5 0.884666621685028
6 0.907333314418793
7 0.900999963283539
8 0.898666560649872
9 0.911666631698608
10 0.902999937534332
11 0.910999953746796
12 0.908333361148834
};
\end{axis}

\end{tikzpicture}}
\end{subfigure}

\begin{subfigure}{0.19\linewidth}
    \centering
    \resizebox{\linewidth}{!}{
\begin{tikzpicture}

\definecolor{darkgray176}{RGB}{176,176,176}
\definecolor{tomato}{RGB}{255,99,71}
\definecolor{turquoise}{RGB}{64,224,208}

\begin{axis}[
tick align=outside,
tick pos=left,
x grid style={darkgray176},
xmin=0.45, xmax=12.55,
xtick style={color=black},
y grid style={darkgray176},
ymin=0.453999996185303, ymax=1,
ytick style={color=black}
]
\path [draw=tomato, fill=tomato, opacity=0.2]
(axis cs:1,0.588031470775604)
--(axis cs:1,0.351301848888397)
--(axis cs:2,0.490501195192337)
--(axis cs:3,0.549370408058167)
--(axis cs:4,0.569936037063599)
--(axis cs:5,0.581674456596375)
--(axis cs:6,0.635360836982727)
--(axis cs:7,0.639318823814392)
--(axis cs:8,0.659452319145203)
--(axis cs:9,0.724833488464355)
--(axis cs:10,0.730629503726959)
--(axis cs:11,0.681103229522705)
--(axis cs:12,0.705076992511749)
--(axis cs:12,0.868923008441925)
--(axis cs:12,0.868923008441925)
--(axis cs:11,0.849563479423523)
--(axis cs:10,0.86270397901535)
--(axis cs:9,0.8458331823349)
--(axis cs:8,0.823881030082703)
--(axis cs:7,0.792681097984314)
--(axis cs:6,0.807305812835693)
--(axis cs:5,0.774325609207153)
--(axis cs:4,0.766064047813416)
--(axis cs:3,0.725296139717102)
--(axis cs:2,0.66483211517334)
--(axis cs:1,0.588031470775604)
--cycle;

\path [draw=turquoise, fill=turquoise, opacity=0.2]
(axis cs:1,0.864098131656647)
--(axis cs:1,0.718568623065948)
--(axis cs:2,0.699775218963623)
--(axis cs:3,0.745326638221741)
--(axis cs:4,0.816575586795807)
--(axis cs:5,0.845248758792877)
--(axis cs:6,0.852923691272736)
--(axis cs:7,0.874210000038147)
--(axis cs:8,0.862980663776398)
--(axis cs:9,0.875914990901947)
--(axis cs:10,0.893064796924591)
--(axis cs:11,0.882419288158417)
--(axis cs:12,0.89931333065033)
--(axis cs:12,0.954686641693115)
--(axis cs:12,0.954686641693115)
--(axis cs:11,0.946247398853302)
--(axis cs:10,0.950935065746307)
--(axis cs:9,0.942751586437225)
--(axis cs:8,0.938352525234222)
--(axis cs:7,0.930456638336182)
--(axis cs:6,0.923076212406158)
--(axis cs:5,0.915417730808258)
--(axis cs:4,0.890757739543915)
--(axis cs:3,0.87200665473938)
--(axis cs:2,0.85955810546875)
--(axis cs:1,0.864098131656647)
--cycle;

\addplot [semithick, tomato, dashed]
table {%
1 0.469666659832001
2 0.577666640281677
3 0.637333273887634
4 0.668000042438507
5 0.678000032901764
6 0.72133332490921
7 0.715999960899353
8 0.741666674613953
9 0.785333335399628
10 0.796666741371155
11 0.765333354473114
12 0.787000000476837
};
\addplot [semithick, turquoise]
table {%
1 0.791333377361298
2 0.779666662216187
3 0.80866664648056
4 0.853666663169861
5 0.880333244800568
6 0.887999951839447
7 0.902333319187164
8 0.90066659450531
9 0.909333288669586
10 0.921999931335449
11 0.914333343505859
12 0.926999986171722
};
\end{axis}

\end{tikzpicture}}
\end{subfigure}
\begin{subfigure}{0.19\linewidth}
    \centering
    \resizebox{\linewidth}{!}{
\begin{tikzpicture}

\definecolor{darkgray176}{RGB}{176,176,176}
\definecolor{tomato}{RGB}{255,99,71}
\definecolor{turquoise}{RGB}{64,224,208}

\begin{axis}[
tick align=outside,
tick pos=left,
x grid style={darkgray176},
xmin=0.45, xmax=12.55,
xtick style={color=black},
y grid style={darkgray176},
ymin=0.453999996185303, ymax=1,
ytick style={color=black}
]
\path [draw=tomato, fill=tomato, opacity=0.2]
(axis cs:1,0.588031470775604)
--(axis cs:1,0.351301848888397)
--(axis cs:2,0.497087836265564)
--(axis cs:3,0.485615134239197)
--(axis cs:4,0.522386908531189)
--(axis cs:5,0.506794512271881)
--(axis cs:6,0.540166020393372)
--(axis cs:7,0.595593512058258)
--(axis cs:8,0.544349789619446)
--(axis cs:9,0.660003364086151)
--(axis cs:10,0.648704767227173)
--(axis cs:11,0.550286650657654)
--(axis cs:12,0.632854223251343)
--(axis cs:12,0.885145783424377)
--(axis cs:12,0.885145783424377)
--(axis cs:11,0.859713315963745)
--(axis cs:10,0.897295236587524)
--(axis cs:9,0.886663258075714)
--(axis cs:8,0.819650173187256)
--(axis cs:7,0.831739723682404)
--(axis cs:6,0.835167288780212)
--(axis cs:5,0.771205365657806)
--(axis cs:4,0.74894642829895)
--(axis cs:3,0.737051486968994)
--(axis cs:2,0.690245509147644)
--(axis cs:1,0.588031470775604)
--cycle;

\path [draw=turquoise, fill=turquoise, opacity=0.2]
(axis cs:1,0.864098131656647)
--(axis cs:1,0.718568623065948)
--(axis cs:2,0.631751537322998)
--(axis cs:3,0.682661652565002)
--(axis cs:4,0.808571934700012)
--(axis cs:5,0.831287920475006)
--(axis cs:6,0.85108095407486)
--(axis cs:7,0.85367625951767)
--(axis cs:8,0.864863336086273)
--(axis cs:9,0.873612761497498)
--(axis cs:10,0.88124144077301)
--(axis cs:11,0.87755936384201)
--(axis cs:12,0.893816113471985)
--(axis cs:12,0.952850461006165)
--(axis cs:12,0.952850461006165)
--(axis cs:11,0.943773925304413)
--(axis cs:10,0.945425152778625)
--(axis cs:9,0.945720553398132)
--(axis cs:8,0.937136709690094)
--(axis cs:7,0.930990397930145)
--(axis cs:6,0.94425243139267)
--(axis cs:5,0.915378749370575)
--(axis cs:4,0.896094679832458)
--(axis cs:3,0.88000500202179)
--(axis cs:2,0.838915228843689)
--(axis cs:1,0.864098131656647)
--cycle;

\addplot [semithick, tomato, dashed]
table {%
1 0.469666659832001
2 0.593666672706604
3 0.611333310604095
4 0.63566666841507
5 0.638999938964844
6 0.687666654586792
7 0.713666617870331
8 0.681999981403351
9 0.773333311080933
10 0.773000001907349
11 0.704999983310699
12 0.75900000333786
};
\addplot [semithick, turquoise]
table {%
1 0.791333377361298
2 0.735333383083344
3 0.781333327293396
4 0.852333307266235
5 0.873333334922791
6 0.897666692733765
7 0.892333328723907
8 0.901000022888184
9 0.909666657447815
10 0.913333296775818
11 0.910666644573212
12 0.923333287239075
};
\end{axis}

\end{tikzpicture}}
\end{subfigure}
\begin{subfigure}{0.19\linewidth}
    \centering
    \resizebox{\linewidth}{!}{
\begin{tikzpicture}

\definecolor{darkgray176}{RGB}{176,176,176}
\definecolor{tomato}{RGB}{255,99,71}
\definecolor{turquoise}{RGB}{64,224,208}

\begin{axis}[
tick align=outside,
tick pos=left,
x grid style={darkgray176},
xmin=0.45, xmax=12.55,
xtick style={color=black},
y grid style={darkgray176},
ymin=0.453999996185303, ymax=1,
ytick style={color=black}
]
\path [draw=tomato, fill=tomato, opacity=0.2]
(axis cs:1,0.593979597091675)
--(axis cs:1,0.350687026977539)
--(axis cs:2,0.536916077136993)
--(axis cs:3,0.535307228565216)
--(axis cs:4,0.597705364227295)
--(axis cs:5,0.600512444972992)
--(axis cs:6,0.649647235870361)
--(axis cs:7,0.694901704788208)
--(axis cs:8,0.706222057342529)
--(axis cs:9,0.711768984794617)
--(axis cs:10,0.704605937004089)
--(axis cs:11,0.72587639093399)
--(axis cs:12,0.728006422519684)
--(axis cs:12,0.825326859951019)
--(axis cs:12,0.825326859951019)
--(axis cs:11,0.828790247440338)
--(axis cs:10,0.826060652732849)
--(axis cs:9,0.817564249038696)
--(axis cs:8,0.82111132144928)
--(axis cs:7,0.797098278999329)
--(axis cs:6,0.793685913085938)
--(axis cs:5,0.750154197216034)
--(axis cs:4,0.76229465007782)
--(axis cs:3,0.677359521389008)
--(axis cs:2,0.673750579357147)
--(axis cs:1,0.593979597091675)
--cycle;

\path [draw=turquoise, fill=turquoise, opacity=0.2]
(axis cs:1,0.861148238182068)
--(axis cs:1,0.718851804733276)
--(axis cs:2,0.705994725227356)
--(axis cs:3,0.746085703372955)
--(axis cs:4,0.813867628574371)
--(axis cs:5,0.832267463207245)
--(axis cs:6,0.850210785865784)
--(axis cs:7,0.863078653812408)
--(axis cs:8,0.855543494224548)
--(axis cs:9,0.871256947517395)
--(axis cs:10,0.883607685565948)
--(axis cs:11,0.872242212295532)
--(axis cs:12,0.887336134910583)
--(axis cs:12,0.941330552101135)
--(axis cs:12,0.941330552101135)
--(axis cs:11,0.943757772445679)
--(axis cs:10,0.936392247676849)
--(axis cs:9,0.940076231956482)
--(axis cs:8,0.927123308181763)
--(axis cs:7,0.926921188831329)
--(axis cs:6,0.928455948829651)
--(axis cs:5,0.911065876483917)
--(axis cs:4,0.898799002170563)
--(axis cs:3,0.868580877780914)
--(axis cs:2,0.848671913146973)
--(axis cs:1,0.861148238182068)
--cycle;

\addplot [semithick, tomato, dashed]
table {%
1 0.472333312034607
2 0.60533332824707
3 0.606333374977112
4 0.680000007152557
5 0.675333321094513
6 0.721666574478149
7 0.745999991893768
8 0.763666689395905
9 0.764666616916656
10 0.765333294868469
11 0.777333319187164
12 0.776666641235352
};
\addplot [semithick, turquoise]
table {%
1 0.790000021457672
2 0.777333319187164
3 0.807333290576935
4 0.856333315372467
5 0.871666669845581
6 0.889333367347717
7 0.894999921321869
8 0.891333401203156
9 0.905666589736938
10 0.909999966621399
11 0.907999992370605
12 0.914333343505859
};
\end{axis}

\end{tikzpicture}}
\end{subfigure}
\begin{subfigure}{0.19\linewidth}
    \centering
    \resizebox{\linewidth}{!}{
\begin{tikzpicture}

\definecolor{darkgray176}{RGB}{176,176,176}
\definecolor{tomato}{RGB}{255,99,71}
\definecolor{turquoise}{RGB}{64,224,208}

\begin{axis}[
tick align=outside,
tick pos=left,
x grid style={darkgray176},
xmin=0.45, xmax=12.55,
xtick style={color=black},
y grid style={darkgray176},
ymin=0.453999996185303, ymax=1,
ytick style={color=black}
]
\path [draw=tomato, fill=tomato, opacity=0.2]
(axis cs:1,0.592070579528809)
--(axis cs:1,0.34459599852562)
--(axis cs:2,0.59476900100708)
--(axis cs:3,0.631247103214264)
--(axis cs:4,0.703405320644379)
--(axis cs:5,0.704751014709473)
--(axis cs:6,0.768681943416595)
--(axis cs:7,0.7697913646698)
--(axis cs:8,0.806750059127808)
--(axis cs:9,0.81114274263382)
--(axis cs:10,0.845009624958038)
--(axis cs:11,0.812293350696564)
--(axis cs:12,0.826381087303162)
--(axis cs:12,0.923618793487549)
--(axis cs:12,0.923618793487549)
--(axis cs:11,0.913039982318878)
--(axis cs:10,0.926323592662811)
--(axis cs:9,0.914190590381622)
--(axis cs:8,0.905916690826416)
--(axis cs:7,0.885541796684265)
--(axis cs:6,0.888651311397552)
--(axis cs:5,0.851915717124939)
--(axis cs:4,0.855261266231537)
--(axis cs:3,0.820086300373077)
--(axis cs:2,0.729230880737305)
--(axis cs:1,0.592070579528809)
--cycle;

\path [draw=turquoise, fill=turquoise, opacity=0.2]
(axis cs:1,0.868806958198547)
--(axis cs:1,0.715193033218384)
--(axis cs:2,0.7692911028862)
--(axis cs:3,0.814265131950378)
--(axis cs:4,0.843512356281281)
--(axis cs:5,0.845605373382568)
--(axis cs:6,0.862105846405029)
--(axis cs:7,0.869686901569366)
--(axis cs:8,0.864636778831482)
--(axis cs:9,0.868351578712463)
--(axis cs:10,0.886588096618652)
--(axis cs:11,0.877873301506042)
--(axis cs:12,0.883413851261139)
--(axis cs:12,0.949919402599335)
--(axis cs:12,0.949919402599335)
--(axis cs:11,0.943459987640381)
--(axis cs:10,0.939411878585815)
--(axis cs:9,0.941648364067078)
--(axis cs:8,0.931363224983215)
--(axis cs:7,0.934313118457794)
--(axis cs:6,0.939894199371338)
--(axis cs:5,0.926394581794739)
--(axis cs:4,0.919820845127106)
--(axis cs:3,0.909734964370728)
--(axis cs:2,0.884042203426361)
--(axis cs:1,0.868806958198547)
--cycle;

\addplot [semithick, tomato, dashed]
table {%
1 0.468333303928375
2 0.661999940872192
3 0.725666701793671
4 0.779333293437958
5 0.778333365917206
6 0.828666627407074
7 0.827666580677032
8 0.856333374977112
9 0.862666666507721
10 0.885666608810425
11 0.862666666507721
12 0.874999940395355
};
\addplot [semithick, turquoise]
table {%
1 0.791999995708466
2 0.826666653156281
3 0.862000048160553
4 0.881666600704193
5 0.885999977588654
6 0.901000022888184
7 0.90200001001358
8 0.898000001907349
9 0.904999971389771
10 0.912999987602234
11 0.910666644573212
12 0.916666626930237
};
\end{axis}

\end{tikzpicture}}
\end{subfigure}
\begin{subfigure}{0.19\linewidth}
    \centering
    \resizebox{\linewidth}{!}{
\begin{tikzpicture}

\definecolor{darkgray176}{RGB}{176,176,176}
\definecolor{tomato}{RGB}{255,99,71}
\definecolor{turquoise}{RGB}{64,224,208}

\begin{axis}[
tick align=outside,
tick pos=left,
x grid style={darkgray176},
xmin=0.45, xmax=12.55,
xtick style={color=black},
y grid style={darkgray176},
ymin=0.453999996185303, ymax=1,
ytick style={color=black}
]
\path [draw=tomato, fill=tomato, opacity=0.2]
(axis cs:1,0.64172375202179)
--(axis cs:1,0.402276128530502)
--(axis cs:2,0.632481396198273)
--(axis cs:3,0.64881032705307)
--(axis cs:4,0.713337361812592)
--(axis cs:5,0.714478850364685)
--(axis cs:6,0.786427974700928)
--(axis cs:7,0.779492139816284)
--(axis cs:8,0.804700911045074)
--(axis cs:9,0.826885581016541)
--(axis cs:10,0.851402878761292)
--(axis cs:11,0.824831068515778)
--(axis cs:12,0.837973952293396)
--(axis cs:12,0.928692698478699)
--(axis cs:12,0.928692698478699)
--(axis cs:11,0.914502203464508)
--(axis cs:10,0.930597066879272)
--(axis cs:9,0.914447665214539)
--(axis cs:8,0.911965787410736)
--(axis cs:7,0.893174529075623)
--(axis cs:6,0.88890528678894)
--(axis cs:5,0.860854268074036)
--(axis cs:4,0.861329257488251)
--(axis cs:3,0.827856361865997)
--(axis cs:2,0.744851887226105)
--(axis cs:1,0.64172375202179)
--cycle;

\path [draw=turquoise, fill=turquoise, opacity=0.2]
(axis cs:1,0.876363039016724)
--(axis cs:1,0.730303645133972)
--(axis cs:2,0.782418668270111)
--(axis cs:3,0.82986718416214)
--(axis cs:4,0.845615923404694)
--(axis cs:5,0.848219752311707)
--(axis cs:6,0.874553918838501)
--(axis cs:7,0.873630702495575)
--(axis cs:8,0.863459527492523)
--(axis cs:9,0.86845201253891)
--(axis cs:10,0.887871265411377)
--(axis cs:11,0.878643751144409)
--(axis cs:12,0.882923662662506)
--(axis cs:12,0.953076183795929)
--(axis cs:12,0.953076183795929)
--(axis cs:11,0.947356104850769)
--(axis cs:10,0.939462065696716)
--(axis cs:9,0.940214693546295)
--(axis cs:8,0.93387371301651)
--(axis cs:7,0.937035977840424)
--(axis cs:6,0.942779421806335)
--(axis cs:5,0.925113439559937)
--(axis cs:4,0.925050675868988)
--(axis cs:3,0.91413277387619)
--(axis cs:2,0.890247881412506)
--(axis cs:1,0.876363039016724)
--cycle;

\addplot [semithick, tomato, dashed]
table {%
1 0.521999955177307
2 0.688666641712189
3 0.738333344459534
4 0.787333309650421
5 0.78766655921936
6 0.837666630744934
7 0.836333334445953
8 0.858333349227905
9 0.87066662311554
10 0.890999972820282
11 0.869666635990143
12 0.883333325386047
};
\addplot [semithick, turquoise]
table {%
1 0.803333342075348
2 0.836333274841309
3 0.871999979019165
4 0.885333299636841
5 0.886666595935822
6 0.908666670322418
7 0.905333340167999
8 0.898666620254517
9 0.904333353042603
10 0.913666665554047
11 0.912999927997589
12 0.917999923229218
};
\end{axis}

\end{tikzpicture}}
\end{subfigure}

\begin{subfigure}{0.19\linewidth}
    \centering
    \resizebox{\linewidth}{!}{
\begin{tikzpicture}

\definecolor{darkgray176}{RGB}{176,176,176}
\definecolor{tomato}{RGB}{255,99,71}
\definecolor{turquoise}{RGB}{64,224,208}

\begin{axis}[
tick align=outside,
tick pos=left,
x grid style={darkgray176},
xmin=0.45, xmax=12.55,
xtick style={color=black},
y grid style={darkgray176},
ymin=0.453999996185303, ymax=1,
ytick style={color=black}
]
\path [draw=tomato, fill=tomato, opacity=0.2]
(axis cs:1,0.59379369020462)
--(axis cs:1,0.384206265211105)
--(axis cs:2,0.49498051404953)
--(axis cs:3,0.537232875823975)
--(axis cs:4,0.582504272460938)
--(axis cs:5,0.608235120773315)
--(axis cs:6,0.653829336166382)
--(axis cs:7,0.676195383071899)
--(axis cs:8,0.672329664230347)
--(axis cs:9,0.734467566013336)
--(axis cs:10,0.722165465354919)
--(axis cs:11,0.726646602153778)
--(axis cs:12,0.725766479969025)
--(axis cs:12,0.864900052547455)
--(axis cs:12,0.864900052547455)
--(axis cs:11,0.856020033359528)
--(axis cs:10,0.837834477424622)
--(axis cs:9,0.839532434940338)
--(axis cs:8,0.817003726959229)
--(axis cs:7,0.804471254348755)
--(axis cs:6,0.788837194442749)
--(axis cs:5,0.7757648229599)
--(axis cs:4,0.744162321090698)
--(axis cs:3,0.747433662414551)
--(axis cs:2,0.654352843761444)
--(axis cs:1,0.59379369020462)
--cycle;

\path [draw=turquoise, fill=turquoise, opacity=0.2]
(axis cs:1,0.802138805389404)
--(axis cs:1,0.655194520950317)
--(axis cs:2,0.800075173377991)
--(axis cs:3,0.834780812263489)
--(axis cs:4,0.861241698265076)
--(axis cs:5,0.851788341999054)
--(axis cs:6,0.871990025043488)
--(axis cs:7,0.879754543304443)
--(axis cs:8,0.880126237869263)
--(axis cs:9,0.896451592445374)
--(axis cs:10,0.89334774017334)
--(axis cs:11,0.888942182064056)
--(axis cs:12,0.884689807891846)
--(axis cs:12,0.950643420219421)
--(axis cs:12,0.950643420219421)
--(axis cs:11,0.945057809352875)
--(axis cs:10,0.946652293205261)
--(axis cs:9,0.955548286437988)
--(axis cs:8,0.95054042339325)
--(axis cs:7,0.949578642845154)
--(axis cs:6,0.936676681041718)
--(axis cs:5,0.938211500644684)
--(axis cs:4,0.926758170127869)
--(axis cs:3,0.920552492141724)
--(axis cs:2,0.88659143447876)
--(axis cs:1,0.802138805389404)
--cycle;

\addplot [semithick, tomato, dashed]
table {%
1 0.488999992609024
2 0.574666678905487
3 0.642333269119263
4 0.663333296775818
5 0.691999971866608
6 0.721333265304565
7 0.740333318710327
8 0.744666695594788
9 0.787000000476837
10 0.779999971389771
11 0.791333317756653
12 0.79533326625824
};
\addplot [semithick, turquoise]
table {%
1 0.728666663169861
2 0.843333303928375
3 0.877666652202606
4 0.893999934196472
5 0.894999921321869
6 0.904333353042603
7 0.914666593074799
8 0.915333330631256
9 0.925999939441681
10 0.920000016689301
11 0.916999995708466
12 0.917666614055634
};
\end{axis}

\end{tikzpicture}}
\end{subfigure}
\begin{subfigure}{0.19\linewidth}
    \centering
    \resizebox{\linewidth}{!}{
\begin{tikzpicture}

\definecolor{darkgray176}{RGB}{176,176,176}
\definecolor{tomato}{RGB}{255,99,71}
\definecolor{turquoise}{RGB}{64,224,208}

\begin{axis}[
tick align=outside,
tick pos=left,
x grid style={darkgray176},
xmin=0.45, xmax=12.55,
xtick style={color=black},
y grid style={darkgray176},
ymin=0.453999996185303, ymax=1,
ytick style={color=black}
]
\path [draw=tomato, fill=tomato, opacity=0.2]
(axis cs:1,0.59379369020462)
--(axis cs:1,0.384206265211105)
--(axis cs:2,0.426758885383606)
--(axis cs:3,0.466195374727249)
--(axis cs:4,0.50342208147049)
--(axis cs:5,0.57227498292923)
--(axis cs:6,0.562372863292694)
--(axis cs:7,0.610277771949768)
--(axis cs:8,0.615016043186188)
--(axis cs:9,0.680532872676849)
--(axis cs:10,0.644034087657928)
--(axis cs:11,0.629667639732361)
--(axis cs:12,0.672435402870178)
--(axis cs:12,0.884231209754944)
--(axis cs:12,0.884231209754944)
--(axis cs:11,0.877665638923645)
--(axis cs:10,0.83863240480423)
--(axis cs:9,0.884133756160736)
--(axis cs:8,0.84965056180954)
--(axis cs:7,0.83305549621582)
--(axis cs:6,0.836293756961823)
--(axis cs:5,0.80439168214798)
--(axis cs:4,0.769244611263275)
--(axis cs:3,0.754471302032471)
--(axis cs:2,0.64990770816803)
--(axis cs:1,0.59379369020462)
--cycle;

\path [draw=turquoise, fill=turquoise, opacity=0.2]
(axis cs:1,0.802138805389404)
--(axis cs:1,0.655194520950317)
--(axis cs:2,0.784458816051483)
--(axis cs:3,0.82269424200058)
--(axis cs:4,0.848876655101776)
--(axis cs:5,0.835565268993378)
--(axis cs:6,0.861832559108734)
--(axis cs:7,0.872629940509796)
--(axis cs:8,0.877734839916229)
--(axis cs:9,0.88658332824707)
--(axis cs:10,0.886660039424896)
--(axis cs:11,0.879715383052826)
--(axis cs:12,0.871653616428375)
--(axis cs:12,0.945013105869293)
--(axis cs:12,0.945013105869293)
--(axis cs:11,0.938284575939178)
--(axis cs:10,0.942673146724701)
--(axis cs:9,0.947416663169861)
--(axis cs:8,0.947598397731781)
--(axis cs:7,0.943369925022125)
--(axis cs:6,0.928834021091461)
--(axis cs:5,0.927101314067841)
--(axis cs:4,0.92045658826828)
--(axis cs:3,0.911972343921661)
--(axis cs:2,0.884207904338837)
--(axis cs:1,0.802138805389404)
--cycle;

\addplot [semithick, tomato, dashed]
table {%
1 0.488999992609024
2 0.538333296775818
3 0.610333323478699
4 0.636333346366882
5 0.688333332538605
6 0.699333310127258
7 0.721666634082794
8 0.732333302497864
9 0.782333314418793
10 0.741333246231079
11 0.753666639328003
12 0.778333306312561
};
\addplot [semithick, turquoise]
table {%
1 0.728666663169861
2 0.83433336019516
3 0.867333292961121
4 0.884666621685028
5 0.881333291530609
6 0.895333290100098
7 0.907999932765961
8 0.912666618824005
9 0.916999995708466
10 0.914666593074799
11 0.908999979496002
12 0.908333361148834
};
\end{axis}

\end{tikzpicture}}
\end{subfigure}
\begin{subfigure}{0.19\linewidth}
    \centering
    \resizebox{\linewidth}{!}{
\begin{tikzpicture}

\definecolor{darkgray176}{RGB}{176,176,176}
\definecolor{tomato}{RGB}{255,99,71}
\definecolor{turquoise}{RGB}{64,224,208}

\begin{axis}[
tick align=outside,
tick pos=left,
x grid style={darkgray176},
xmin=0.45, xmax=12.55,
xtick style={color=black},
y grid style={darkgray176},
ymin=0.453999996185303, ymax=1,
ytick style={color=black}
]
\path [draw=tomato, fill=tomato, opacity=0.2]
(axis cs:1,0.594267189502716)
--(axis cs:1,0.38506618142128)
--(axis cs:2,0.483750998973846)
--(axis cs:3,0.540594279766083)
--(axis cs:4,0.573461413383484)
--(axis cs:5,0.6142418384552)
--(axis cs:6,0.639847338199615)
--(axis cs:7,0.671494245529175)
--(axis cs:8,0.668561398983002)
--(axis cs:9,0.713225126266479)
--(axis cs:10,0.707246422767639)
--(axis cs:11,0.714681506156921)
--(axis cs:12,0.741004228591919)
--(axis cs:12,0.832329034805298)
--(axis cs:12,0.832329034805298)
--(axis cs:11,0.831318259239197)
--(axis cs:10,0.829420208930969)
--(axis cs:9,0.811441421508789)
--(axis cs:8,0.798105299472809)
--(axis cs:7,0.795838952064514)
--(axis cs:6,0.777485907077789)
--(axis cs:5,0.753758192062378)
--(axis cs:4,0.72053861618042)
--(axis cs:3,0.689405620098114)
--(axis cs:2,0.677582323551178)
--(axis cs:1,0.594267189502716)
--cycle;

\path [draw=turquoise, fill=turquoise, opacity=0.2]
(axis cs:1,0.800303816795349)
--(axis cs:1,0.653696179389954)
--(axis cs:2,0.747833490371704)
--(axis cs:3,0.799553751945496)
--(axis cs:4,0.84428858757019)
--(axis cs:5,0.831928372383118)
--(axis cs:6,0.86429750919342)
--(axis cs:7,0.877613842487335)
--(axis cs:8,0.87574690580368)
--(axis cs:9,0.888935565948486)
--(axis cs:10,0.888133347034454)
--(axis cs:11,0.884960889816284)
--(axis cs:12,0.87629771232605)
--(axis cs:12,0.951035618782043)
--(axis cs:12,0.951035618782043)
--(axis cs:11,0.942372441291809)
--(axis cs:10,0.943866670131683)
--(axis cs:9,0.945731163024902)
--(axis cs:8,0.947586476802826)
--(axis cs:7,0.932386100292206)
--(axis cs:6,0.927035808563232)
--(axis cs:5,0.922738313674927)
--(axis cs:4,0.915044784545898)
--(axis cs:3,0.90044629573822)
--(axis cs:2,0.856833100318909)
--(axis cs:1,0.800303816795349)
--cycle;

\addplot [semithick, tomato, dashed]
table {%
1 0.489666670560837
2 0.580666661262512
3 0.614999949932098
4 0.647000014781952
5 0.684000015258789
6 0.708666622638702
7 0.733666598796844
8 0.733333349227905
9 0.762333273887634
10 0.768333315849304
11 0.772999882698059
12 0.786666631698608
};
\addplot [semithick, turquoise]
table {%
1 0.726999998092651
2 0.802333295345306
3 0.850000023841858
4 0.879666686058044
5 0.877333343029022
6 0.895666658878326
7 0.904999971389771
8 0.911666691303253
9 0.917333364486694
10 0.916000008583069
11 0.913666665554047
12 0.913666665554047
};
\end{axis}

\end{tikzpicture}}
\end{subfigure}
\begin{subfigure}{0.19\linewidth}
    \centering
    \resizebox{\linewidth}{!}{
\begin{tikzpicture}

\definecolor{darkgray176}{RGB}{176,176,176}
\definecolor{tomato}{RGB}{255,99,71}
\definecolor{turquoise}{RGB}{64,224,208}

\begin{axis}[
tick align=outside,
tick pos=left,
x grid style={darkgray176},
xmin=0.45, xmax=12.55,
xtick style={color=black},
y grid style={darkgray176},
ymin=0.453999996185303, ymax=1,
ytick style={color=black}
]
\path [draw=tomato, fill=tomato, opacity=0.2]
(axis cs:1,0.597813665866852)
--(axis cs:1,0.368852972984314)
--(axis cs:2,0.545479416847229)
--(axis cs:3,0.645379662513733)
--(axis cs:4,0.687264680862427)
--(axis cs:5,0.723189473152161)
--(axis cs:6,0.786993741989136)
--(axis cs:7,0.7889643907547)
--(axis cs:8,0.770175933837891)
--(axis cs:9,0.822126507759094)
--(axis cs:10,0.831134080886841)
--(axis cs:11,0.804080009460449)
--(axis cs:12,0.836635529994965)
--(axis cs:12,0.927364408969879)
--(axis cs:12,0.927364408969879)
--(axis cs:11,0.913920044898987)
--(axis cs:10,0.914865970611572)
--(axis cs:9,0.919873595237732)
--(axis cs:8,0.913157343864441)
--(axis cs:7,0.876368880271912)
--(axis cs:6,0.871672987937927)
--(axis cs:5,0.874810576438904)
--(axis cs:4,0.843401908874512)
--(axis cs:3,0.815286993980408)
--(axis cs:2,0.710520625114441)
--(axis cs:1,0.597813665866852)
--cycle;

\path [draw=turquoise, fill=turquoise, opacity=0.2]
(axis cs:1,0.8017857670784)
--(axis cs:1,0.66821414232254)
--(axis cs:2,0.806044161319733)
--(axis cs:3,0.833542108535767)
--(axis cs:4,0.846294820308685)
--(axis cs:5,0.851048111915588)
--(axis cs:6,0.866959750652313)
--(axis cs:7,0.872971713542938)
--(axis cs:8,0.86478066444397)
--(axis cs:9,0.883766174316406)
--(axis cs:10,0.887446641921997)
--(axis cs:11,0.872704684734344)
--(axis cs:12,0.875429272651672)
--(axis cs:12,0.93723726272583)
--(axis cs:12,0.93723726272583)
--(axis cs:11,0.924628555774689)
--(axis cs:10,0.935220003128052)
--(axis cs:9,0.950233697891235)
--(axis cs:8,0.937885880470276)
--(axis cs:7,0.935028374195099)
--(axis cs:6,0.923706829547882)
--(axis cs:5,0.928951740264893)
--(axis cs:4,0.927038490772247)
--(axis cs:3,0.925791144371033)
--(axis cs:2,0.874622523784637)
--(axis cs:1,0.8017857670784)
--cycle;

\addplot [semithick, tomato, dashed]
table {%
1 0.483333319425583
2 0.628000020980835
3 0.73033332824707
4 0.765333294868469
5 0.799000024795532
6 0.829333364963531
7 0.832666635513306
8 0.841666638851166
9 0.871000051498413
10 0.873000025749207
11 0.859000027179718
12 0.881999969482422
};
\addplot [semithick, turquoise]
table {%
1 0.73499995470047
2 0.840333342552185
3 0.8796666264534
4 0.886666655540466
5 0.88999992609024
6 0.895333290100098
7 0.904000043869019
8 0.901333272457123
9 0.916999936103821
10 0.911333322525024
11 0.898666620254517
12 0.906333267688751
};
\end{axis}

\end{tikzpicture}}
\end{subfigure}
\begin{subfigure}{0.19\linewidth}
    \centering
    \resizebox{\linewidth}{!}{
\begin{tikzpicture}

\definecolor{darkgray176}{RGB}{176,176,176}
\definecolor{tomato}{RGB}{255,99,71}
\definecolor{turquoise}{RGB}{64,224,208}

\begin{axis}[
tick align=outside,
tick pos=left,
x grid style={darkgray176},
xmin=0.45, xmax=12.55,
xtick style={color=black},
y grid style={darkgray176},
ymin=0.453999996185303, ymax=1,
ytick style={color=black}
]
\path [draw=tomato, fill=tomato, opacity=0.2]
(axis cs:1,0.644599437713623)
--(axis cs:1,0.412067174911499)
--(axis cs:2,0.577101945877075)
--(axis cs:3,0.664000451564789)
--(axis cs:4,0.70506489276886)
--(axis cs:5,0.741226494312286)
--(axis cs:6,0.804183840751648)
--(axis cs:7,0.806729376316071)
--(axis cs:8,0.78540575504303)
--(axis cs:9,0.837900221347809)
--(axis cs:10,0.843335330486298)
--(axis cs:11,0.820995330810547)
--(axis cs:12,0.849636912345886)
--(axis cs:12,0.926362991333008)
--(axis cs:12,0.926362991333008)
--(axis cs:11,0.925004601478577)
--(axis cs:10,0.917331278324127)
--(axis cs:9,0.924099624156952)
--(axis cs:8,0.914594173431396)
--(axis cs:7,0.887270629405975)
--(axis cs:6,0.876482725143433)
--(axis cs:5,0.879440009593964)
--(axis cs:4,0.847601890563965)
--(axis cs:3,0.830666244029999)
--(axis cs:2,0.74356484413147)
--(axis cs:1,0.644599437713623)
--cycle;

\path [draw=turquoise, fill=turquoise, opacity=0.2]
(axis cs:1,0.810379803180695)
--(axis cs:1,0.688953459262848)
--(axis cs:2,0.812527537345886)
--(axis cs:3,0.8450648188591)
--(axis cs:4,0.847389042377472)
--(axis cs:5,0.852578163146973)
--(axis cs:6,0.865930795669556)
--(axis cs:7,0.877160370349884)
--(axis cs:8,0.867050588130951)
--(axis cs:9,0.883001625537872)
--(axis cs:10,0.885078310966492)
--(axis cs:11,0.873212456703186)
--(axis cs:12,0.875157594680786)
--(axis cs:12,0.940842270851135)
--(axis cs:12,0.940842270851135)
--(axis cs:11,0.925454139709473)
--(axis cs:10,0.935588359832764)
--(axis cs:9,0.950331628322601)
--(axis cs:8,0.938282787799835)
--(axis cs:7,0.93483966588974)
--(axis cs:6,0.923402428627014)
--(axis cs:5,0.925421833992004)
--(axis cs:4,0.929277718067169)
--(axis cs:3,0.917601764202118)
--(axis cs:2,0.876139163970947)
--(axis cs:1,0.810379803180695)
--cycle;

\addplot [semithick, tomato, dashed]
table {%
1 0.528333306312561
2 0.660333395004272
3 0.747333347797394
4 0.776333391666412
5 0.810333251953125
6 0.84033328294754
7 0.847000002861023
8 0.849999964237213
9 0.88099992275238
10 0.880333304405212
11 0.872999966144562
12 0.887999951839447
};
\addplot [semithick, turquoise]
table {%
1 0.749666631221771
2 0.844333350658417
3 0.881333291530609
4 0.888333380222321
5 0.888999998569489
6 0.894666612148285
7 0.906000018119812
8 0.902666687965393
9 0.916666626930237
10 0.910333335399628
11 0.899333298206329
12 0.907999932765961
};
\end{axis}

\end{tikzpicture}}
\end{subfigure}

\begin{subfigure}{0.19\linewidth}
    \centering
    \resizebox{\linewidth}{!}{
\begin{tikzpicture}

\definecolor{darkgray176}{RGB}{176,176,176}
\definecolor{tomato}{RGB}{255,99,71}
\definecolor{turquoise}{RGB}{64,224,208}

\begin{axis}[
tick align=outside,
tick pos=left,
x grid style={darkgray176},
xmin=0.45, xmax=12.55,
xtick style={color=black},
y grid style={darkgray176},
ymin=0.453999996185303, ymax=1,
ytick style={color=black}
]
\path [draw=tomato, fill=tomato, opacity=0.2]
(axis cs:1,0.626630902290344)
--(axis cs:1,0.380035728216171)
--(axis cs:2,0.455430179834366)
--(axis cs:3,0.523545265197754)
--(axis cs:4,0.545271992683411)
--(axis cs:5,0.634646117687225)
--(axis cs:6,0.630043089389801)
--(axis cs:7,0.682645261287689)
--(axis cs:8,0.678333878517151)
--(axis cs:9,0.701837360858917)
--(axis cs:10,0.679949104785919)
--(axis cs:11,0.704580843448639)
--(axis cs:12,0.705701172351837)
--(axis cs:12,0.859632074832916)
--(axis cs:12,0.859632074832916)
--(axis cs:11,0.847419083118439)
--(axis cs:10,0.855384171009064)
--(axis cs:9,0.839495956897736)
--(axis cs:8,0.820332765579224)
--(axis cs:7,0.822687923908234)
--(axis cs:6,0.789956867694855)
--(axis cs:5,0.788020431995392)
--(axis cs:4,0.763394594192505)
--(axis cs:3,0.742454767227173)
--(axis cs:2,0.687903165817261)
--(axis cs:1,0.626630902290344)
--cycle;

\path [draw=turquoise, fill=turquoise, opacity=0.2]
(axis cs:1,0.826919972896576)
--(axis cs:1,0.695746719837189)
--(axis cs:2,0.806891560554504)
--(axis cs:3,0.835606873035431)
--(axis cs:4,0.833810269832611)
--(axis cs:5,0.872246861457825)
--(axis cs:6,0.865009784698486)
--(axis cs:7,0.872769296169281)
--(axis cs:8,0.860849320888519)
--(axis cs:9,0.870126247406006)
--(axis cs:10,0.88496333360672)
--(axis cs:11,0.877321779727936)
--(axis cs:12,0.889842391014099)
--(axis cs:12,0.943490862846375)
--(axis cs:12,0.943490862846375)
--(axis cs:11,0.945344746112823)
--(axis cs:10,0.949703276157379)
--(axis cs:9,0.940540432929993)
--(axis cs:8,0.938483893871307)
--(axis cs:7,0.930563986301422)
--(axis cs:6,0.938990235328674)
--(axis cs:5,0.938419699668884)
--(axis cs:4,0.929522931575775)
--(axis cs:3,0.903059661388397)
--(axis cs:2,0.907775044441223)
--(axis cs:1,0.826919972896576)
--cycle;

\addplot [semithick, tomato, dashed]
table {%
1 0.503333330154419
2 0.571666657924652
3 0.633000016212463
4 0.654333293437958
5 0.711333274841309
6 0.709999978542328
7 0.752666592597961
8 0.749333322048187
9 0.770666658878326
10 0.767666637897491
11 0.775999963283539
12 0.782666623592377
};
\addplot [semithick, turquoise]
table {%
1 0.761333346366882
2 0.857333302497864
3 0.869333267211914
4 0.881666600704193
5 0.905333280563354
6 0.90200001001358
7 0.901666641235352
8 0.899666607379913
9 0.905333340167999
10 0.91733330488205
11 0.91133326292038
12 0.916666626930237
};
\end{axis}

\end{tikzpicture}}
\end{subfigure}
\begin{subfigure}{0.19\linewidth}
    \centering
    \resizebox{\linewidth}{!}{
\begin{tikzpicture}

\definecolor{darkgray176}{RGB}{176,176,176}
\definecolor{tomato}{RGB}{255,99,71}
\definecolor{turquoise}{RGB}{64,224,208}

\begin{axis}[
tick align=outside,
tick pos=left,
x grid style={darkgray176},
xmin=0.45, xmax=12.55,
xtick style={color=black},
y grid style={darkgray176},
ymin=0.453999996185303, ymax=1,
ytick style={color=black}
]
\path [draw=tomato, fill=tomato, opacity=0.2]
(axis cs:1,0.626630902290344)
--(axis cs:1,0.380035728216171)
--(axis cs:2,0.4532550573349)
--(axis cs:3,0.474421560764313)
--(axis cs:4,0.506450653076172)
--(axis cs:5,0.558942019939423)
--(axis cs:6,0.553536236286163)
--(axis cs:7,0.598144173622131)
--(axis cs:8,0.566739678382874)
--(axis cs:9,0.630621790885925)
--(axis cs:10,0.600000858306885)
--(axis cs:11,0.635716259479523)
--(axis cs:12,0.60429573059082)
--(axis cs:12,0.899704217910767)
--(axis cs:12,0.899704217910767)
--(axis cs:11,0.888283789157867)
--(axis cs:10,0.891332268714905)
--(axis cs:9,0.850711584091187)
--(axis cs:8,0.849260330200195)
--(axis cs:7,0.831189036369324)
--(axis cs:6,0.823797047138214)
--(axis cs:5,0.805724561214447)
--(axis cs:4,0.78221583366394)
--(axis cs:3,0.772245109081268)
--(axis cs:2,0.669411659240723)
--(axis cs:1,0.626630902290344)
--cycle;

\path [draw=turquoise, fill=turquoise, opacity=0.2]
(axis cs:1,0.826919972896576)
--(axis cs:1,0.695746719837189)
--(axis cs:2,0.805620610713959)
--(axis cs:3,0.821964681148529)
--(axis cs:4,0.830783724784851)
--(axis cs:5,0.859922826290131)
--(axis cs:6,0.861192584037781)
--(axis cs:7,0.870038866996765)
--(axis cs:8,0.856127619743347)
--(axis cs:9,0.865009784698486)
--(axis cs:10,0.879870176315308)
--(axis cs:11,0.869365692138672)
--(axis cs:12,0.886381804943085)
--(axis cs:12,0.941618025302887)
--(axis cs:12,0.941618025302887)
--(axis cs:11,0.942634224891663)
--(axis cs:10,0.94612979888916)
--(axis cs:9,0.938990235328674)
--(axis cs:8,0.935205698013306)
--(axis cs:7,0.9266277551651)
--(axis cs:6,0.937474012374878)
--(axis cs:5,0.934743702411652)
--(axis cs:4,0.928549528121948)
--(axis cs:3,0.900701820850372)
--(axis cs:2,0.899712741374969)
--(axis cs:1,0.826919972896576)
--cycle;

\addplot [semithick, tomato, dashed]
table {%
1 0.503333330154419
2 0.561333358287811
3 0.623333334922791
4 0.644333243370056
5 0.682333290576935
6 0.688666641712189
7 0.714666604995728
8 0.708000004291534
9 0.740666687488556
10 0.745666563510895
11 0.762000024318695
12 0.751999974250793
};
\addplot [semithick, turquoise]
table {%
1 0.761333346366882
2 0.852666676044464
3 0.861333250999451
4 0.8796666264534
5 0.897333264350891
6 0.899333298206329
7 0.898333311080933
8 0.895666658878326
9 0.90200001001358
10 0.912999987602234
11 0.905999958515167
12 0.913999915122986
};
\end{axis}

\end{tikzpicture}}
\end{subfigure}
\begin{subfigure}{0.19\linewidth}
    \centering
    \resizebox{\linewidth}{!}{
\begin{tikzpicture}

\definecolor{darkgray176}{RGB}{176,176,176}
\definecolor{tomato}{RGB}{255,99,71}
\definecolor{turquoise}{RGB}{64,224,208}

\begin{axis}[
tick align=outside,
tick pos=left,
x grid style={darkgray176},
xmin=0.45, xmax=12.55,
xtick style={color=black},
y grid style={darkgray176},
ymin=0.453999996185303, ymax=1,
ytick style={color=black}
]
\path [draw=tomato, fill=tomato, opacity=0.2]
(axis cs:1,0.625369846820831)
--(axis cs:1,0.385963499546051)
--(axis cs:2,0.480056047439575)
--(axis cs:3,0.528533101081848)
--(axis cs:4,0.549689590930939)
--(axis cs:5,0.592918395996094)
--(axis cs:6,0.641767382621765)
--(axis cs:7,0.645514726638794)
--(axis cs:8,0.659544706344604)
--(axis cs:9,0.715419352054596)
--(axis cs:10,0.715059757232666)
--(axis cs:11,0.718388855457306)
--(axis cs:12,0.717898011207581)
--(axis cs:12,0.840101957321167)
--(axis cs:12,0.840101957321167)
--(axis cs:11,0.823611080646515)
--(axis cs:10,0.838273406028748)
--(axis cs:9,0.829247176647186)
--(axis cs:8,0.802455306053162)
--(axis cs:7,0.795818567276001)
--(axis cs:6,0.788899302482605)
--(axis cs:5,0.743081569671631)
--(axis cs:4,0.741643846035004)
--(axis cs:3,0.690800070762634)
--(axis cs:2,0.683277249336243)
--(axis cs:1,0.625369846820831)
--cycle;

\path [draw=turquoise, fill=turquoise, opacity=0.2]
(axis cs:1,0.82733541727066)
--(axis cs:1,0.689997851848602)
--(axis cs:2,0.790270149707794)
--(axis cs:3,0.828779339790344)
--(axis cs:4,0.8320032954216)
--(axis cs:5,0.868191242218018)
--(axis cs:6,0.860464215278625)
--(axis cs:7,0.871833443641663)
--(axis cs:8,0.861090660095215)
--(axis cs:9,0.869069933891296)
--(axis cs:10,0.883974313735962)
--(axis cs:11,0.874977469444275)
--(axis cs:12,0.88800984621048)
--(axis cs:12,0.943990170955658)
--(axis cs:12,0.943990170955658)
--(axis cs:11,0.944355726242065)
--(axis cs:10,0.95002555847168)
--(axis cs:9,0.940263390541077)
--(axis cs:8,0.936242580413818)
--(axis cs:7,0.930166602134705)
--(axis cs:6,0.940202355384827)
--(axis cs:5,0.939142107963562)
--(axis cs:4,0.925330102443695)
--(axis cs:3,0.898553967475891)
--(axis cs:2,0.897729814052582)
--(axis cs:1,0.82733541727066)
--cycle;

\addplot [semithick, tomato, dashed]
table {%
1 0.505666673183441
2 0.581666648387909
3 0.609666585922241
4 0.645666718482971
5 0.667999982833862
6 0.715333342552185
7 0.720666646957397
8 0.731000006198883
9 0.772333264350891
10 0.776666581630707
11 0.77099996805191
12 0.778999984264374
};
\addplot [semithick, turquoise]
table {%
1 0.758666634559631
2 0.843999981880188
3 0.863666653633118
4 0.878666698932648
5 0.90366667509079
6 0.900333285331726
7 0.901000022888184
8 0.898666620254517
9 0.904666662216187
10 0.916999936103821
11 0.90966659784317
12 0.916000008583069
};
\end{axis}

\end{tikzpicture}}
\end{subfigure}
\begin{subfigure}{0.19\linewidth}
    \centering
    \resizebox{\linewidth}{!}{
\begin{tikzpicture}

\definecolor{darkgray176}{RGB}{176,176,176}
\definecolor{tomato}{RGB}{255,99,71}
\definecolor{turquoise}{RGB}{64,224,208}

\begin{axis}[
tick align=outside,
tick pos=left,
x grid style={darkgray176},
xmin=0.45, xmax=12.55,
xtick style={color=black},
y grid style={darkgray176},
ymin=0.453999996185303, ymax=1,
ytick style={color=black}
]
\path [draw=tomato, fill=tomato, opacity=0.2]
(axis cs:1,0.621269822120667)
--(axis cs:1,0.37739685177803)
--(axis cs:2,0.538211584091187)
--(axis cs:3,0.61274266242981)
--(axis cs:4,0.668482899665833)
--(axis cs:5,0.743489325046539)
--(axis cs:6,0.763303279876709)
--(axis cs:7,0.790422320365906)
--(axis cs:8,0.774036884307861)
--(axis cs:9,0.814205706119537)
--(axis cs:10,0.80991119146347)
--(axis cs:11,0.831549644470215)
--(axis cs:12,0.815958976745605)
--(axis cs:12,0.919374227523804)
--(axis cs:12,0.919374227523804)
--(axis cs:11,0.918450355529785)
--(axis cs:10,0.919422090053558)
--(axis cs:9,0.923794329166412)
--(axis cs:8,0.900629758834839)
--(axis cs:7,0.90357768535614)
--(axis cs:6,0.892030000686646)
--(axis cs:5,0.868510663509369)
--(axis cs:4,0.842183709144592)
--(axis cs:3,0.827923893928528)
--(axis cs:2,0.748455047607422)
--(axis cs:1,0.621269822120667)
--cycle;

\path [draw=turquoise, fill=turquoise, opacity=0.2]
(axis cs:1,0.842594027519226)
--(axis cs:1,0.679405927658081)
--(axis cs:2,0.798471570014954)
--(axis cs:3,0.83344042301178)
--(axis cs:4,0.831006944179535)
--(axis cs:5,0.861832082271576)
--(axis cs:6,0.861002445220947)
--(axis cs:7,0.861026108264923)
--(axis cs:8,0.852748334407806)
--(axis cs:9,0.8590127825737)
--(axis cs:10,0.886220455169678)
--(axis cs:11,0.873186945915222)
--(axis cs:12,0.878624975681305)
--(axis cs:12,0.92737489938736)
--(axis cs:12,0.92737489938736)
--(axis cs:11,0.936812996864319)
--(axis cs:10,0.946446061134338)
--(axis cs:9,0.936987102031708)
--(axis cs:8,0.931251585483551)
--(axis cs:7,0.926307141780853)
--(axis cs:6,0.930997490882874)
--(axis cs:5,0.935501158237457)
--(axis cs:4,0.924326360225677)
--(axis cs:3,0.90455961227417)
--(axis cs:2,0.901528358459473)
--(axis cs:1,0.842594027519226)
--cycle;

\addplot [semithick, tomato, dashed]
table {%
1 0.499333322048187
2 0.643333315849304
3 0.720333278179169
4 0.755333304405212
5 0.805999994277954
6 0.827666640281677
7 0.847000002861023
8 0.83733332157135
9 0.869000017642975
10 0.864666640758514
11 0.875
12 0.867666602134705
};
\addplot [semithick, turquoise]
table {%
1 0.760999977588654
2 0.849999964237213
3 0.869000017642975
4 0.877666652202606
5 0.898666620254517
6 0.89599996805191
7 0.893666625022888
8 0.891999959945679
9 0.897999942302704
10 0.916333258152008
11 0.904999971389771
12 0.902999937534332
};
\end{axis}

\end{tikzpicture}}
\end{subfigure}
\begin{subfigure}{0.19\linewidth}
    \centering
    \resizebox{\linewidth}{!}{
\begin{tikzpicture}

\definecolor{darkgray176}{RGB}{176,176,176}
\definecolor{tomato}{RGB}{255,99,71}
\definecolor{turquoise}{RGB}{64,224,208}

\begin{axis}[
tick align=outside,
tick pos=left,
x grid style={darkgray176},
xmin=0.45, xmax=12.55,
xtick style={color=black},
y grid style={darkgray176},
ymin=0.453999996185303, ymax=1,
ytick style={color=black}
]
\path [draw=tomato, fill=tomato, opacity=0.2]
(axis cs:1,0.668038010597229)
--(axis cs:1,0.435961961746216)
--(axis cs:2,0.581790924072266)
--(axis cs:3,0.635142207145691)
--(axis cs:4,0.68700909614563)
--(axis cs:5,0.755411267280579)
--(axis cs:6,0.771685302257538)
--(axis cs:7,0.807275950908661)
--(axis cs:8,0.784996867179871)
--(axis cs:9,0.81751811504364)
--(axis cs:10,0.820428848266602)
--(axis cs:11,0.828335404396057)
--(axis cs:12,0.827095746994019)
--(axis cs:12,0.92023754119873)
--(axis cs:12,0.92023754119873)
--(axis cs:11,0.921664595603943)
--(axis cs:10,0.922237753868103)
--(axis cs:9,0.92648184299469)
--(axis cs:8,0.907003045082092)
--(axis cs:7,0.898724019527435)
--(axis cs:6,0.904981315135956)
--(axis cs:5,0.881255507469177)
--(axis cs:4,0.856990814208984)
--(axis cs:3,0.829524278640747)
--(axis cs:2,0.768209099769592)
--(axis cs:1,0.668038010597229)
--cycle;

\path [draw=turquoise, fill=turquoise, opacity=0.2]
(axis cs:1,0.860384583473206)
--(axis cs:1,0.690948724746704)
--(axis cs:2,0.805305540561676)
--(axis cs:3,0.839149355888367)
--(axis cs:4,0.832089304924011)
--(axis cs:5,0.867131114006042)
--(axis cs:6,0.867729544639587)
--(axis cs:7,0.866133213043213)
--(axis cs:8,0.856940269470215)
--(axis cs:9,0.861459612846375)
--(axis cs:10,0.885537087917328)
--(axis cs:11,0.87351381778717)
--(axis cs:12,0.878189027309418)
--(axis cs:12,0.92981094121933)
--(axis cs:12,0.92981094121933)
--(axis cs:11,0.936486124992371)
--(axis cs:10,0.945129573345184)
--(axis cs:9,0.935873746871948)
--(axis cs:8,0.93372631072998)
--(axis cs:7,0.927866816520691)
--(axis cs:6,0.932270407676697)
--(axis cs:5,0.937535524368286)
--(axis cs:4,0.924577236175537)
--(axis cs:3,0.906850576400757)
--(axis cs:2,0.899361073970795)
--(axis cs:1,0.860384583473206)
--cycle;

\addplot [semithick, tomato, dashed]
table {%
1 0.551999986171722
2 0.675000011920929
3 0.732333242893219
4 0.771999955177307
5 0.818333387374878
6 0.838333308696747
7 0.852999985218048
8 0.845999956130981
9 0.871999979019165
10 0.871333301067352
11 0.875
12 0.873666644096375
};
\addplot [semithick, turquoise]
table {%
1 0.775666654109955
2 0.852333307266235
3 0.872999966144562
4 0.878333270549774
5 0.902333319187164
6 0.899999976158142
7 0.897000014781952
8 0.895333290100098
9 0.898666679859161
10 0.915333330631256
11 0.904999971389771
12 0.903999984264374
};
\end{axis}

\end{tikzpicture}}
\end{subfigure}
\caption{Fraction of test documents classified correctly vs number of representative documents. Solid turquoise line corresponds and dashed tomato line correspond to learned and random reference documents respectively. Top to bottom: each row corresponds to $\rho$ increasing in $\{0.0, 0.01, 0.1, 1.0, 10.0\}$. Left to right: each column corresponds to the methods \{1NN, MAD, MBL, MBL-QP, MC\}. We observe that learned documents outperform random documents in every experiment for all levels of $\rho$. The smaller variance of the learned documents is explained by the fact that the learned documents were trained with more documents than were used.}
\label{fig:nlp-std}
\end{figure}

\end{document}